\documentclass{article}

\usepackage{microtype}
\usepackage{graphicx}
\usepackage{subfigure}
\usepackage{booktabs} 
\usepackage{hyperref}
\usepackage{enumitem}

\usepackage[accepted]{icml2024}

\usepackage{amsmath}
\usepackage{amssymb}
\usepackage{mathtools}
\usepackage{amsthm}
\usepackage[capitalize,noabbrev]{cleveref}

\theoremstyle{plain}
\newtheorem{theorem}{Theorem}[section]
\newtheorem{proposition}[theorem]{Proposition}
\newtheorem{lemma}[theorem]{Lemma}
\newtheorem{corollary}[theorem]{Corollary}
\theoremstyle{definition}
\newtheorem{definition}[theorem]{Definition}
\newtheorem{assumption}[theorem]{Assumption}
\theoremstyle{remark}
\newtheorem{remark}[theorem]{Remark}

\usepackage[textsize=tiny]{todonotes}

\newcommand{\rb}[1]{\left(#1\right)} 
\newcommand{\bb}[1]{\left[#1\right]} 
\newcommand{\cb}[1]{\left\{#1\right\}} 
\newcommand{\norm}[1]{\left\Vert#1\right\Vert}
\newcommand{\abs}[1]{{\left\lvert{#1}\right\rvert}}
\newcommand{\vect}[1]{\mathbf{#1}} 
\newcommand{\R}{\mathbb{R}} 
\newcommand{\calD}{\mathcal{D}}
\newcommand{\calP}{\mathcal{P}}
\newcommand{\calS}{\mathcal{S}}

\newcommand{\calH}{\mathcal{H}}
\newcommand{\calL}{\mathcal{L}}

\newcommand{\bbP}{\mathbb{P}}

\newcommand{\bbN}{\mathbb{N}}
\newcommand{\GaCtext}{G\&C}
\newcommand{\exError}[1]{\mathcal{L}_\calD(#1)}
\newcommand{\emError}[1]{\mathcal{L}_\mathcal{S}(#1)}

\DeclareMathOperator*{\argmin}{\arg\!\min}

\DeclareMathOperator{\sign}{sign}
\newcommand{\btheta}{\boldsymbol{\theta}}
\newcommand{\bgamma}{\boldsymbol{\gamma}}
\newcommand{\teacher}{h_{\btheta^{\star}}}
\newcommand{\student}{h_{\btheta}}

\newcommand{\fcnhid}[1]{f_{\mathrm{FC}}^{\left(#1\right)}}

\newcommand{\sfchid}[1]{f_{\mathrm{SFC}}^{\left(#1\right)}}
\newcommand{\scn}{\mathbf{h}^{\mathrm{SCN}}_{\btheta}}
\newcommand{\scnlc}[1]{\mathbf{f}_{\mathrm{SCN}}^{\left(#1\right)}}

\newcommand{\cnn}{\mathbf{h}^{\mathrm{CNN}}_{\btheta}}
\newcommand{\cnnlc}[1]{\mathbf{f}_{\mathrm{CNN}}^{\left(#1\right)}}
\newcommand{\cnnl}[2]{\mathbf{f}_{\mathrm{CNN}}^{\left(#1\right), #2}}
\newcommand{\vechid}[1]{\mathbf{f}_{D}^{\left(#1\right)}}
\newcommand{\veccnl}[2]{\mathbf{f}_{D}^{\left(#1\right), #2}}
\newcommand{\vhblock}[2]{\mathbf{f}_{D,#2}^{\left(#1\right)}}
\newcommand{\Hcnn}[1]{\mathcal{H}^{\mathrm{CNN}}_{#1}}
\newcommand{\Hscn}[1]{\mathcal{H}^{\mathrm{SCN}}_{#1}}
\newcommand{\Hfc}[1]{\mathcal{H}^{\mathrm{FC}}_{#1}}
\newcommand{\Hsfc}[1]{\mathcal{H}^{\mathrm{SFC}}_{#1}}

\newcommand{\pteacher}{\tilde{p}}
\newcommand{\pteacherSquared}{{\pteacher}^2}

\def\secref#1{Section~\ref{#1}}

\def\subsecref#1{Section~\ref{#1}}

\def\lemref#1{Lemma~\ref{#1}}

\def\thmref#1{Theorem~\ref{#1}}

\def\corref#1{Corollary~\ref{#1}}

\def\defref#1{Def.~\ref{#1}}
\def\propref#1{Prop.~\ref{#1}}

\def\asmref#1{Assumption~\ref{#1}}
\def\appref#1{Appendix~\ref{#1}}

\def\figref#1{Figure~\ref{#1}}
\def\algref#1{Algorithm~\ref{#1}}
\def\remref#1{Remark~\ref{#1}}

\newcommand{\lconv}{l}

\newcommand{\PC}{\mathcal{M}} 
\newcommand{\SC}{\tilde{C}} 
\newcommand{\C}{C} 
\newcommand{\BSC}{\hat{C}} 
\newcommand{\BSCM}[1]{\BSC^{\mathrm{#1}}} 
\newcommand{\ccfct}{M \rb{D^\star}}
\newcommand{\wcfct}{W \rb{D^\star}}
\newcommand{\bcfct}{B \rb{D^\star}}

\icmltitlerunning{Typical Neural Networks Generalize with Narrow Teachers}

\begin{document}

\twocolumn[
\icmltitle{How Uniform Random Weights Induce Non-uniform Bias:\\ Typical Interpolating Neural Networks Generalize with Narrow Teachers}


\icmlsetsymbol{equal}{*}

\begin{icmlauthorlist}
\icmlauthor{Gon Buzaglo}{equal,technion}
\icmlauthor{Itamar Harel}{equal,technion}
\icmlauthor{Mor Shpigel Nacson}{equal,technion}
\icmlauthor{Alon Brutzkus}{technion}
\icmlauthor{Nathan Srebro}{ttic}
\icmlauthor{Daniel Soudry}{technion}
\end{icmlauthorlist}

\icmlaffiliation{technion}{Technion Institute of Technology, Haifa, Israel}
\icmlaffiliation{ttic}{Toyota Technological Institute at Chicago, Chicago IL, USA}

\icmlcorrespondingauthor{Gon Buzaglo}{gon.buzaglo@gmail.com}
\icmlcorrespondingauthor{Itamar Harel}{itamarharel01@gmail.com}

\icmlkeywords{Machine Learning, ICML}

\vskip 0.3in
]



\printAffiliationsAndNotice{\icmlEqualContribution} 

\begin{abstract}
\textbf{Background.} A main theoretical puzzle is why over-parameterized Neural Networks (NNs) generalize well when trained to zero loss (i.e., so they interpolate the data). Usually, the NN is trained with Stochastic Gradient Descent (SGD) or one of its variants. However, recent empirical work examined the generalization of a random NN that interpolates the data: the NN was sampled from a seemingly uniform prior over the parameters, conditioned on that the NN perfectly classifies the training set. Interestingly, such a NN sample typically generalized as well as SGD-trained NNs. \textbf{Contributions.} We prove that such a random NN interpolator typically generalizes well if there exists an underlying narrow ``teacher NN'' that agrees with the labels. Specifically, we show that such a `flat' prior over the NN parameterization induces a rich prior over the NN functions, due to the redundancy in the NN structure. In particular, this creates a bias towards simpler functions, which require less relevant parameters to represent --- enabling learning with a sample complexity approximately proportional to the complexity of the teacher (roughly, the number of non-redundant parameters), rather than the student's. 
\end{abstract}

\section{Introduction}
A central theoretical question in deep learning is why Neural Networks (NNs) generalize, despite being over-parameterized, and even when perfectly fitted to the data \citep{Zhang16}. One of the leading explanations for this phenomenon is that NNs have an ``implicit bias" toward generalizing solutions (e.g.,  \citet{Gunasekar2017,soudry2017implicit,arora2019implicit,lyu2019gradient,pmlr-v125-chizat20a,vardi2023implicit}). This bias stems from underlying interactions between the model and the training method --- including the type of optimization step, the initialization, the parameterization, and the loss function.

Previous works \citep{vallepérez2019deep,mingard2021sgd,chiang2023loss} suggested, based on empirical evidence, that a significant part of this implicit bias in NNs is the mapping from the model parameters to the model function. Specifically, suppose we randomly sample the NN parameters from a `uniform' prior\footnote{A truly uniform prior does not exist for infinite sets, so the prior is chosen similarly to standard `uniform-like' initializations: in each layer, the prior is Gaussian (uniform on the $\ell_2$ sphere) or uniform in the $\ell_{\infty}$ ball.}, and accept only parameter samples in which the NN perfectly classifies all the training data --- i.e., samples from the posterior composed of the same prior and the likelihood of a 0-1 loss function. Then, \citet{chiang2023loss} found that the sampled NNs generalize as well as SGD in small-scale experiments. 

These results may suggest that such a uniform sampling of the NN parameters induces simple and generalizing NN functions. 
In this paper, we prove this is indeed the case, and aim to uncover the mechanism behind this phenomenon. In short, we prove typical NN interpolators sampled this way (``students") generalize well, given there exists a “narrow” NN teacher that generates the labels. Next, we explain these results in more detail.

    \textbf{Contributions.} In  \secref{sec:gen_bound} we prove that a typical NN sampled from the posterior over interpolators generalizes well (i.e.,~has a small test error with high probability) with     
    \begin{equation} \label{eq: sample complexity log p}
        \#\textrm{samples} = O\left(-\log \pteacher\right) \,,
    \end{equation} 
    where $\pteacher$ is the probability that a random NN (sampled from the `uniform' prior) is equivalent to the teacher function.\footnote{This is a generic argument and essentially a special case of a PAC-Bayes guarantee, but we give a simple proof based on the following idea: the number of hypotheses sampled until a successful interpolation is $\left|\mathcal{H}\right| \lessapprox 1/\pteacher$. Plugging this into the standard sample complexity of a finite hypotheses class  $O\left(\log\left|\mathcal{H}\right|\right)$, we obtain the result. The actual proof is slightly more complicated since $\left|\mathcal{H}\right|$ here weakly depends on the training set.} Thus, to obtain generalization guarantees for NNs, we proceed to upper bound $\rb{-\log \pteacher}$. 
    
     Next, in \secref{sec:quant_nets},  we examine the case where both the student and teacher parameters are quantized to $Q$ levels, including zero (as in standard numerical formats), and we assume the prior is uniform over all possible quantized values. We examine several architectures:
    \begin{itemize}
        \item  For a fully connected multi-layer network with a scalar output, hidden neuron layer widths $\{d_l\}_{l=1}^{L}$ and $\{d^{\star}_l\}_{l=1}^{L}$ respectively for the student and teacher, input width $d_0=d_0^{\star}$, and any activation function $\sigma$ that satisfies $\sigma\rb{0}=0$, we prove 
        \begin{equation} \label{eq: FC bound}
            -{ \log \pteacher}  \leq \sum_{l=1}^{L}\rb{d_{l}^{\star}d_{l-1}^{\star}+2d_{l}}{\log Q}.
        \end{equation}
        \item  For convolutional NNs, we obtain analogous results, where channel numbers replace layer widths, with an additional multiplicative factor of the kernel size.        
        \item The proofs in both cases are 
        simple\footnote{Proof idea for two-layer FC nets without biases: The NN function is identical to the teacher NN function, if we set $d_{1}^{\star}$ hidden neurons with the same ingoing and outgoing weights as in the teacher; and for the other $d_{1}-d_{1}^{\star}$ hidden neurons we set the outgoing weights (for the zeroed neurons, the input weights do not matter). This event probability is $\pteacher=Q^{-d_{0}d_{1}^{\star}-d_{1}}$, satisfying \eqref{eq: FC bound}.}, and can be extended for more general architectures. 
    \end{itemize}
    
    Lastly, in \secref{sec: continuous} we examine a two-layer neural network with continuous weights and derive similar results, except a margin assumption replaces the quantization assumption, and a margin factor replaces $Q$ in the bound.
    
    \textbf{Implications.} Combining these relatively easy-to-prove results (\eqref{eq: sample complexity log p} and \eqref{eq: FC bound}), we get a surprisingly novel result: typical NN interpolators have sample complexity approximately proportional to the number of teacher parameters times the number of quantization bits, with only a weak dependence on the student width. Thus, the student generalizes well \textit{if there exists} a teacher that is sufficiently narrow and under-parameterized in comparison to the sample number. As a corollary, we show that with high probability over the training set, the volume of interpolators with high generalization error is exponentially small in the size of the training set. 

In \secref{sec: discussion} we discuss our assumptions (teacher narrowness and weight quantization), how our results can be straightforwardly extended beyond interplators (to functions with a non-zero training error), whether posterior samping biases us towards sparse represetnations, the effect of parameterization via the minimum description length framework, and the relation of our results to SGD.

\section{Preliminaries}
\textbf{Notation.} We use boldface letters for vectors and matrices. 
A vector $\vect{x}\in\R^d$ is assumed to be a column vector, and we use $x_i$ to denote its $i$-th coordinate. We denote by $\mathrm{Vec}\rb{\cdot}$ the vectorization operation, which converts a tensor into a column vector by stacking its columns. The indicator function $\mathbb{I}\bb{A}$ is $1$ if statement $A$ is true and $0$ if statement $A$ is false. Additionally, we use the standard notation $[N]=\{1,\dots,N\}$ and take  $\norm{\cdot}$ to be the Euclidean norm.
We use the symbols $\odot$ to denote the Hadamard product, i.e. elementwise multiplication, $\otimes$ to denote the Kronecker product, and $*$ to denote the convolution operator. 
For a pair of vectors $D^\prime=\rb{d_1^{\prime},\dots,d_L^{\prime}},D^{\prime\prime}=\rb{d_1^{\prime\prime},\dots,d_L^{\prime\prime}} \in \mathbb{N}^L$ we denote $D^\prime \le D^{\prime\prime}$ if for all $l\in\bb{L}$, $d_l^\prime \le d_l^{\prime\prime}$.

\textbf{Data.}
Let $\mathcal{D}$ be some data distribution. We consider the problem of binary classification over a finite training set $\mathcal{S}$ that contains $N$ datapoints sampled i.i.d. from $\mathcal{D}$: 
\begin{equation*}
    \mathcal{S}\triangleq\{\vect{x}_n \}_{n=1}^N\sim\mathcal{D}^N\,,
\end{equation*} where $\vect{x}_n\in\R^{d_0}$. Since we are interested in realizable models we assume there exists a teacher model, $h^{\star}$, generating binary labels, i.e. $h^{\star}(\vect{x})\in \{ \pm 1\}$ for any $ \vect{x} \sim \mathcal{D}$.

\textbf{Evaluation metrics.}
For a predictor $h:\R^{d_0}\mapsto \cb{\pm1}$, we define the risk, i.e. the population error
$\exError{h}\triangleq\bbP_{\vect{x}\sim\calD}\left(h(\vect{x})\neq h^{\star}\rb{\vect{x}}\right),
$ and the empirical risk, i.e. the training error
$
    \emError{h}\triangleq\frac{1}{N}\sum_{n=1}^N \mathbb{I}\bb{h(\vect{x}_n)\neq h^{\star}\rb{\vect{x}_n}}\,.
$

\textbf{Hypothesis parameterization.} We discuss parameterized predictors $\btheta \mapsto \student$, where $\btheta \in \mathbb{R}^{M}$.
Distributions over $\btheta$ therefore induce distributions over hypotheses via 
\begin{align*}
    \mathcal{P} \rb{h} 
    \triangleq 
    \mathbb{P}_{\btheta} \rb{ 
    h_{\btheta} = h
    }~.
\end{align*}
That is, $\calP(h)$ is the probability mass function of sampling parameters $\btheta$ mapping to $h$ when the distribution is discrete, or (with a slight abuse of notation) their density when the distribution is continuous.

\section{Generalization Bounds for Random Interpolating Hypotheses}\label{sec:gen_bound}

In this paper, we study the generalization of interpolating predictors sampled from the posterior of NNs: 
\begin{equation} \label{eq: posterior}
    \mathcal{P}_\calS \triangleq \mathcal{P}(h\mid \mathcal{L}_\calS(h)=0)\propto \mathcal{P}(h)\mathbb{I}\left[\mathcal{L}_{\mathcal{S}}(h)=0\right] \, ,
\end{equation} 
where $\mathcal{P}(h)$ is some prior over the hypotheses class. That is, our ``learning rule'' amounts to sampling a single predictor from the posterior,
\begin{equation}
    \mathcal{A}_\calP(\calS) \sim \mathcal{P}_\calS,
\end{equation}
and we would like to analyze the population error $\exError{\mathcal{A}_\calP(\calS)}$ of this sampled predictor.  

Samples from the posterior $\mathcal{P}_\calS$ can be obtained by the Guess and Check procedure (\GaCtext{}; \citet{chiang2023loss}), defined in Algorithm \ref{alg:Guess and Check}, which can be viewed as a rejection sampling procedure for \eqref{eq: posterior}.  
That is, we can think of drawing a sequence $\rb{h_t}_{t=1}^{\infty}$ of hypotheses i.i.d.~from the prior $\calP$ and independent of $\calS$ (i.e.~before seeing the training set). Then, given the training set $\calS$, we pick the first hypothesis in the sequence that interpolates the data. We will employ this equivalence in our analysis, and view samples from the posterior $\mathcal{P}_\calS$ as if they were generated by this procedure.

\begin{algorithm}[ht]
   \caption{Guess and Check (\GaCtext{})}
   \label{alg:Guess and Check}
\begin{algorithmic}
   \STATE {\bfseries Input:} (1) $\calP$, Prior over hypotheses  (2) $\calS$, Training set. 
   \STATE  {\bfseries Output:} $\mathcal{A}_\calP(S)$
   
   \STATE \textbf{Algorithm:}
   \STATE \hskip1.5em Draw $h_1, h_2,\dots\overset{\textrm{i.i.d.}}{\sim} \; \calP$.
   \STATE \hskip1.5em Choose $T \triangleq \min \left\{t \,\vert\, \emError{h_t}=0\right\}$
   \STATE \hskip1.5em {\bfseries Return:} $\mathcal{A}_\calP(S)\triangleq h_T$
\end{algorithmic}
\end{algorithm}

We will be particularly interested in the case in which $\mathcal{P}(h)$ is defined through a `uniform' (or otherwise fairly `flat' or benign) prior on the parameters $\btheta$ in some parameterization $h_{\btheta}$.  But in this section, we analyze posterior sampling, or equivalently Guess and Check, directly through the induced distribution $\mathcal{P}$ over predictors.  In particular, we analyze generalization performance in terms of the probability that a random hypothesis $h\sim\mathcal{P}$ is equivalent to the teacher model.
This is formalized in the following definition.
\begin{definition}\label{def: TE}
We say that a predictor $h$ is  \textit{teacher-equivalent} (TE) w.r.t.~a data distribution $\mathcal{D}$, and denote $h \equiv h^{\star}$, if 
$
\mathbb{P}_{\mathbf{x}\sim \mathcal{D}} \rb{h\left(\mathbf{x}\right)=h^{\star}\left(\mathbf{x}\right)}=1\,,
$
and denote the probability of a random hypothesis to be TE by 
\begin{align*}
    \pteacher\triangleq\bbP_{h \sim \mathcal{P}} \rb{h \equiv h^{\star}}\,.
\end{align*}
\end{definition}
As we show in the next result, $\pteacher$ plays an important role in \GaCtext{} generalization. Specifically, in \appref{app: direct gen bound} we prove the following generalization bound

\begin{lemma}[{\GaCtext{} (i.e.~Posterior Sampling) Generalization}]\label{thm: gen of gac with teacher simple}
Let $\varepsilon\in\left(0,1\right)$ and $\delta\in\rb{0,\frac{1}{5}}$, and assume that $\pteacher<\frac{1}{2}$.
For any $N$ larger than 
\begin{align*}
 \frac{-\log\left({\pteacher}\right) + 3\log\left(\frac{2}{\delta}\right)}{\varepsilon} \,,    
\end{align*}

the sample complexity, we have that
    \begin{align*}
    \mathbb{P}_{\calS \sim\mathcal{D}^{N}, h\sim\mathcal{P}_\mathcal{S}}\left(\exError{h}<\varepsilon\right) \ge 1-\delta\,.
    \end{align*}
\end{lemma}
We observe that the sample complexity required to ensure $(\varepsilon,\delta)$-PAC generalization depends on $\rb{-\log\rb{\pteacher}}$. Thus, we define the effective sample complexity as
\[
    \SC \triangleq -\log\rb{\pteacher}\,.
\]

Moreover, using Markov's inequality, the above lemma implies (see \appref{app:proof_volume}) the following corollary.

\begin{corollary}[Volume of Generalizing Interpolators]\label{cor:volume}
For $\varepsilon,\delta$ as above, and any $N$ larger than
\begin{align*}
         \frac{-\log\left({\pteacher}\right) + 6\log\left(\frac{2}{\delta}\right)}{\varepsilon} \,,    
\end{align*}
the sample complexity, we have that
\begin{align*}
\mathbb{P}_{\mathcal{S}}\left(\mathbb{P}_{h\sim\mathcal{P_{S}}}\left(\mathcal{L_{D}}\left(h\right)\geq \varepsilon\right)<\delta\right)\ge1-\delta\,.
\end{align*}
\end{corollary}

\textbf{\corref{cor:volume} implications.} 
\corref{cor:volume} examines, for a single sample of the data $\calS$, the relative volume of `bad' interpolators out of all interpolators --- i.e.~the probability to sample an interpolator for which $\mathcal{L_{D}}\left(h\right)\geq \varepsilon$, given the data $\calS$. It states that this relative volume is small ($\delta$) with high probability ($1-\delta$) over the sampling of the data. And $\delta$ can be quite small, since for any $\varepsilon$, we have that $\delta$ decays exponentially fast in $N$
\begin{align*}
    \delta = 2 \exp\rb{-\frac{\varepsilon N + \log \rb{\pteacher}}{6}}.
\end{align*}

\textbf{Proof idea of \lemref{thm: gen of gac with teacher simple}.}
In \appref{app: direct gen bound}, we provide a self-contained proof of \lemref{thm: gen of gac with teacher simple} by noting that the expected number of hypotheses we will consider is $1/\pteacher$, and so we are essentially selecting an interpolating hypothesis from the effective hypothesis class $\mathcal{H}=\{ h_1,\ldots,h_\tau \}$ with $\tau \approx 1/\pteacher$, where the hypotheses in this class are chosen before seeing the training set. 
The sample complexity is thus $\log\left\vert\mathcal{H}\right\vert=\log \tau \approx -\log \pteacher$. 
The only complication is that the stopping time $\tau$ of \GaCtext{} is random and depends on $\calS$. 
But it is enough to bound $\tau$ very crudely (which we do with high probability), as a multiplicative factor to $\tau$ results only in an additive logarithmic factor. 

\begin{remark} 
A similar result can also be proved using PAC-Bayes, as we discuss later in this section. 
Lastly, \label{rem: non interpolating teacher} the same method can be straightforwardly used to extend these results to the case where the teacher NN is not a perfect interpolator, and the \GaCtext{} algorithm is modified to stop when the training error is below some threshold instead of 0 (see \appref{app: non interpolators}).
\end{remark}

\textbf{Relationship to PAC-Bayes} The analysis here may also be seen as a special case of PAC-Bayes analysis \citep{McAllester1999}, which studies the behavior of a ``posterior'' over hypotheses in terms of its KL-divergence to the prior, where here we specialize to a specific posterior $\mathcal{P}_\mathcal{S}$, conditioning on interpolation.

Noting that $\mathrm{KL}(\mathcal{P}_\calS\Vert \mathcal{P})=-\log \mathbb{P}_{h\sim\mathcal{P}}(\emError{h}\!=\!0)\geq -\log \pteacher$, a standard PAC-Bayes bound \citep{langford2001bounds,SimplifiedPAC} will yield  that with the same sample complexity as in \lemref{thm: gen of gac with teacher simple}, 
\begin{equation} \label{eq: PAC}
\bbP_{\calS \sim \mathcal{D}^N} \left(\;\mathbb{E}_{h\sim\mathcal{P}_{\calS}} \left[\exError{h} \right] < \varepsilon \; \right) \geq 1-\delta.
\end{equation}

The subtlety is that \eqref{eq: PAC} bounds the {\em expected} population error $\mathbb{E}_{h\sim\mathcal{P}_\calS}\left[\exError{h} \right]$ for a sample $h\sim\mathcal{P}_\mathcal{S}$ from the posterior, while \lemref{thm: gen of gac with teacher simple} holds {\em with high probability for a single posterior sample}.
That is, \lemref{thm: gen of gac with teacher simple} establishes that not only are random interpolators good on average, but only a small fraction of them are bad\footnote{Using Markov's inequality one can derive from \eqref{eq: PAC} a high probability bound for a single draw from the posterior, but with a sample complexity that depends polynomially rather than logarithmically on the failure probability $\delta$ (see \appref{app: pac bayes rel app}).}.

High-probability PAC-Bayes guarantees on a single draw from the posterior, as in  \lemref{thm: gen of gac with teacher simple}, have been derived by \citet{cantoni2007pac} and simplified by \cite{alquier2021user}. 
 Applying \citet[Theorem 2.7]{alquier2021user}, yields a single-posterior-sample guarantee as in \lemref{thm: gen of gac with teacher simple}, but with a sample complexity of $\Theta\left(\tfrac{-\log(\tilde{p})+\log(1/\delta)}{\epsilon^2}\right)$ that depends quadratically $1/\epsilon$.  
 The reason is that \citeauthor{alquier2021user}'s analysis is more generic and applies also to posteriors with high empirical error. \lemref{thm: gen of gac with teacher simple} can be viewed as a tighter specialization to interpolators, thus allowing a rate of $1/\epsilon$, as in \citeauthor{SimplifiedPAC}'s bounds for the average-over-the-posterior.  
It should be possible to write down a generic PAC-Bayes guarantee that both applies to single draws from the posterior {\em and} yields a $1/\epsilon$ rate for interpolators and then derive \lemref{thm: gen of gac with teacher simple} from it. However, for the sake of completeness, and to give a simplified intuition, we choose instead to present a self-contained specialized proof of the Lemma in \appref{app: direct gen bound}.

\textbf{Occam and Redundancy} A sample complexity of $(-\log \tilde{p})$ should not be surprising, and can also be obtained by an Occam Razor / Minimum Description Length learning rule $\mathrm{MDL}_{\calP} (\calS) = \arg\max_{\calL_{\calS} (h) = 0} \calP (h)$ (\citet{blumer1987occam}, and see also Section 7.3 in \citet{shalev2014understanding}).
Here, we discussed how the same sample complexity is obtained by a single draw from the posterior, as in \GaCtext. 
More interesting is how, starting from a uniform prior $\calP_{\btheta} (\btheta)$ over parameters, we end up with an informative induced prior $\calP (h)$ over hypotheses, which has high $\tilde{p}$ and thus low sample complexity. 
The key here is redundancy in the parameterization.
In \appref{app: MDL} we review the general principle of how non-uniform redundancy in the parameterization can induce non-uniform informative priors $\calP (h)$ and thus low sample complexity. 
In the next sections we see how this plays specifically for NNs, analyzing $\tilde{p}$ under the prior induced by a uniform choice of NN parameters.

\section{Quantized Nets Sample Complexity}\label{sec:quant_nets}

The generalization bound in \lemref{thm: gen of gac with teacher simple} depends on the effective sample complexity $\SC \triangleq -\log\rb{\pteacher}$. 
In this section, we derive an upper bound on $\SC$ for quantized multi-layer Fully Connected (FC) Neural Networks (NNs) with a single binary output, with and without additional per-node scaling.
\begin{definition}[Vanilla FC]  \label{def: vanilla fcn}
For a depth $L$, widths $D=\rb{d_1,\dots,d_L}$, and activation function $\sigma:\mathbb{R}\rightarrow\mathbb{R}$, a fully connected NN is a mapping $\btheta \mapsto h^{FC}_{\btheta}$ from parameters
\begin{align*}
\left\{ \btheta = \left\{ \mathbf{W}^{\left(l\right)} ,\mathbf{b}^{\left(l\right)} \right\}_{l=1}^{L} 
\,\middle|\,
\mathbf{W}^{\left(l\right)} \in \mathbb{R}^{d_{l}\times d_{l-1}}, \mathbf{b}^{\left(l\right)} \in \mathbb{R}^{d_{l}} \right\}
\end{align*}
defined recursively, starting with $f^{(0)} \left(\mathbf{x}\right) = \mathbf{x}\,, $ as
\begin{align*} 
    \forall l\in\bb{L\!-\!1}: \; f^{(l)} \left(\mathbf{x}\right) 
    &= \sigma\!\left(\mathbf{W}^{\left(l\right)} f^{(l-1)} \left(\mathbf{x}\right) + \mathbf{b}^{\left(l\right)}\right)\\
    h_{\btheta}^{\mathrm{FC}} \rb{\mathbf{x}} 
    &= \sign\!\rb{\mathbf{W}^{\left(L\right)} f^{(L-1)} \left(\mathbf{x}\right) + \mathbf{b}^{\left(L\right)}}\!.
\end{align*}
The total parameter count is $M(D)=\sum_{l=1}^L d_l(d_{l-1}+1)$. We denote the class of all fully connected NNs as $\Hfc{D}$.
\end{definition}

As we will show, considering NNs in which each neuron is multiplied by a scaling parameter can significantly improve the bound. 
This architecture modification is common in empirical practices, e.g. batch-normalization \citep{ioffe2015batch}, weight-normalization \citep{salimans2016weight}, and certain initializations \citep{zhang2019fixup}.
We formally define this model in the following definition.

\begin{definition}[Scaled-neuron FC]  \label{def: scaled neuron fcn}
For a depth $L$, widths $D=\rb{d_1,\dots,d_L}$, and activation function $\sigma:\mathbb{R}\rightarrow\mathbb{R}$, a scaled neuron fully connected neural network is a mapping $\btheta \mapsto h^{\mathrm{SFC}}_{\btheta}$ from parameters
\begin{align*}
 \btheta = \left\{ \mathbf{W}^{\left(l\right)},\mathbf{b}^{\left(l\right)},\bgamma^{\left(l\right)}  \right\}_{l=1}^{L} \,,
\end{align*}
where
$
    \mathbf{W}^{\left(l\right)} \in \mathbb{R}^{d_{l}\times d_{l-1}}, \mathbf{b}^{\left(l\right)} \in \mathbb{R}^{d_{l}}, \bgamma^{\left(l\right)} \in \mathbb{R}^{d_{l}}\,,
$
defined recursively, starting with $f^{(0)} \left(\mathbf{x}\right) = \mathbf{x}\,, $ as 
\begin{align*} 
    \forall l\in\bb{L-1} \; f^{(l)} \left(\mathbf{x}\right) &= \sigma\left(\bgamma^{\rb{l}}\odot{\mathbf{W}^{\left(l\right)} f^{(l-1)} \left(\mathbf{x}\right)} + \mathbf{b}^{\left(l\right)}\right)\\
    h_{\btheta}^{\mathrm{SFC}} \rb{\mathbf{x}} &= \sign\rb{\mathbf{W}^{\left(L\right)} f^{(L-1)} \left(\mathbf{x}\right) + \mathbf{b}^{\left(L\right)}} .
\end{align*}
The total parameter count is $M(D)=\sum_{l=1}^L d_l(d_{l-1}+2)$. We denote the class of all scaled neuron fully connected NNs as $\Hsfc{D}$.
\end{definition}

We consider {\bf $Q$-quantized} networks where each of the parameters is chosen from a fixed set $\mathcal{Q}\subset\mathbb{R}$ such that $0\in\mathcal{Q}$ and $\abs{\mathcal{Q}}\leq Q$.  This can be the set of integers $\{ -\frac{Q}{2},-\frac{Q}{2}+1,\ldots,(\frac{Q}{2}-1)\}$ for even $Q$, or the set of numbers representable as $\log_2 Q$-bit floats (for, e.g.~$\log_2 Q = 32$).  Fully connected quantized NNs thus have parameters $\btheta \in \mathcal{Q}^M$ corresponding to a complexity $C = M \log Q$ (from the classic log cardinality bound, see Appendix \ref{app: finite hypo PAC}).  

We consider a teacher $h^{\star}=h_{\btheta^{\star}}$ that is a $Q$-quantized network of some depth $L$
and small widths $D^{\star}=(d^{\star}_1,\ldots,d^{\star}_L)$, and a wider student of the same depth $L$ but widths $D > D^{\star}$. 
For the student, we consider a {\bf uniform prior over $Q$-quantized parameterizations}, i.e.~$\btheta \sim \textrm{Uniform}\left(Q^{M(D)}\right)$. 
In other words, to generate $h_{\btheta}\sim\mathcal{P}$, each weight (and bias) in the NN is chosen independently and uniformly from $\mathcal{Q}$.   

\begin{figure*}[t]
\centering     
\subfigure[Teacher network]{\label{fig: scaled teacher} \includegraphics[width=0.31\textwidth, trim={0cm 5cm 0.8cm 3cm},clip]{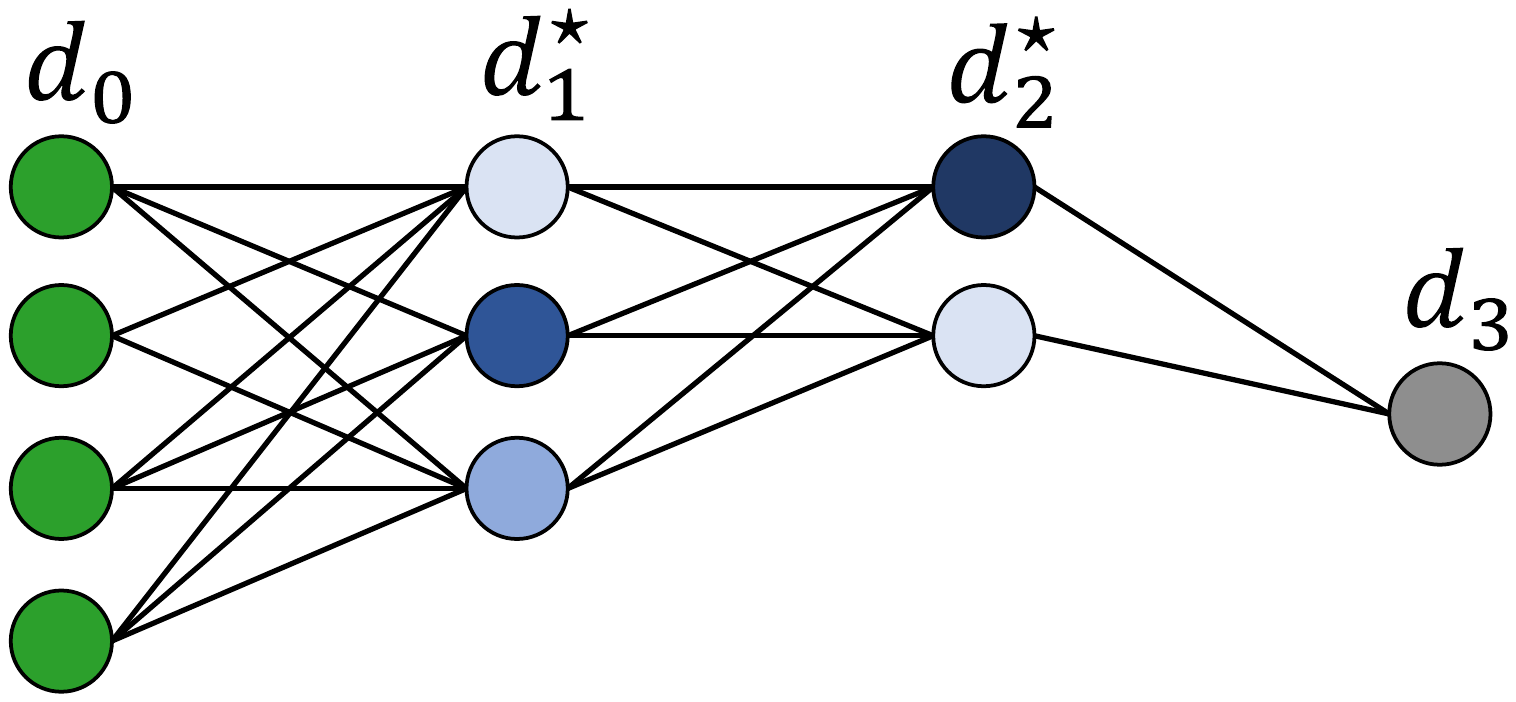}}
\hfill
\subfigure[Vanilla student network]{\label{fig: sparse student} \includegraphics[width=0.31\textwidth, trim={0cm 5cm 0.8cm 3cm},clip]{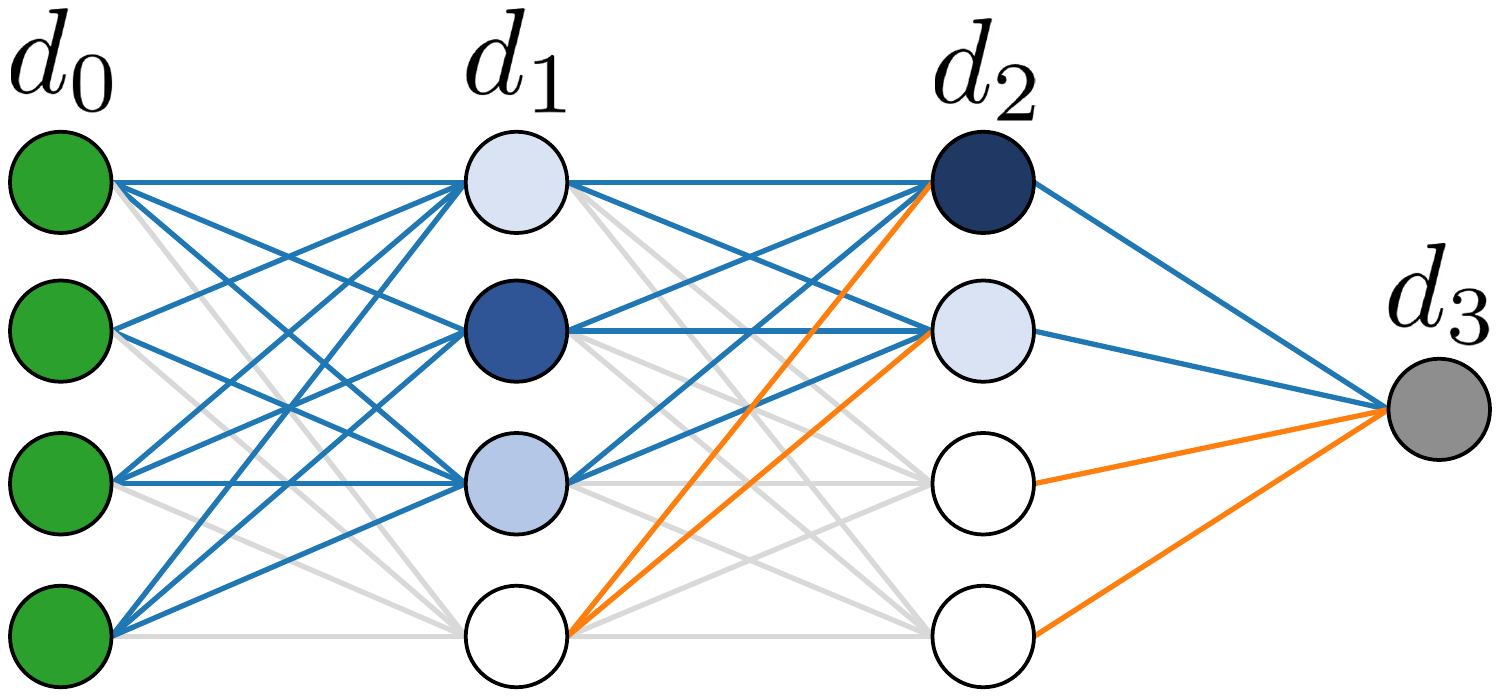}}
\hfill
\subfigure[Scaled neuron student network]{\label{fig: scaled student} \includegraphics[width=0.31\textwidth, trim={0cm 5cm 0.8cm 3cm},clip]{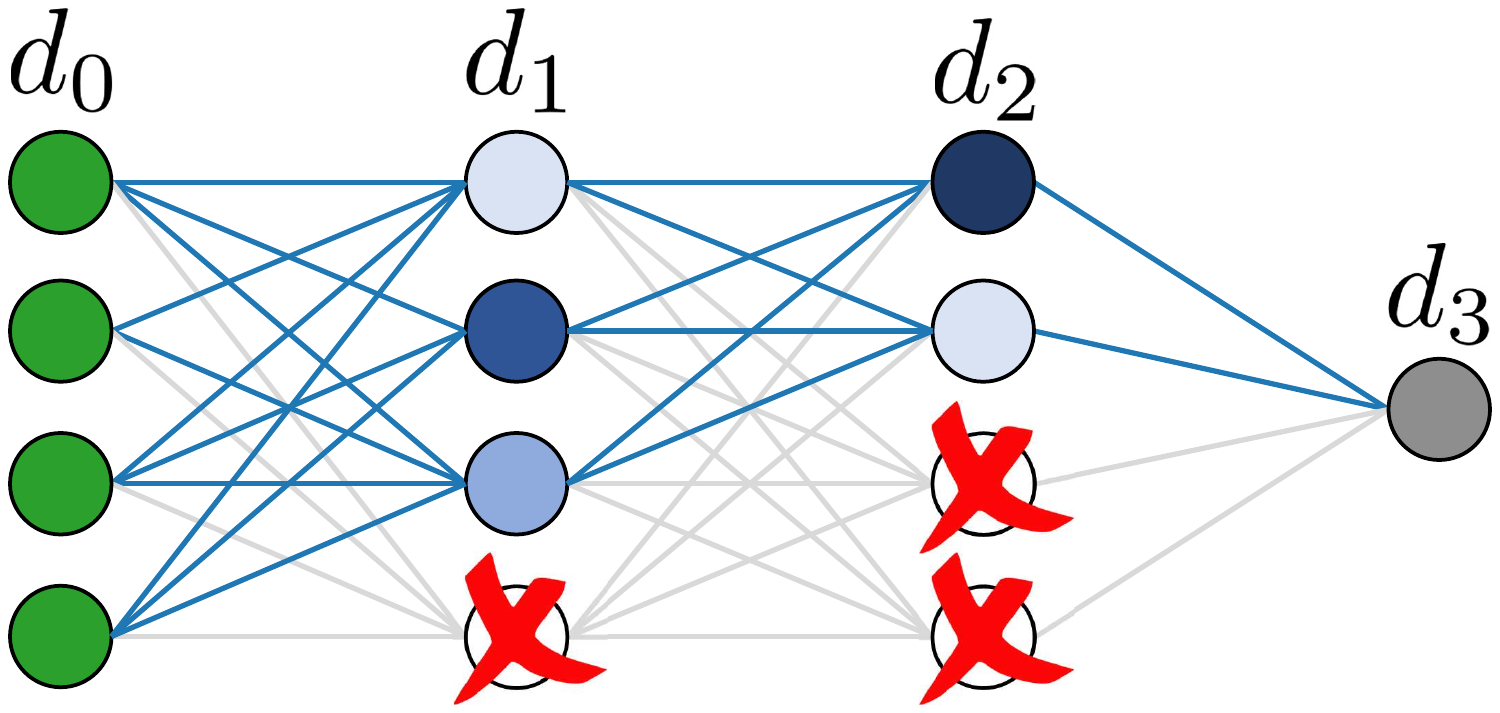}}
\caption{\textbf{Illustration of vanilla and scaled neuron three-layer quantized teacher and student neural networks.} 
Note that the visualization does not show the bias units. 
The proof of \thmref{lem:int_quantized} relies on counting student networks which are functionally equivalent to the teacher network. 
\figref{fig: scaled teacher} depicts a narrow teacher. 
In \figref{fig: sparse student}, we visualize a FC student network that replicates the teacher by zeroing out all outgoing weights of any neuron that does not exist in the teacher. 
Specifically, the blue edges are weights identical to the teacher, and the orange edges are set to zero. 
Therefore, the white neurons do not affect the network output.
In \figref{fig: scaled student}, we visualize an SFC student network that replicates the teacher by setting the scaling parameter to zero (each zero marked with a red `x') for any neuron that does not exist in the teacher. In both cases, the gray edges do not affect the function.
In this specific example, we can see how the redundancy is higher in SFC than in the vanilla FC network, hinting at better generalization capabilities. 
\label{Fig: teacher-student illustration}}
\end{figure*}

\subsection{Main Results} 
Using the definitions above, we can state the following.

\begin{theorem}[Main result for fully connected neural networks]
\label{lem:int_quantized}

For any activation function such that $\sigma(0)=0$, depth $L$, $Q$-quantized teacher with widths $D^{\star}$,  student with widths $D>D^{\star}$, $d_0^{\star}\triangleq d_0$, and prior $\mathcal{P}$ uniform over $Q$-quantized parameterizations, we have that: 
\begin{enumerate}
    \item For Vanilla Fully Connected Networks:
    \begin{align} \label{eq: PC FCN}
        \SC \leq \BSCM{FC} \triangleq \rb{\sum_{l=1}^{L}{\rb{d_{l}^{\star}d_{l-1}+d_{l}^{\star}}}}\log Q\,.
    \end{align}
    \item For Scaled Neuron Fully Connected Networks:
    \begin{align}\label{Eq: PC FCN better} 
        \SC \leq \BSCM{SFC} \triangleq \rb{\sum_{l=1}^{L}\rb{d_{l}^{\star}d_{l-1}^{\star}+2d_{l}}} \log Q\,.
    \end{align}
    
\end{enumerate}
And, by \lemref{thm: gen of gac with teacher simple}, $N=(\SC+3\log 2/\delta)/\varepsilon$ samples are enough to ensure that for posterior sampling (i.e.~\GaCtext{}), $\mathcal{L}(\mathcal{A}_\mathcal{P}(\calS))\leq\varepsilon$ with probability $1-\delta$ over $\calS\sim\calD^N$ and the sampling.
\end{theorem}
\begin{remark} \label{rem: volume quantized}
From \corref{cor:volume} we deduce that for a given training set, the volume of `bad' ($\varepsilon$) interpolators is `small' ($\delta$) with high probability ($1-\delta$) with this sample complexity.
\end{remark}

\textbf{Proof idea.} The idea is simple and centers on counting a sufficient number of constraints on the parameters of the student network to ensure it is TE. 
A fundamental illustration of this concept is provided in the caption of Figure \ref{Fig: teacher-student illustration}. 
Full proof is in \appref{app:quant_proofs}. 

\subsection{Discussion: Comparing Sample Complexities} 

 To understand the sample complexity bound $O(\BSCM{FC})$ and $O(\BSCM{SFC})$ implied by Theorem \ref{lem:int_quantized}, let us first consider the complexities (number of bits, or log cardinalities) of the teacher and student models:
\begin{align*}
    & C^{\star} =\rb{\sum_{l=1}^{L} \rb{d_l^{\star} d_{l-1}^{\star} + k d_{l}^{\star}}}\log Q \\
    & \C=\rb{\sum_{l=1}^{L} \rb{d_l d_{l-1} + k d_{l}}}\log Q
\end{align*}
where $k=1$ for Vanilla FC Networks and $k=2$ for Scaled FC Networks.

Either way, the dominant term is the quadratic term $\sum_l d_l d_{l-1}$.  Lacking other considerations, the sample complexity of learning with the student network would be $C$.  However, we see here that thanks to the parameterization and prior, the student network implicitly adapts to the complexity of the teacher, with $\hat{C} \ll C$ when $D^{\star} \ll D$ and in any case $C^{\star} \leq \BSCM{} \leq C$.   With Vanilla FC Networks, the sample complexity $\BSCM{FC}$, although smaller than $C$, is still significantly larger than the complexity $C^{\star}$ of the teacher, which is what we could have hoped for.  The quadratic (dominant) terms in $\BSCM{FC}$ are roughly geometric averages of terms from $C$ and $C^{\star}$, and so we have that, very roughly, $\BSCM{FC} \approx \sqrt{C C^{\star}}$. We can improve this using scaling, which creates more redundancy since zero scales can deactivate entire units.  Indeed, for Scaled Fully Connected Networks, we have that $\BSCM{SFC} = C^{\star} + \sum_{l=1}^L d_l \ll \BSCM{FC} \ll C$ (when $D^{\star} \ll D$).  We still pay a bit for the width of the student, but only linearly instead of quadratically.  In particular, even if the width of the student is quadratic in the width of the teacher, the sample complexity of learning by sampling random NNs is almost the same as that of using a much narrower teacher.

\textbf{Minimum widths.} We note that, just as in the minimum description length example in the previous section, an explicit narrowness prior could have of course been fully adaptive to the width of the teacher and ensured learning with the ideal sample complexity $C^{\star}$.  
E.g., this is achieved by allowing the student to choose the width of each layer, and using the Occam rule 
\begin{align}
    \min_{D'\leq D, \btheta\in \mathcal{Q}^{M(D')} \text{with widths}\, D'} M(D') \;\;\text{s.t.~}\mathcal{L}_S(h_{\btheta})=0.
\end{align}
But from Theorem \ref{lem:int_quantized} we see that even without such an explicit bias, choosing weights uniformly induces significant inductive bias toward narrow networks.

\textbf{Maximum sparsity.} It is also insightful to compare this to using an explicit sparsity bias, e.g.~with an Occam rule of the form:
\begin{align}\label{eq:sparse}
    \min_{\btheta\in \mathcal{Q}^{M(D)}} \norm{\btheta}_0 \;\;\text{s.t.~}\mathcal{L}_S(h_{\btheta})=0.
\end{align}
The sparsity-inducing rule \eqref{eq:sparse} would have the following bound for the effective sample complexity
\begin{align} \label{eq:sparseC}
    \hat{C}^{\textrm{sparse}} \triangleq O \rb{ \ccfct \rb{ \log\left( \ccfct +d_{0} \right) + \log Q }}.
\end{align}
See \appref{app:sparse} for a derivation of this equation. Note that $\hat{C}^{\textrm{sparse}}$ in \eqref{eq:sparseC} does not depend on the size of the student, but rather only on the size of the teacher, which is smaller, that is $M\rb{D^{\star}}\ll M\rb{D}$. For comparison, recall our previous bound for posterior sampling, in the case of FC scaled-neuron networks from \eqref{Eq: PC FCN better}:
    \begin{align*}
        \BSCM{SFC} &= \rb{\sum_{l=1}^{L}\rb{d_{l}^{\star}d_{l-1}^{\star}+2d_{l}}} \log Q
        \\
        &= O\rb{\C^{\star}+ \sum_{l=1}^{L}2d_{l}\log Q}\,.
    \end{align*}
When $\log Q = O(1)$, we can rewrite $\hat{C}^{\textrm{sparse}}$ as follows:
\begin{align*}
    \hat{C}^{\textrm{sparse}} = O\rb{\C^{\star}\rb{1+\frac{\log M\rb{D^{\star}}}{\log{Q}}}}=O\rb{\C^{\star}\log\C^{\star}}\,,
\end{align*}
that is $\hat{C}^{\textrm{sparse}}$ has an additional multiplicative factor bounded by $\log\C^{\star}$, and so can be worse than the bounds for posterior sampling.
Specifically, in the regime where $2 \sum_{l=1}^{L}{d_l}\log Q \le \frac{\log M\rb{D^{\star}}}{\log Q }$, we have that $\BSCM{SFC}<\hat{C}^{\textrm{sparse}}$. 
For example, when $\sum_{l=1}^{L}{d_l}\log Q=O\rb{\C^{\star}}$, we have that $\BSCM{SFC}$ is suboptimal with a factor of 2 with respect to $\C^{\star}$, whereas $\hat{C}^{\textrm{sparse}}$ is off by a larger factor of $\log\C^{\star}$.

\textbf{Minimum norm.} One might also ask whether a similar adaptation to teacher width can be obtained by regularizing the norm of the weights.  Indeed, \cite{neyshabur2015norm,golowich2018size} obtained sample complexity bounds that depend only on the $\ell_2$ norm $\norm{\btheta}$ of the learned network, without any dependence on the width of the student (but with an exponential dependence on depth!).  These guarantees are not directly applicable in our setting, since applying them to get guarantees on the misclassification error we study here requires bounding the margin.  Even with a discrete teacher, without further assumptions on the input distribution, we cannot ensure a margin.  If we did consider only integer inputs and integer weights, we could at least ensure a margin of 1.  In this case: on one hand, the norm-based sample complexity would scale as $Q^{O(L)}$, i.e.~exponential in the $L \log Q$ dependence of $\BSCM{}$.  On the other hand, the norm-based guarantee would not depend at all on the student widths $D$, while even $\BSCM{SFC}$ increases linearly with the student widths.

\subsection{Extension to convolutional neural networks}
In \appref{app:quant_proofs}, we extend Theorem \ref{lem:int_quantized} to convolutional neural networks (CNN) and convolutional neural networks where each channel is multiplied by a learned parameter (SCNN). 
Specifically, we show that for quantized convolutional networks, we get similar bounds on the sample complexity $\BSC$ with channel numbers substituting layer widths.
\begin{align*}
    \BSCM{CNN}=&\left(d_s + 1 + \sum_{l=1}^L \rb{k_l c_{l}^{\star}c_{l-1}  + c_{l}^{\star}} \right) \log Q\\
\BSCM{SCNN}=& \left( d_{s}^{\star} + 1 + \sum_{\lconv=1}^{L}\rb{{k_{\lconv}c_{\lconv}^{\star}c_{\lconv-1}^{\star} +2c_{l}}} \right) \log Q\,,
\end{align*}
where $d_s$ is the number of neurons in the last convolutional layer (i.e., the width of the last layer which is a fully connected layer with a single output), $k_l$ and $c_l$ are the $l^{\text{th}}$ layer's kernel size and number of channels.
See \appref{app:quant_proofs} for precise definitions of the model and statement of the result. Importantly, similar to the fully connected architecture we observe that  $\BSCM{SCNN}$ is 
again close to the teacher's complexity, with only a weak dependence on the student's channel numbers. The proof here is analogous to the proof of \thmref{lem:int_quantized}, where neurons are replaced with channels.

\textbf{Implications to realistic benchmarks.}
Currently, no NN has achieved zero test error on real-world datasets, even for simple ones like MNIST. 
Therefore, the size of the teacher NN is unknown to us when facing practical applications. 
However, we can use the dimensions of common NN architectures and datasets to approximate the necessary size of the teacher to obtain meaningful generalization bounds using our results. 
For example, suppose that we substitute the actual sizes of a training set and network into our bound
$$
N = \frac{ \hat{C} + 3 \log{ \left( \frac{2}{\delta} \right) } }{\varepsilon} .
$$
We can deduce the size of the teacher required to satisfy this equation. 
For example, using the size of the ImageNet dataset, $\delta=0.05$, and the actual test error of some CNNs, we can estimate the required size of the teacher.
For simplicity of calculation, we assume that each layer of the teacher NN has exactly $\alpha$ channels compared to the same layer in the student NN, where $0< \alpha < 1$. 
In the following table, we show the required width reduction $\alpha$ and the number of parameters in the resulting teacher NN. 
We use the smallest quantization level (2bit) for which the NN accuracy remains near the FP32 accuracy (less than ~0.5\% degradation, from \citep{liu2022nonuniform}).

\begin{table}[h]
\centering
\begin{tabular}{|c|c|c|c|}
\hline
Architecture & $\varepsilon$ & $\alpha$ & \#parameters \\ \hline
ResNet18     & 0.3        & 0.125    & $\sim 241k$  \\ \hline
ResNet50     & 0.25       & 0.05     & $\sim 159k$  \\ \hline
\end{tabular}
\caption{Approximate relative width reduction and number of parameters in the teacher required to obtain meaningful bounds using standard ResNet architectures and the ImageNet dataset.}
\label{tab: numberical examples}
\end{table}

Table~\ref{tab: numberical examples} shows that our bound is consistent with a narrow teacher of non-trivial size and widths.
For example, in ResNet18 the resulting channel numbers in the teacher layers are $[3,8,...,8,16,..., 16, 32,..., 32, 64,....64]$, where ‘3’ counts the input channels and the rest are the following hidden layers.
This architecture can implement highly complex non-linear functions. 
Note that we cannot directly validate the existence of this teacher, as standard optimization methods may not be capable of finding such a solution efficiently, even if it exists (without over-parameterization, it is much harder to find global minima).

\section{Continuous Nets Sample Complexity}
\label{sec: continuous}
So far, we focused on quantized uniform priors. However, per-layer continuous spherical priors (e.g., Gaussian) are also quite common in practice and theory. Therefore, in this section, we show how to extend our results beyond the quantized case into a continuous setting, for the special case of two-layer NNs without bias and with the leaky rectifier linear unit (LReLU, \citet{Maas2013RectifierNI}) activation function. 
Formally, let $\student, \teacher$ be fully connected (Definition \ref{def: vanilla fcn}) two layer NNs with input dimension $d_{0}$, output dimension $d_{2}=1$ and hidden layer dimensions $d_{1}$ and $d_{1}^{\star}$, respectively.
Explicitly:
\begin{equation*}
\student\left(\mathbf{x}\right)=\sign\rb{\mathbf{W}^{\rb{2}}\sigma\left(\mathbf{W}^{\rb{1}}\mathbf{x}\right)}
\end{equation*}
\begin{equation*}
\teacher\left(\mathbf{x}\right)=\sign\rb{\mathbf{W}^{\rb{2}}_{\star}\sigma\left(\mathbf{W}^{\rb{1}}_{\star}\mathbf{x}\right)}
\end{equation*}
where
\[
\mathbf{W}^{\rb{1}}=\left[\mathbf{w}_{1},\dots,\mathbf{w}_{d_{1}}\right]^{\top}\in\mathbb{R}^{d_{1}\times d_{0}}\,,\,\mathbf{W}^{\rb{2}}\in\mathbb{R}^{d_2\times d_{1}}\,,
\]
\[
\mathbf{W}^{\rb{1}}_{\star}=\left[\mathbf{w}_{1}^{\star},\dots,\mathbf{w}_{d_{1}^{\star}}^{\star}\right]^{\top}\in\mathbb{R}^{d_{1}^{\star} \times d_{0}}\,,\,\mathbf{W}^{\rb{2}}_{\star} \in \mathbb{R}^{d_2\times d_1^{\star}}\,,
\]
and $\sigma\left(\cdot\right)$ is the common LReLU with parameter $\rho\notin \{0,1\}$.

As in the previous sections, our goal is to obtain generalization guarantees by lower bounding $\pteacher$ and then combining this results with \lemref{thm: gen of gac with teacher simple}. 
To this end, we first need to define some prior on the hypotheses. 
\begin{assumption}[Prior over parameters, continuous setting]\label{asm:continuous prior}
    Suppose that the weights of $\student$ are random such that each row of the first layer, $\mathbf{w}_i$, is independently sampled from a uniform distribution on the unit sphere $\mathbb{S}^{d_0 -1}$, and the second layer $\mathbf{W}^{\rb{2}}$ is sampled uniformly\footnote{Sampling $\mathbf{W}^{\rb{2}}$ from the unit sphere is equivalent to sampling it from a Gaussian distribution. In fact, any spherically symmetric distribution in $\mathbb{R}^{d_1}$ will suffice, as it amounts to scaling of the output without affecting the classification.} from $\mathbb{S}^{d_1 - 1}$.
    Both $\mathbf{W}^{\rb{(1)}}$ and $\mathbf{W}^{\rb{2}}$ are independent of the teacher and data.
\end{assumption}

\begin{figure}[t]
\centering     
\subfigure[]{\label{fig:a}\includegraphics[width=0.2\textwidth, trim={4cm 0cm 3.5cm 0cm},clip]{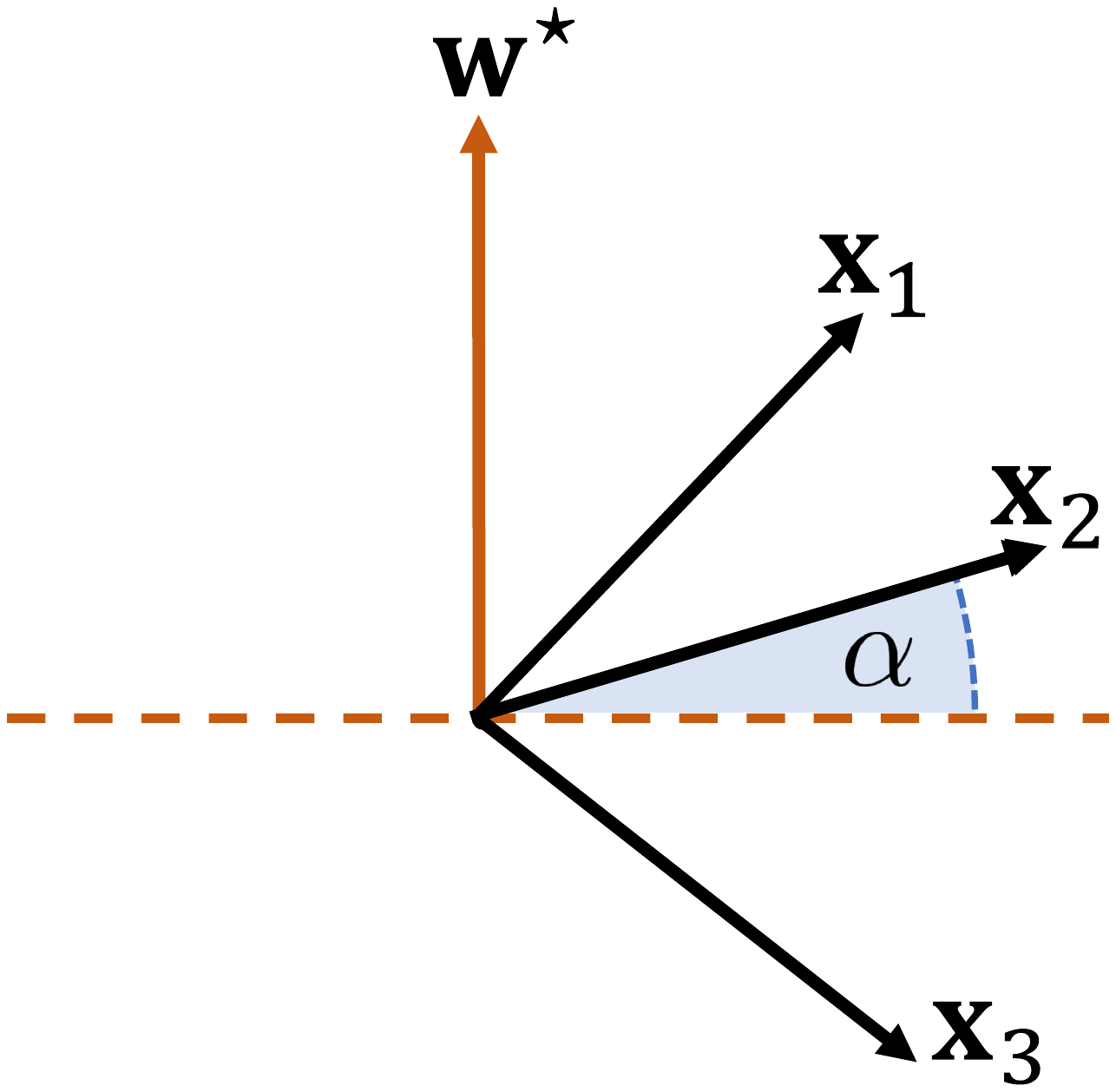}}\quad\quad~%
\subfigure[]{\label{fig:b}\includegraphics[width=0.2\textwidth, trim={4cm 0cm 3.5cm 0cm},clip]{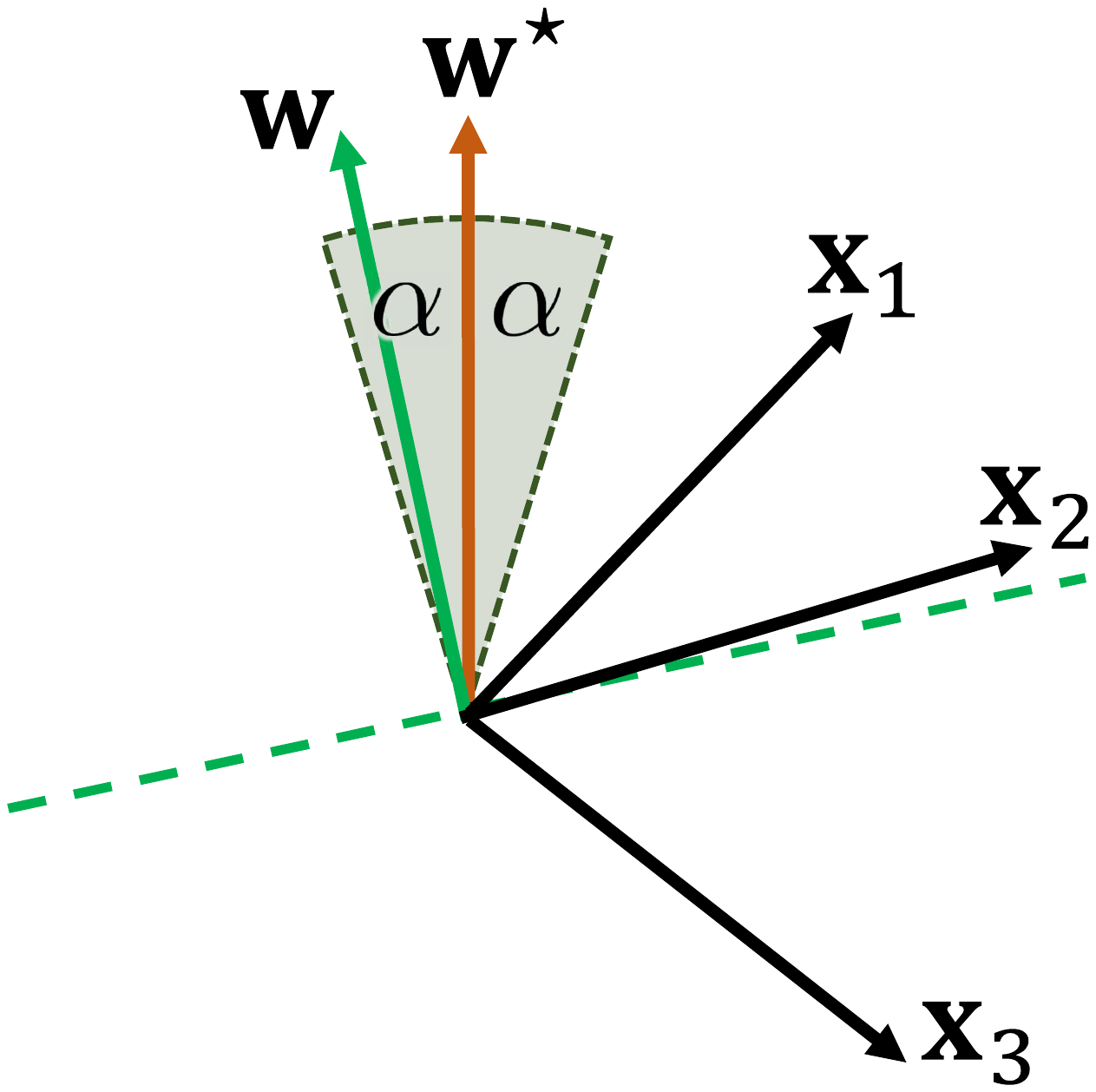}}
\caption{A two-dimensional illustration of the first layer angular margin. In \ref{fig:a}, we show how the angle $\alpha$ is defined for a single $\vect{w}_{i}^{\star}$. Note that $\alpha$ is defined as the minimal angle when considering all rows of $\vect{W}_{\star}^{\rb{1}}$. In \ref{fig:b}, we illustrate how $\alpha$ margin creates a cone around $\vect{w}_{i}^{\star}$ in which any $\mathbf{w}_i$ results in the same activation pattern as $\mathbf{w}_{i}^{\star}$ (i.e., as the teacher) on a training set.
}
\label{fig:margin}
\end{figure}

To extend the notion of teacher-equivalence from the quantized setting to the continuous setting, we assume an ``angular margin" exists between the training set $\calS$ and the teacher, similarly to \citet{soudry2017exponentially}.

\begin{definition}[First layer angular margin]\label{def: First layer angular margin}
    For any training set $\calS=\{\vect{x}_n \}_{n=1}^N$, we say that $\calS$ has \emph{first layer angular margin $\alpha$} w.r.t. the teacher if
    \begin{equation}
    \forall i\in[d_1^{\star}]\,,\,n\in[N]:\left|\frac{\mathbf{x}_n^{\top}\mathbf{w}_{i}^{\star}}{\left\Vert \mathbf{x}_n\right\Vert_2 \left\Vert \mathbf{w}_{i}^{\star}\right\Vert_2 }\right|>\sin\alpha \,.\label{eq: angular margin}
    \end{equation}
\end{definition}
In words, we say that $\calS$ has \emph{first layer angular margin $\alpha$} if all datapoints $\vect{x}_n$ are at an angle of at least $\alpha$ from any hyperplane induced by a row $\vect{w}_{i}^{\star}$ of the matrix $\mathbf{W}^{\rb{1}}_{\star}$. Here, the rows of $\mathbf{W}^{\rb{1}}_{\star}$ represent the normals to the hyperplanes. As illustrated in Figure \ref{fig:margin}, we require some angular margin $\alpha>0$ to guarantee that the first layer of the student network can be within a certain angular margin of the first layer of the teacher network and still achieve accurate classification on the training set. This assumption prevents degenerate neurons in the first layer of the teacher network.

Similarly, for the output of the teacher network, we define the second layer angular margin.

\begin{definition}[Second layer angular margin]\label{def: Second layer angular margin}
    For any training set $\calS=\{\vect{x}_n \}_{n=1}^N$, we say that $\calS$ has \emph{second layer angular margin $\beta$} w.r.t. the teacher if
    \begin{equation}\label{eq:output angular margin}
    \forall n\!\in\!\left[N\right]\!:
    \left|
    \frac{
    \mathbf{W}^{\rb{2}}_{\star}
    \sigma
    \big(\mathbf{W}^{\rb{1}}_{\star}\mathbf{x}_n\big)
    }{
    \big\Vert\mathbf{x}_n\big\Vert_2 \big\Vert \mathbf{W}^{\rb{2}}_{\star}
    \big\Vert_2 }
    \right|
    >\!
    \sqrt{d_1
    \!
    \left(1\!+\!\rho^2\right)}\sin\beta
    .
    \!
    \end{equation}
\end{definition}
In essence, this ensures some margin in the output of the teacher network. With this definition, our main assumption for the continuous case is stated below.
\begin{assumption}\label{asm: angular margin}
    Let $\alpha<\beta\in\rb{0,\frac{\pi}{2}}$. There exists $\lambda\in\rb{0,1}$ such that with probability at least $1-\lambda$ over $\calS \sim \mathcal{D}^N$, $\calS$ has first layer angular margin $\alpha$ (Definition \ref{def: First layer angular margin}) and second layer angular margin $\beta$ (Definition \ref{def: Second layer angular margin}).
\end{assumption}
\begin{remark}
    Note that $\alpha$ is the minimal margin of all hidden neurons, while $\beta$ is the margin of the single network output. Thus, intuitively, $\beta$ is usually larger than $\alpha$.
    For Gaussian data, we show empirically in \figref{fig:alpha-beta-sim} in \appref{app:cont_proofs}, that the assumption $\beta > \alpha$ holds with high probability.
\end{remark}

This assumption allows us to extend the results from the previous section to a continuous setting (proof in  \appref{app:cont_proofs}).

\begin{theorem} [Interpolation of Continuous Networks] \label{thm: interpolation probability of continuous networks}
Assume that $\hat{p}_{\calS} < \frac{1}{2}$ a.s. and $d_0 \gg d_1^{\star} \gg 1$ \footnote{This assumption is used to simplify the bound. A non-asymptotic version that does not require this assumption is in \appref{app:cont_proofs}.}. 
Then under \asmref{asm: angular margin}, for any $\varepsilon, \delta \in \rb{0,1}$ we have 
$$\bbP_{\calS \sim\calD^N, h \sim \mathcal{P}_{\mathcal{S}}} \rb{\exError{h} \leq \varepsilon} \ge 1-\delta\ - \lambda,$$
whenever $N$ is larger than the sample complexity
\begin{align*}
\frac{ \BSCM{cont} + 2\log \rb{\BSCM{cont}} + 4\log \rb{\frac{8}{\delta}}}{\varepsilon}\,, 
\end{align*}
with $\gamma = \arccos{\frac{\cos{\beta}}{\cos{\alpha}}}$ and 
\begin{align*}
    \BSCM{cont} &= - d_{1}^{\star} d_{0} \log\left(\sin\left(\alpha\right)\right) - d_{1} \log\left(\sin\left(\gamma\right)\right) \\
    & \quad + \frac{1}{2} d_1^{\star} \log \rb{d_0} + O\left( d_1^{\star} + \log{\left(d_1\right)}\right) \,.
\end{align*}
\end{theorem} 

\begin{remark}
Since \asmref{asm: angular margin} gives a positive margin only in high probability, we use a different generalization guarantee than the one in \lemref{thm: gen of gac with teacher simple} which takes into account the interpolation probability $\hat{p}_\calS$, instead of $\pteacher$.
See \appref{app: nonuniform gen bound} for more details.
\end{remark}
\begin{remark}
Although \asmref{asm: angular margin} may be natural in some cases\footnote{Realistic data many times has some intrinsic margin. For example, there is a low probability for (semantic) `mixings' of dogs and birds from a distribution of natural images of dogs and birds (with a single animal in every image). }, 
in other cases the margins $\alpha, \beta$ may decay with $N$ or with the network's dimensions.
Therefore, even if the assumption holds, for the generalization bound in \thmref{thm: interpolation probability of continuous networks} to remain meaningful, we need the margin decay rate to be sufficiently slow.\footnote{For example, for Gaussian data with a random teacher it is possible to derive a lower bound in high probability on the first layer angular margin $\alpha$ and show that it decays as $1/(N d_0 d_1^\star)$, so it adds only a multiplicative $\log(N d_0 d_1^\star)$ factor to the bound.}
\end{remark}

\textbf{Implications.} Our results in this section rely on a margin assumption instead of quantized weights. This margin assumption can improve the quantized bound given the data $\calS$ has a large enough margin from the teacher. However, for a generic input distribution without any pre-set margin (e.g., standard Gaussian) we observed empirically the resulting margin is near the numerical precision level, so the resulting bound is not better than the quantized approach.

\section{Related work}

\textbf{Random interpolating neural networks.} \citet{chiang2023loss} empirically studied a few gradient-free algorithms for the optimization of neural networks, suggesting that generalization is not necessarily due to the properties of the SGD. 
For the \GaCtext{} algorithm, they also empirically investigated the effect of the loss on the generalization. We focus on the 0-1, rather than some surrogate loss function.
\citet{chiang2023loss} focused only on small-scale datasets since it is not possible to run their experiments when the training set is larger than a few dozen samples.
\citet{theisen2021good} provided a theoretical analysis of random interpolating linear models. 
In contrast, our work focuses on deeper models with non-linear activation functions, and our teacher assumption relies on depth, such that it is possible to have many zeroed-out hidden neurons and obtain generalization guarantees that stem from the redundancy in parameterized deep networks.
\citet{vallepérez2019deep,mingard2021sgd} used a connection between neural networks and Gaussian processes to model random interpolating networks, which usually requires the width of the network to be infinite, in contrast with our finite width setting.
\citet{teney2024neural} empirically investigated random fully connected neural networks without conditioning on the interpolation of a training set. They examined their spectral properties and found a bias towards simple functions (according to various metrics). \citet{Berchenko_2024} showed generalization results for uniformly sampled trees learning Boolean functions, and discuss the simplicity bias of random deep NNs.
In contrast, we proved an explicit generalization bound for typical interpolating deep neural networks.

\textbf{Redundancy in neural networks.} The Lottery Ticket Hypothesis \citep{frankle2019lottery} suggests that for any random neural network, there exists a sparse sub-network capable of competitive generalization with the original. A related hypothesis was conjectured by \citet{ramanujan2020whats} and later proven by \citet{malach2020proving} and states that it is possible to prune an initialized neural network without significantly affecting the obtained predictions.
In contrast, our work, while also focusing on random neural networks, is oriented towards providing assurances regarding the generalization of these random networks, when conditioned on the event of perfectly classifying the training data. Also, we operate under the assumption of a narrow teacher, which is sparse in the number of neurons --- in contrast to the sub-network in the Lottery Ticket Hypothesis which is sparse in the number of weights.

\section{Discussion} \label{sec: discussion}
\textbf{Summary.} In this work, we examined the generalization of samples from the NN posterior with the 0-1 loss \eqref{eq: posterior}. We proved that even when using a uniform (or `uniform-like') prior for the parameters, typical samples from the NN posterior are biased toward low generalization error --- if there exist sufficiently narrow NN teachers.

\textbf{Implications of narrowness.} 
Assuming a narrow teacher may sometimes limit the class of possible target functions. 
For example, in a fully connected NN, if the first hidden layer width is smaller than the input width $\rb{d_{0}>d_{1}^{\star}}$ then the teacher NN must be constant on the $d_{0}-d_{1}^{*}$ nullspace of the first weight layer.
However, this is not an issue for a convolutional NN, since typically the number of input channels is small (e.g., $c_{0}=3$ for RGB images), and so $c_{0}<c_{1}^{*}$ is reasonable. 
In any case, (1) many realistic datasets are low rank \citep{udell2019big, zeno2024minimum}, and (2) in order to represent potentially complex target functions with narrow NNs we require depth \citep{kidger2020universal} --- both for student and teacher. 
This might suggest another reason why deep architectures are more useful in realistic settings.

\textbf{Quantized models.} Our results in Section \ref{sec:quant_nets} rely on NNs to have a quantized weights. The sample complexity we obtain is a product of the parameter count of the model (mainly the teacher) and the quantization bits. The simplicity of the bound allows it to be applied to other architectures as well. 
For example, we can add pooling layers to the CNN models we analyzed. 
Also, using the same considerations, we may obtain similar results for multi-head attention layers when the student has redundant heads and no LayerNorm.  Moreover, using quantization-aware methods \citep{hubara2018quantized} one can reduce numerical precision to improve the bound. For example, using 2bits weights (and activations) in ResNet50 on ImageNet results in only 0.5\% degradation in accuracy \citep{Liu2022}). Such quantization approaches are common in compression-type bounds (e.g. \citet{lotfi2022pacbayes}). However, there the goal is to compress the learned model (the student), while here we only need some compressed model to \textit{exist} (i.e.~the narrow teacher).

\textbf{Beyond interpolators.} In this work we focus for simplicity on NN interpolators, but in many scenarios, the training loss does not reach zero. 
In \appref{app: non interpolators} we show how to generalize our results to this non-realizable case, where the teacher NN does not reach zero trainning error (i.e., there is some irreducible error). 
Unfortunately, the use of the non-realizable generalization bound for finite hypothesis classes introduce a quadratic dependence of the sample complexity on the generalization error bound $\varepsilon$.
It is interesting to see if this quadratic dependence can be improved.

\textbf{Does posterior sampling bias towards sparse representations?} Our results indicate that posterior sampling biases NNs towards sparse representations. To examine whether this also happens in other models, we empirically examined posterior sampling in a sparse regression setting for linear diagonal networks (a common theoretical model, e.g. \citet{pmlr-v125-woodworth20a,moroshko2020implicit}) with a Gaussian weights prior, where the ground truth has a single nonzero component.
In that case, we do get a bias toward a sparser predictor, but only with sufficient depth.
Specifically, we found that a small depth (2 or 3) does not help much to improve generalization compared to depth 1. 
However, larger depth did seem to help significantly.
In contrast, the NN's result in this paper is different, since there depth is not necessary to obtain good generalization results (as our results hold even with depth 2).

\textbf{Parameterization and minimum description length.}
The results in this paper rely heavily on the choice of parameterization of the hypotheses.
Specifically, as discussed in \secref{sec:gen_bound}, we can obtain results similar to \lemref{thm: gen of gac with teacher simple} using a Minimum Description Length (MDL)/Occam's Razor learning rule.
Choosing a different mapping from parameters (or descriptions) to hypotheses may not benefit from the redundancy and will result in different generalization bounds which may be worse and even trivial.
An example for such parameterization and further discussion are presented in \appref{app: MDL}.

\textbf{Relation to SGD.}
As we have mentioned in \remref{rem: volume quantized}, the volume of `bad' interpolators is exponentially decaying in the size of the training set.
Therefore, algorithms with bad generalization properties must have significant probability to sample this `small' subset of hypotheses.
Therefore, it would be interesting to know when practical training algorithms are biased towards this `bad' domain.
For example, the implicit bias of SGD is partially known in some cases (e.g., \citet{lyu2019gradient}), so it would be interesting to understand the generalization of typical NN interpolators that also obey this implicit bias.

\section*{Impact Statement}
This paper presents work whose goal is to advance the field of Machine Learning. There are many potential societal consequences of our work, none which we feel must be specifically highlighted here.

\section*{Acknowledgments}
The authors would like to thank Itay Evron and Yaniv Blumenfeld for valuable comments and discussions. The research of DS was Funded by the European Union (ERC, A-B-C-Deep, 101039436). Views and opinions expressed are however those of the author only and do not necessarily reflect those of the European Union or the European Research Council Executive Agency (ERCEA). Neither the European Union nor the granting authority can be held responsible for them. DS also acknowledges the support of the Schmidt Career Advancement Chair in AI. Part of this work was done as part of the NSF-Simons funded Collaboration on the Theoretical Foundations of Deep Learning.  NS was supported in part by NSF IIS awards. 

\bibliography{999_biblio}
\bibliographystyle{icml2024}

\newpage
\appendix
\onecolumn
\newpage
\section{Redundancy and Description Length}
\label{app: MDL}
\newcommand{\tirm}[1]{\mathrm{trim}(\sigma)}

In this section, we want to give a broader perspective on the effect of the parameterization on  $\pteacher = \bbP_{h \sim \mathcal{P}} \rb{h \equiv h^{\star}}$. 
As we saw, the parameterization controls the generalization behavior of \GaCtext{} and posterior sampling, and how a seemingly uniform prior over parameters can induce a rich prior over hypothesis via redundancy. We first consider a conceptually very direct parameterization.  

\textbf{How parameter redundancy works: a minimal example.} Consider a predictor $h_\sigma$ specified by a description language (i.e.~parameterization) using bit strings $\sigma$, where descriptions end with the string ``END'' (encoded in bits), and this string never appears elsewhere in the description\footnote{This can be a Turing complete programming language, or perhaps better to think of simpler descriptions such as boolean formulas or strings encoding decision trees.}. 
We will consider students that use fixed-length descriptions, i.e.,~$\sigma \in \{0,1\}^\C$, where we only consider the description up to the first ``END'', i.e.~$h_\sigma=h_{\tilde{\sigma}}$ where $\tilde{\sigma}$ is a prefix of $\sigma$ ending with the first ``END''.\footnote{If $\sigma$ does not contain ``END'', or $\tilde{\sigma}$ is not a valid description, we can set $h_\sigma$ to the constant 0 predictor.}  
Let us now consider a uniform prior over this parameterization, i.e.~$\mathcal{P}(\sigma)=2^{-\C}$ for all $\sigma \in \{0,1\}^\C$. 
Although this prior is uniform over {\em parameterizations}, it is easy to see that it induces a highly {\em non-uniform} prior over predictors: for a predictor $h_{\tau}$ with a short description $\tau\in\{0,1\}^{\abs{\tau}}$, $\abs{\tau} \ll \C$, it's $\C$-bit description is highly redundant, as the final $\C-\abs{\tau}$ bits can be set arbitrarily, as long as the first bits are $\tau$, and so $\mathcal{P}(h_\tau) \geq 2^{-\abs{\tau}}$. 
For a teacher with a short description $h^\star=h_{\tau^\star}$, $\tau^\star \in \{0,1\}^{\C^\star}$, we thus have $-\log \pteacher \leq \C^\star$. 
The sample complexity of learning a small teacher thus depends {\em only} on the complexity of the teacher and not the complexity of the (possibly much larger) student.

\textbf{Why would MDL fail here?} Returning to the MDL / Occam principle mentioned at the end of \secref{sec:gen_bound}, it is important to note that applying this principle to the distribution over {\em parameters} would not work here.  Consider the MDL / Occam / MAP rule which selects the most likely interpolating parameters $\arg\max_{\mathcal{L}_\calS(h_\sigma)}\mathcal{P}(\sigma)$.  This rule would be useless here, since for all $\sigma$ in our parameter space $\{0,1\}^\C$ have the same prior probability, and its associated sample complexity would be $O(-\log \mathcal{P}(\sigma^\star))=\C$ (not $\mathcal{P}(h_{\sigma^\star})$ !), where $\sigma^\star$ is a description of $h^\star$ in our parameter space (e.g.~$\sigma^\star = \tau^\star+$`0'$\cdot(\C-\C^\star)$).  The important distinction is that Occam considers the probability mass function over {\em parameters} (roughly, the density under a specific parameterization or base measure).  On the other hand,  even if we sample $h_\sigma$ by sampling parameters $\sigma$, posterior sampling (i.e. \GaCtext{}) can be thought of directly in terms of the {\em distribution over hypotheses}.

\textbf{When would redundancy fail?} To see how changing the parameterization can change the induced distribution over hypothesis, consider instead a non-redundant parameterization, where we take $\sigma$ to be uniform over valid descriptions, of length at most $\C$, ending with ``END'' (i.e. strings ending with ``END'' and that do not otherwise contain ``END'').  The induced prior over hypothesis is now uniform\footnote{We might still have further redundancies in that multiple valid descriptions can describe the same function, again introducing non-uniformity.  Consider here a non-redundant description language, e.g.~a read-once branching program.} and even using posterior sampling/\GaCtext{} would have sample complexity determined by the complexity of the student, rather than the teacher.

\textbf{When would MDL succeed?} Finally, we note that we could of course learn a short teacher with an Occam rule that is {\em explicitly} biased towards short descriptions, e.g.~using the Kraft prior\footnote{We again absorb any remaining probability in the constant zero predictor.  Here we are thinking of an unbounded student and the parameter space being any string ending with the first occurrence of ``END''.  We can also of course bound the size of the student.} $\mathcal{P}(\sigma)=2^{-\abs{\sigma}}$.  Using such a prior, the Occam rule also enjoys a sample complexity of $O(\C^\star)$ that depends only on the length of the teacher.  But here the prior over parameters is non-uniform and explicitly biased, and our interest in this paper is in how seemingly uniform priors over parameters can induce non-uniform priors and generalization over predictors due to the choice of parameterization.  
\newpage
\section{Generalization Results}
\label{app:gen_bound}

This section contains:
\begin{itemize}
    \item A restatement and discussion on the connection to a well-known result from \citep[Corollary 2.3]{shalev2014understanding} in \appref{app: finite hypo PAC}.
    \item The proof of our primary generalization result (\lemref{thm: gen of gac with teacher simple}) in Appendix \ref{app: direct gen bound}.
    \item The proof of \corref{cor:volume} in \appref{app:proof_volume}.
    \item An extension of \lemref{thm: gen of gac with teacher simple} to non-interpolating solutions in \appref{app: non interpolators}.
    \item A further discussion on the relation of \lemref{thm: gen of gac with teacher simple} to PAC-Bayes is brought in \appref{app: pac bayes rel app}
    \item In Appendix \ref{app: refined gen bound} we prove a refined version of \lemref{thm: gen of gac with teacher simple} and discuss its implications.
    \item The proof of an alternative generalization result, used for proving generalization in the continuous setting, in Appendix \ref{app: nonuniform gen bound}.
\end{itemize}

\subsection{Finite Hypothesis Class PAC Generalization Bound}\label{app: finite hypo PAC}
Here we present a classic generalization bound for finite hypothesis classes adapted from \cite{shalev2014understanding}, which we will use throughout our paper.
\begin{theorem}\label{Theorem: classic finite hypothesis class generalization result} $[$adapted from \citep[Corollary 2.3]{shalev2014understanding}$]$
Let $\mathcal{H}$ be a finite hypothesis class. 
Let $\delta\in\rb{0,1}$ and $\varepsilon>0$ and let $N$ be an integer that satisfies 
\begin{align*}
N\ge\frac{\log\rb{\abs{\mathcal{H}}/\delta}}{\varepsilon}\,.
\end{align*}
Then, for any realizable data distribution $\mathcal{D}$, with probability of at least $1-\delta$ over the choice of an i.i.d sample $S$ of size $N$, we have that for every interpolating\footnote{\citep[Corollary 2.3]{shalev2014understanding} is more general, but we focus on interpolating hypotheses.} hypothesis, $\student$, it holds that
\begin{align*}
\mathcal{L}_\mathcal{D}\rb{h_\calS}\le\varepsilon\,.
\end{align*}
\end{theorem}

\textbf{Applying \thmref{Theorem: classic finite hypothesis class generalization result} to quantized neural networks.} It is possible to directly obtain a generalization bound for any quantized model, such as those presented in \secref{sec:quant_nets}. 
Note that the hypothesis class of quantized models with a finite number of parameters, for example, neural networks with finite width and depth, with $M$ parameters and $Q$ quantization levels is finite and of size $\left|\mathcal{H}\right|=Q^{M}$. Therefore, the sample complexity from \thmref{Theorem: classic finite hypothesis class generalization result} becomes 
\begin{align}\label{eq: naive quantized sample complexity}
N\ge\frac{M\log\rb{Q}+\log\rb{\frac{1}{\delta}}}{\varepsilon}
\end{align}

Note that \thmref{Theorem: classic finite hypothesis class generalization result} is used to prove \lemref{thm: gen of gac with teacher simple}.The sample complexity in \eqref{eq: naive quantized sample complexity} is dependent on the student parameters, and therefore much worse than the sample complexities we derive throughout our paper.
\newpage
\subsection{Proving the Generalization of Algorithm \ref{alg:Guess and Check}} \label{app: direct gen bound}
In this section, we rely on the existence of the teacher to provide guarantee on the generalization of Algorithm \ref{alg:Guess and Check}. We will first restate \lemref{thm: gen of gac with teacher simple}.

\begin{lemma}[\lemref{thm: gen of gac with teacher simple} restated]\label{thm: gen of gac with teacher app}
Let $\varepsilon\in\left(0,1\right)$ and $\delta\in\rb{0,\frac{1}{5}}$, and assume that $\pteacher<\frac{1}{2}$.
For any $N$ larger than 
\begin{align*}
 \left(-\log\left({\pteacher}\right) + 3\log\left(\frac{2}{\delta}\right)  \right) \frac{1}{\varepsilon} \,,    
\end{align*}
the sample complexity, we have that
    \begin{align*}
    \mathbb{P}_{\calS \sim\mathcal{D}^{N}, h\sim\mathcal{P}_\mathcal{S}}\left(\exError{h}\le\varepsilon\right) \ge 1-\delta\,.
    \end{align*}
\end{lemma}

\textbf{Proof Outline.} The idea is to show that the probability to sample a TE model within a finite number of steps $\tau$ is large. Conditioning on that event, we treat those models as a realizable finite hypothesis class, obtaining the sample complexity using \thmref{Theorem: classic finite hypothesis class generalization result}, which is then bounded, and after some technical details the theorem is obtained.

We first recall \defref{def: TE} and add some notation:
\begin{definition}$[$\defref{def: TE} extended$]$
For any hypothesis class $\mathcal{H}$, we say that $h \in\mathcal{H}$ is a \textit{teacher-equivalent} (TE)
model w.r.t $\mathcal{D}$, and denote $h \equiv h^{\star}$, if 
\[
\mathbb{P}_{\mathbf{x}\sim \mathcal{D}} \rb{h\left(\mathbf{x}\right)=h^{\star}\left(\mathbf{x}\right)}=1\,.
\]
We denote the probability of a random hypothesis to be TE by 
\begin{align*}
    \pteacher=\bbP_{h \sim \mathcal{P}} \rb{h \equiv h^{\star}}\,
\end{align*}
and denote
\begin{align*}
    h^\star \in \mathcal{H} \; \Longleftrightarrow \; \exists h \in \mathcal{H} \; : \; h \equiv h^\star \;.
\end{align*}
\end{definition}

\begin{definition}
    Recall the sequence $\rb{h_t}_{t=1}^\infty$ from \algref{alg:Guess and Check} is sampled such that $h_t \overset{\textrm{i.i.d.}}{\sim} \calP$. 
    For any $\tau\in\bbN$ we define
    \begin{align*}
        \mathcal{H}_{\tau}=\{h_1,\dots,h_{\tau}\}\sim\mathcal{P}^{\tau}        
    \end{align*}
    as the finite hypothesis class created from the first $\tau$ hypotheses.
\end{definition}

We first show that the probability of sampling a TE model early enough is large.

\begin{lemma}\label{lem: eta bound}
    For $\tau=\left\lceil \frac{\log\left(\delta_{h}\right)}{\log\left(1-\pteacher\right)}\right\rceil$ it holds that 
       \begin{align*}
        \bbP_{\mathcal{H}_{\tau}}\rb{h^{\star} \in \mathcal{H}_{\tau}} \ge 1-\delta_{h}\,.
    \end{align*}
    That is, the probability to have a TE model within the first $\tau$ sampled model is at least $1-\delta_{h}$.
\end{lemma}

\begin{proof}
Since $h_t$ is sampled $i.i.d$ from $\mathcal{P}$, we have that
\begin{align*}
\bbP_{\mathcal{H}_\tau}\rb{h^{\star} \in \mathcal{H}_{\tau}} & =1-\mathbb{P}\left(\forall t=1,\dots,\tau\ h_{t}\not\equiv h^\star \right)\\
 & =1-\prod_{t=1}^{\tau}\mathbb{P}\left(h_{t} \not\equiv h^\star \right)\\
 & =1-\left(1-\pteacher\right)^{\tau}.
\end{align*}
Choosing $\tau=\left\lceil \frac{\log\left(\delta_{h}\right)}{\log\left(1-\pteacher\right)}\right\rceil $
we get 
\begin{align*}
\left(1-\pteacher\right)^{\tau} & =\left(1-\pteacher\right)^{\left\lceil \frac{\log\left(\delta_{h}\right)}{\log\left(1-\pteacher\right)}\right\rceil }\\
 & \le\left(1-\pteacher\right)^{\frac{\log\left(\delta_{h}\right)}{\log\left(1-\pteacher\right)}}\\
 & =\exp\left(\log\left(1-\pteacher\right)\frac{\log\left(\delta_{h}\right)}{\log\left(1-\pteacher\right)}\right)\\
 & =\delta_{h}
\end{align*}
which means 
\[
\bbP_{\mathcal{H}_\tau}\rb{h^{\star} \in \mathcal{H}_{\tau}}\ge1-\delta_{h}.
\]
\end{proof}

We now use \thmref{Theorem: classic finite hypothesis class generalization result} to obtain the sample complexity.

\begin{lemma}\label{lem: app use finite hypothesis for gen bound}
    Let $\varepsilon\in\rb{0,1}$, and let $\tau\in\mathbb{N}$ and $\mathcal{H}_\tau$ such that $h^{\star} \in \mathcal{H}_{\tau}$.
    Then for any interpolating model $h_{\mathcal{S}} \in \mathcal{H}_\tau$, and $N\ge\frac{\log\rb{{\tau}/{\delta_{\mathcal{S}}}}}{\varepsilon}$
    \begin{align*}    \bbP_{\mathcal{S}\sim\mathcal{D}^N}\rb{\mathcal{L}_{\mathcal{D}}\rb{h_{\mathcal{S}}}\le\varepsilon}\ge1-\delta_{\mathcal{S}}\,,
    \end{align*}    
\end{lemma}

\begin{proof}
    If $h^{\star} \in \mathcal{H}_{\tau}$ then $\mathcal{H}_{\tau}$ is realizable, so from Theorem \ref{Theorem: classic finite hypothesis class generalization result}
    with
    \begin{align*}
        N\ge\frac{\log\rb{\frac{\abs{\mathcal{H}_{\tau}}}{\delta_{\mathcal{S}}}}}{\varepsilon}
        =\frac{\log\rb{\frac{\tau}{\delta_{\mathcal{S}}}}}{\varepsilon} \,,
    \end{align*}
    we get the lemma.
\end{proof}

We now wish to bound the sample complexity.

\begin{lemma}\label{lem: epsilon bound}
    For $\tau=\left\lceil \frac{\log\left(\delta_{h}\right)}{\log\left(1-\pteacher\right)}\right\rceil$, under the assumption that $\pteacher<\frac{1}{2}$, it holds that
    \begin{align*}
        \log\rb{\frac{\tau}{\delta_{\mathcal{S}}}}
        \le
        \log\left(\frac{1}{\pteacher}\right)+\log\left(\frac{1}{\delta_{\mathcal{S}}}\right)+\log\left(\log\left(\frac{1}{\delta_{h}}\right)\right)+\frac{\pteacher}{\log\left(\frac{1}{\delta_{h}}\right)}
    \end{align*}
\end{lemma}

\begin{proof}
First, we bound
\begin{align*}
\log\left(\frac{\tau}{\delta_{\mathcal{S}}}\right) & =\log\left(\frac{\lceil\frac{\log\left(\delta_{h}\right)}{\log\left(1-\pteacher\right)}\rceil}{\delta_{\mathcal{S}}}\right)\\
 & \le\log\left(\frac{\frac{\log\left(\delta_{h}\right)}{\log\left(1-\pteacher\right)}+1}{\delta_{\mathcal{S}}}\right) \, ,
\end{align*}
which, after some simplification, becomes
\begin{align}\label{eq: epsilon first bound}
    \log\left(\frac{\tau}{\delta_{\mathcal{S}}}\right)\
    \le 
    \log\left(\frac{\log\left(\delta_{h}\right)+\log\left(1-\pteacher\right)}{\delta_{\mathcal{S}}\log\left(1-\pteacher\right)}\right)\,.
\end{align}
We now recall the Taylor expansion of $\log\rb{c+x}$ around $x=0$ for some $c > 0$,
\begin{align}\label{eq: taylor log}
    \log\rb{c+x}=\log\rb{c}+\frac{x}{c}-\frac{x^2}{2c^2}+O\rb{x^3}\,,
\end{align}
plugging $c=1$ and $x\leftarrow-x$ into \eqref{eq: taylor log} we get the following bounds for any $x\in\bb{0,\frac{1}{2}}$: \begin{align}\label{eq: log bounds}
    -x-x^{2}\le\log\left(1-x\right)\le-x\,.
\end{align}
Combining \eqref{eq: log bounds} with \eqref{eq: epsilon first bound} we get
\begin{align*}
\log\left(\frac{\log\left(\delta_{h}\right)+\log\left(1-\pteacher\right)}{\delta_{\mathcal{S}}\log\left(1-\pteacher\right)}\right) & \le\log\left(\frac{\log\left(\delta_{h}\right)-\pteacher}{\delta_{\mathcal{S}}\left(-\pteacher-\pteacherSquared\right)}\right),
\end{align*}
which can be written as 
\begin{align}\label{eq: log addition bound}
    \log\left(\frac{\log\left(\delta_{h}\right)+\log\left(1-\pteacher\right)}{\delta_{\mathcal{S}}\log\left(1-\pteacher\right)}\right)
    \le 
    \log\left(\frac{1}{\delta_{\mathcal{S}}\left(\pteacher+\pteacherSquared\right)}\right)+\log\left(\pteacher+\log\left(\frac{1}{\delta_{h}}\right)\right)\,.
\end{align}
We now recall that $\log(c+x)$ is a concave function, and therefore its graph in any $x$ is below the graph of its tangent at $x=0$, that is, from \eqref{eq: taylor log},
\begin{align}\label{eq: tangent log}
    \log(c+x)\le\log\rb{c}+ \left( \frac{d}{dx}\rb{\log(c+x)}\mid_{x=0}\right) x =\log\rb{c}+\frac{x}{c}\,.
\end{align}
Setting $c=\log\left(\frac{1}{\delta_{h}}\right)\in\bb{0,\frac{1}{2}}$ and $x=\pteacher$ into \eqref{eq: tangent log}, we obtain
\begin{align}\label{eq: apply tangent}
    \log\left(\pteacher+\log\left(\frac{1}{\delta_{h}}\right)\right)
    \le
    \log\left(\log\left(\frac{1}{\delta_{h}}\right)\right)+\frac{\pteacher}{\log\left(\frac{1}{\delta_{h}}\right)} \,.
\end{align}
Note that 
\begin{align}\label{eq: silly}
    \log\rb{\frac{1}{\pteacher+\pteacherSquared}}\le \log\rb{\frac{1}{\pteacher}}\,,
\end{align} 
and therefore combining \eqref{eq: silly}, \eqref{eq: apply tangent} and \eqref{eq: log addition bound} together with \eqref{eq: epsilon first bound} we finish the proof:
\begin{align*}
    \log\left(\frac{\tau}{\delta_{\mathcal{S}}}\right)\
    &
    \le 
    \log\left(\frac{\log\left(\delta_{h}\right)+\log\left(1-\pteacher\right)}{\delta_{\mathcal{S}}\log\left(1-\pteacher\right)}\right)
    \\
    &
    \le
    \log\left(\frac{1}{\delta_{\mathcal{S}}\left(\pteacher+\pteacherSquared\right)}\right)+\log\left(\pteacher+\log\left(\frac{1}{\delta_{h}}\right)\right)
    \\
    &
    \le
    \log\rb{\frac{1}{\delta_{\mathcal{S}}}}
    +\log\rb{\frac{1}{\pteacher}}
    +\log\left(\log\left(\frac{1}{\delta_{h}}\right)\right)
    +\frac{\pteacher}{\log\left(\frac{1}{\delta_{h}}\right)} \,.
\end{align*}
\end{proof}

We are now ready to prove \thmref{thm: gen of gac with teacher app}.

\begin{proof}[Proof of \thmref{thm: gen of gac with teacher app}]
Note that $h\sim\mathcal{P}_\mathcal{S}$ is interpolating by definition, 
and therefore, from \lemref{lem: app use finite hypothesis for gen bound}, we have that for any $\tau\in\bbN$ it holds that 
\begin{align}\label{eq:mention lemma in main gen bound prop proof}
\forall \, \varepsilon \ge \frac{\log\rb{\frac{\tau}{\delta_{\mathcal{S}}}}}{N}\quad\bbP_{\mathcal{S}\sim\mathcal{D}^N,h\sim\mathcal{P}_\mathcal{S}}\rb{\mathcal{L}_{\mathcal{D}}\rb{h}\le\varepsilon\middle\vert h^{\star} \in \mathcal{H}_{\tau}}=\bbP_{\mathcal{S}\sim\mathcal{D}^N}\rb{\mathcal{L}_{\mathcal{D}}\rb{\bar{h}(S)}\le\varepsilon}\ge1-\delta_{\mathcal{S}}\,,
\end{align}
Where $\bar{h}\rb{S}\in\mathcal{H}_{\tau}$ is some interpolator.
For $\tau=\left\lceil \frac{\log\left(\delta_{h}\right)}{\log\left(1-\pteacher\right)}\right\rceil$ we similarly get
\begin{align*}
\mathbb{P}_{S\sim\mathcal{D}^{N},h\sim\mathcal{P}_\mathcal{S}}\left(\exError{h}\le\varepsilon\right) & =\\
 \bb{\text{Total Probability}}
 & =\mathbb{P}_{S\sim\mathcal{D}^{N},h\sim\mathcal{P}_\mathcal{S}}\left(\exError{h}\le\varepsilon\middle\vert h^{\star} \in \mathcal{H}_{\tau}\right)\mathbb{P}_{h\sim\mathcal{P}_\mathcal{S}}\left(h^{\star} \in \mathcal{H}_{\tau}\right)\\
 & \, +\mathbb{P}_{S\sim\mathcal{D}^{N},h\sim\mathcal{P}_\mathcal{S}}\left(\exError{h}\le\varepsilon\middle\vert h^{\star} \notin \mathcal{H}_{\tau}\right)\mathbb{P}_{h\sim\mathcal{P}_\mathcal{S}}\left(h^{\star} \notin \mathcal{H}_{\tau}\right)\\
 \bb{\text{Probability is non-negative}}
 & \ge\mathbb{P}_{S\sim\mathcal{D}^{N},h\sim\mathcal{P}_\mathcal{S}}\left(\exError{h}\le\varepsilon\middle\vert h^{\star} \in \mathcal{H}_{\tau}\right)\mathbb{P}_{h\sim\mathcal{P}_\mathcal{S}}\left(h^{\star} \in \mathcal{H}_{\tau}\right)\\
 \bb{\text{\lemref{lem: eta bound}}}
 & \ge\mathbb{P}_{S\sim\mathcal{D}^{N},h\sim\mathcal{P}_\mathcal{S}}\left(\exError{h}\le\varepsilon\middle\vert h^{\star} \in \mathcal{H}_{\tau}\right)\left(1-\delta_{h}\right)\\
 \bb{\text{\eqref{eq:mention lemma in main gen bound prop proof}}}
 & \ge\left(1-\delta_{\mathcal{S}}\right)\left(1-\delta_{h}\right)\,.
\end{align*}
And as before, this is true for any $N\ge\frac{\log\rb{\frac{\tau}{\delta_{\mathcal{S}}}}}{\varepsilon}$. specifically, by \lemref{lem: epsilon bound}, it is true for $N$ greater than
\begin{align*}
        \frac{\log\rb{\frac{1}{\delta_{\mathcal{S}}}}
    +\log\rb{\frac{1}{\pteacher}}
    +\log\left(\log\left(\frac{1}{\delta_{h}}\right)\right)
    +\frac{\pteacher}{\log\left(\frac{1}{\delta_{h}}\right)}}{\varepsilon}\,,
\end{align*}
the sample complexity. Now we can choose $\delta_{\mathcal{S}},\delta_{h}=\frac{\delta}{2}$, and then for any $N$ greater than
\begin{align*}
        \frac{\log\rb{\frac{2}{\delta}}
    +\log\rb{\frac{1}{\pteacher}}
    +\log\left(\log\left(\frac{2}{\delta}\right)\right)
    +\frac{\pteacher}{\log\left(\frac{2}{\delta}\right)}}{\varepsilon}\,,
\end{align*} 
we have that
\begin{align*}
    \mathbb{P}_{S\sim\mathcal{D}^{N},h\sim\mathcal{P}_\mathcal{S}}\left(\exError{h}\le\varepsilon\right)
    &
    \ge
    \rb{1-\frac{\delta}{2}}^2
    \\
    &
    =1-\delta+\frac{\delta^2}{4}
    \\
    &
    \ge
    1-\delta\,.
\end{align*}
Note that for any $\delta\in\rb{0,\frac{1}{5}}$ we have that
    \begin{align*}
        \log\left(\log\left(\frac{2}{\delta}\right)\right)+\frac{\pteacher}{\log\left(\frac{2}{\delta}\right)}
        \le
        2\log\left(\frac{2}{\delta}\right)\,,
    \end{align*}
    bounding the sample complexity for simplification with
\begin{align*}
 \left(-\log\left({\pteacher}\right) + 3\log\left(\frac{2}{\delta}\right)  \right) \frac{1}{\varepsilon} \,.    
\end{align*}
\end{proof}

\newpage
\subsection{Proof for the Volume of Generalizing Interpolators (\corref{cor:volume})}\label{app:proof_volume}
We first restate \corref{cor:volume}, and then we will give its formal proof:
\begin{corollary}[volume of generalizing interpolators restated]\label{cor:volume_app}
For $\varepsilon,\delta$ as above, and any $N$ larger than
\begin{align*}
         \frac{-\log\left({\pteacher}\right) + 6\log\left(\frac{2}{\delta}\right)}{\varepsilon} \,,    
\end{align*}
the sample complexity, we have that
\begin{align*}
\mathbb{P}_{\mathcal{S}\sim\mathcal{D}^N}\left(\mathbb{P}_{h\sim\mathcal{P_{S}}}\left(\mathcal{L_{D}}\left(h\right)\ge\varepsilon\right)\ge\delta\right)\le\delta\,.
\end{align*}
\end{corollary}

\begin{proof}
    Using \lemref{thm: gen of gac with teacher simple}, 
    for any 
    \begin{align*}
        N\ge \frac{-\log\left({\pteacher}\right) + 6\log\left(\frac{2}{\delta}\right)}{\varepsilon}
        \ge
        \frac{-\log\left({\pteacher}\right) + 3\log\left(\frac{2}{\delta^2}\right)}{\varepsilon}
    \end{align*}
    we have that
    \begin{align*}
        \delta^2       &\ge\mathbb{P}_{\mathcal{S}\sim\mathcal{D}^{N},h\sim\mathcal{P_{S}}}\left(\mathcal{L_{D}}\left(h\right)\ge\varepsilon\right)
	\\
& =\mathbb{E}_{\mathcal{S}\sim\mathcal{D}^{N},h\sim\mathcal{P_{S}}}\left[\mathbb{I}\left\{ \mathcal{L_{D}}\left(h\right)\ge\varepsilon\right\} \right]
\\
&
=\mathbb{E}_{\mathcal{S}\sim\mathcal{D}^{N}}\left[\mathbb{E}_{h\sim\mathcal{P_{S}}}\left[\mathbb{\mathbb{I}}\left\{ \mathcal{L_{D}}\left(h\right)\ge\varepsilon\right\}\right]\right]
\\
&
=\mathbb{E}_{\mathcal{S}\sim\mathcal{D}^{N}}\left[\mathbb{P}_{h\sim\mathcal{P_{S}}}\left(\mathcal{L_{D}}\left(h\right)\ge\varepsilon\right)\right] \,,
    \end{align*}
    so using Markov's inequality we have
    \begin{align*}
\mathbb{P}_{\mathcal{S}}\left(\mathbb{P}_{h\sim\mathcal{P_{S}}}\left(\mathcal{L_{D}}\left(h\right)\ge\varepsilon\right)\ge\delta\right)\le\frac{\mathbb{E}_{\mathcal{S}\sim\mathcal{D}^{N}}\left[\mathbb{P}_{h\sim\mathcal{P_{S}}}\left(\mathcal{L_{D}}\left(h\right)\ge\varepsilon\right)\right]}{\delta}\le\delta\,.
    \end{align*}
    That is,
        \begin{align*}
\mathbb{P}_{\mathcal{S}}\left(\mathbb{P}_{h\sim\mathcal{P_{S}}}\left(\mathcal{L_{D}}\left(h\right) < \varepsilon\right)\ge1-\delta\right)\ge1-\delta\,.
    \end{align*}
\end{proof}

\newpage
\subsection{Extension to Non-Interpolators} \label{app: non interpolators}

In this section we extend our results to the case where both the teacher and student are not assumed to be interpolating. 

Specifically, we are interested in the following setting.
Assume that there exists a narrow teacher NN $h^{\star}$ s.t ${\mathcal{L_{D}}\left(h^{\star}\right)=\varepsilon^{\star}>0}$.
Denote by $\tilde{p}$ the teacher equivalence probability, and, given $\gamma > 0$, define the
posterior distribution 
\[
\mathcal{P_{S}} (h) = \calP_{\calS}^\gamma (h) = \mathbb{P}\left(h \, \middle\vert \, \mathcal{L_{S}} (h) \le 
\gamma \right)\,.
\]
In the \GaCtext{} formulation, sampling from $\calP_{\calS}^{\gamma}$ is equivalent to 
stopping at the first model satisfying
$\mathcal{L_{S}}\left(h\right)\le \gamma$.

\begin{theorem} \label{app-thm: generalization of non interpolating}
Assume that there exists a teacher $h^\star$ with $\exError{h^\star} = \varepsilon^\star > 0$ and TE probability $\pteacher \in \rb{0, \frac{1}{2}}$.
Let $\varepsilon\in\left(0,\frac{1}{2}-\varepsilon^{\star}\right)$,
$\delta\in\left(0,\frac{1}{5}\right)$, and $\gamma = \varepsilon^\star + \varepsilon$. 
If 
\begin{align*}
    N \ge \frac{- \log \rb{\pteacher} + 3 \log \rb{\frac{2}{\delta}}}{2 \varepsilon^2}
\end{align*}
then
\[
\mathbb{P}_{\mathcal{S}\sim\mathcal{D}^{N},h\sim\mathcal{P}_{\calS}^{\gamma}}\left(\mathcal{L_{D}}\left(h\right)\le\gamma + \varepsilon\right)\ge1-\delta\,.
\]
\end{theorem}

We proceed to state Hoeffding's inequality, which is used in the proof of Theorem~\ref{app-thm: generalization of non interpolating}, then continue to prove some lemmas leading to the theorem.

\begin{theorem}[Hoeffding's Inequality] \label{thm: Hoeffding}
Let $X_{1},\dots,X_{N}$ be i.i.d random
variables with $\mathbb{E}X_{i}=\mu$ and $0\le X_{i}\le1$ a.s.
Then for all $t>0$
\[
\mathbb{P}\left(\frac{1}{N}\sum_{i=1}^{N}X_{i}-\mu\ge t\right)\le\exp\left(-2Nt^{2}\right)
\]
and 
\[
\mathbb{P}\left(\left|\frac{1}{N}\sum_{i=1}^{N}X_{i}-\mu\right|\ge t\right)\le2\exp\left(-2Nt^{2}\right)\,.
\]
\end{theorem}

By definition, $\mathcal{L_{S}}\left(h\right)=\frac{1}{N}\sum_{\left(\mathbf{x},y\right)\in\mathcal{S}}\mathbb{I}\left\{ h\left(\mathbf{x}\right)\neq y\right\} $
 so using Hoeffding's inequality 
\[
\mathbb{P}_{\mathcal{S}\sim\mathcal{D}^{N}}\left(\mathcal{L_{S}}\left(h^{\star}\right)\ge\varepsilon^{\star}+\varepsilon\right)\le e^{-2\varepsilon^{2}N}\,.
\]

\begin{lemma} \label{lem: low empirical error gen bound}
Let $\varepsilon\in\left(0,\frac{1}{2}-\varepsilon^{\star}\right)$, and
let $\tau\in\mathbb{N}$, and $\mathcal{H}_{\tau}$ s.t $h^{\star}\in\mathcal{H}_{\tau}$.
Then for any model $\tilde{h}\in\mathcal{H}_{\tau}$ satisfying $\mathcal{L_{S}}\left(\tilde{h}\right)\le\varepsilon^{\star}+\varepsilon$
\[
\mathbb{P}_{\mathcal{S}\sim\mathcal{D}^{N}}\left(\mathcal{L_{D}}\left(\tilde{h}\right)\le\varepsilon^{\star}+2\varepsilon,\mathcal{L_{S}}\left(h^{\star}\right)\le\varepsilon^{\star}+\varepsilon\right)\ge1-\left(2\tau+1\right)e^{-2\varepsilon^{2}N}\,.
\]
\end{lemma}

The proof comprises standard arguments from the uniform convergence theory of generalization.

\begin{proof}
Part 1: We show that
\[
\mathcal{L_{D}}\left(\tilde{h}\right)-\varepsilon^{\star}\le\max_{h\in\mathcal{H}_{\tau}}\left|\mathcal{L_{D}}\left(h\right)-\mathcal{L_{S}}\left(h\right)\right|+\varepsilon\,.
\]
Using Hoeffding's inequality, it holds w.p. $1-e^{-2\varepsilon^{2}N}$ that 
$\mathcal{L_{S}}\left(h^{\star}\right)\le\varepsilon^{\star}+\varepsilon$, and since $h^{\star}\in\mathcal{H}_{\tau}$, 
there exists $\tilde{h}\in\mathcal{H}_{\tau}$ satisfying the condition. 
Then 
\begin{align*}
\mathcal{L_{D}}\left(\tilde{h}\right)-\varepsilon^{\star} & =\mathcal{L_{D}}\left(\tilde{h}\right)-\mathcal{L_{S}}\left(\tilde{h}\right)+\mathcal{L_{S}}\left(\tilde{h}\right)-\varepsilon^{\star}\\
 & \le\left(\mathcal{L_{D}}\left(\tilde{h}\right)-\mathcal{L_{S}}\left(\tilde{h}\right)\right)+\varepsilon^{\star}+\varepsilon-\varepsilon^{\star}\\
 & \le\max_{h\in\mathcal{H}_{\tau}}\left|\mathcal{L_{D}}\left(h\right)-\mathcal{L_{S}}\left(h\right)\right|+\varepsilon\,.
\end{align*}
Part 2: Using the union bound and then Hoeffding's inequality 
\begin{align*}
\mathbb{P}_{\mathcal{S}\sim\mathcal{D}^{N}}\left(\mathcal{L_{D}}\left(\tilde{h}\right)-\varepsilon^{\star}>2\varepsilon\right) & \le\mathbb{P}_{\mathcal{S}\sim\mathcal{D}^{N}}\left(\max_{h\in\mathcal{H}_{\tau}}\left|\mathcal{L_{D}}\left(h\right)-\mathcal{L_{S}}\left(h\right)\right|+\varepsilon>2\varepsilon\right)\\
 & =\mathbb{P}_{\mathcal{S}\sim\mathcal{D}^{N}}\left(\max_{h\in\mathcal{H}_{\tau}}\left|\mathcal{L_{D}}\left(h\right)-\mathcal{L_{S}}\left(h\right)\right|>\varepsilon\right)\\
 & =\mathbb{P}_{\mathcal{S}\sim\mathcal{D}^{N}}\left(\exists h\in\mathcal{H}_{\tau} \,:\, \left|\mathcal{L_{D}}\left(h\right)-\mathcal{L_{S}}\left(h\right)\right|>\varepsilon\right)\\
 & \le\sum_{h\in\mathcal{H}_{\tau}}\mathbb{P}_{\mathcal{S}\sim\mathcal{D}^{N}}\left(\left|\mathcal{L_{D}}\left(h\right)-\mathcal{L_{S}}\left(h\right)\right|>\varepsilon\right)\\
 & \le\sum_{h\in\mathcal{H}_{\tau}}2\exp\left(-2N\varepsilon^{2}\right)\\
 & =2\tau e^{-2N\varepsilon^{2}}\,.
\end{align*}
Part 3: Combining the probability lower bounds using the union bound
\begin{align*}
\mathbb{P}_{\mathcal{S}\sim\mathcal{D}^{N}}\left(\mathcal{L_{D}}\left(\tilde{h}\right)\le\varepsilon^{\star}+2\varepsilon,\mathcal{L_{S}}\left(h^{\star}\right)\le\varepsilon^{\star}+\varepsilon\right) & \ge1-2\tau e^{-2\varepsilon^{2}N}-e^{-2\varepsilon^{2}N}\\
 & =1-\left(2\tau+1\right)e^{-2\varepsilon^{2}N}\,.
\end{align*}

\end{proof}

With \lemref{lem: eta bound} and \lemref{lem: low empirical error gen bound} we deduce the following.

\begin{lemma} \label{app-lem: generalization of non interpolating in terms of tau}
Let $\varepsilon\in\left(0,\frac{1}{2}-\varepsilon^{\star}\right)$ and
$\delta\in\left(0,1\right)$. Taking $\tau=\left\lceil \frac{\log\left(\frac{\delta}{2}\right)}{\log\left(1-\tilde{p}\right)}\right\rceil $
and $N\ge\frac{1}{2\varepsilon^{2}}\log\left(2\cdot\frac{2\tau+1}{\delta}\right)$
we get 
\[
\mathbb{P}_{\mathcal{S}\sim\mathcal{D}^{N},h\sim\mathcal{P_{S}}}\left(\mathcal{L_{D}}\left(h\right)\le\varepsilon^{\star}+2\varepsilon\right)\ge1-\delta\,.
\]
\end{lemma}

\begin{proof}
Recall that sampling $h\sim\mathcal{P_S}$ is equivalent to sampling an hypothesis with the Guess\&Check algorithm. 
Using $h^{\star}\in\mathcal{H}_{\tau}$ to denote the event that there is a teacher equivalent hypothesis sampled within the first $\tau$
samples, 
\begin{align*}
\mathbb{P}_{\mathcal{S}\sim\mathcal{D}^{N},h\sim\mathcal{P_{S}}}\left(\mathcal{L_{D}}\left(h\right)\le\varepsilon^{\star}+2\varepsilon\right) & \ge\mathbb{P}_{\mathcal{S}\sim\mathcal{D}^{N},h\sim\mathcal{P_{S}}}\left(\mathcal{L_{D}}\left(h\right)\le\varepsilon^{\star}+2\varepsilon,\mathcal{L_{S}}\left(h^{\star}\right)\le\varepsilon^{\star}+\varepsilon,h^{\star}\in\mathcal{H}_{\tau}\right)\\
 & \ge1-\frac{\delta}{2}-\left(2\tau+1\right)e^{-2\varepsilon^{2}N}\\
 & =1-\frac{\delta}{2}-\frac{\delta}{2}\\
 & =1-\delta\,.
\end{align*}
\end{proof}

\begin{proof}[Proof of Theorem~\ref{app-thm: generalization of non interpolating}]
We can simplify the lower bound on $N$ by explicitly writing $\tau$ in terms of $\tilde{p}$
in a manner similar to \lemref{lem: epsilon bound}. 
Specifically, assuming $\tilde{p}<\frac{1}{2}$
\begin{align*}
\log\left(2\tau+1\right)&\le\log\left(\frac{2\log\left(\frac{\delta}{2}\right)}{\log\left(1-\tilde{p}\right)}+3\right) \\
&\le\log\left(\frac{1}{\tilde{p}+\tilde{p}^{2}}\right)+\log\left(3\tilde{p}+2\log\left(\frac{2}{\delta}\right)\right) \\
&\le\log\left(\frac{1}{\tilde{p}}\right)+\log\left(2\log\left(\frac{2}{\delta}\right)\right)+\frac{3\tilde{p}}{2\log\left(\frac{2}{\delta}\right)}
\end{align*}
 and for $\delta<\frac{1}{5}$ 
\[
\log\left(2\tau+1\right)\le-\log\left(\tilde{p}\right)+2\log\left(\frac{2}{\delta}\right)
\]
so the sample complexity is 
\[
\frac{1}{2\varepsilon^{2}}\log\left(2\cdot\frac{2\tau+1}{\delta}\right)\le\frac{\log\left(2\tau+1\right)+\log\left(\frac{2}{\delta}\right)}{2\varepsilon^{2}}\le\frac{-\log\left(\tilde{p}\right)+3\log\left(\frac{2}{\delta}\right)}{2\varepsilon^{2}}
\]
\end{proof}

\begin{remark}
    Notice that the sample complexity is quadratically dependent on the population error $\varepsilon$, as opposed to the linear dependence in the realizable case, i.e. when the teacher is assumed to be a perfect interpolator.
    This is expected as this there is no realizability assumption.
\end{remark}

\newpage
\subsection{Relationship to PAC Bayes} \label{app: pac bayes rel app}

As stated in \secref{sec:gen_bound}, a similar result to \lemref{thm: gen of gac with teacher simple} could be derived with PAC-Bayes analysis \citep{McAllester1999}, which typically focuses on the {\em expected} population error $\mathbb{E}_{h\sim\mathcal{P}_\calS}\left[\exError{h} \right]$ of a sample from the posterior.
A standard PAC-Bayes bound \citep{langford2001bounds,SimplifiedPAC} yields the following result:

\begin{proposition} \label{prop: reg pac bayes res app}
Let $\varepsilon > 0$ and $\delta \in \rb{0, 1}$.
Then with sample complexity
\begin{align*}
O \rb{\frac{-\log \rb{\Tilde{p}} + \log \rb{\frac{1}{\delta}}}{\varepsilon}}
\end{align*}
we have that 
\begin{equation*}
\bbP_{\calS \sim \mathcal{D}^N} \left(\;\mathbb{E}_{h\sim\mathcal{P}_{\calS}} \left[\exError{h} \right] < \varepsilon \; \right) \geq 1-\delta.
\end{equation*}
\end{proposition}

We can naively use this bound together with Markov's inequality to get a single sample bound. 

\begin{corollary}\label{cor:pac_markov_app}
Let $\varepsilon > 0$ and $\delta \in \rb{0, 1}$.
Then with sample complexity
\begin{align*}
O \rb{\frac{-\log \rb{\Tilde{p}} + \log \rb{\frac{1}{\delta}}}{\varepsilon \delta}}
\end{align*}
we have that 
\begin{equation*}
\bbP_{\calS \sim \mathcal{D}^N} \left(\;\mathbb{P}_{h\sim\mathcal{P}_{\calS}} \left(\exError{h} < \varepsilon \right) \ge 1 - \delta \; \right) \geq 1-\delta.
\end{equation*}
\end{corollary}

\begin{proof}
Using \propref{prop: reg pac bayes res app},
\begin{align*}
\bbP_{\calS \sim \mathcal{D}^N} \left( \;\mathbb{E}_{h\sim\mathcal{P}_{\calS}} \bb{ \exError{h} } < \varepsilon \delta \right) \; \geq 1 - \delta\,. 
\end{align*}
Then using Markov's inequality, w.p. $1 - \delta$ over $\calS \sim \mathcal{D}^N$,
\begin{align*}
\mathbb{P}_{h\sim\mathcal{P}_{\calS}} \left(\exError{h} \ge \varepsilon \right) \le \frac{\mathbb{E}_{h\sim\mathcal{P}_{\calS}} \left[ \exError{h} \right]}{\varepsilon} < \frac{\varepsilon \delta}{\varepsilon} = \delta \,.
\end{align*}
By using the complement probability we get the result.
\end{proof}

Note that the sample complexity in \corref{cor:pac_markov_app} is larger than the one in \corref{cor:volume}.
Instead of additive $\log\rb{\frac{1}{\delta}}$ factors in \corref{cor:volume}, here we have a multiplicative $\frac{1}{\delta}$ factor.

\newpage
\subsection{Proving a Refined Version of \lemref{thm: gen of gac with teacher simple}} \label{app: refined gen bound}

\textbf{Motivation.} Note that two sources of randomness affect the sample complexity in \lemref{thm: gen of gac with teacher simple}: the random sampling of hypotheses from the prior $\calP$ in \GaCtext{} algorithm and the random sampling of the dataset from $\mathcal{D}^N$. To understand how each of these sources affects the obtained random complexity, we derive a refined generalization bound:

\begin{theorem}[\emph{\GaCtext{} Generalization, refined}]\label{thm: gen of gac with teacher refined}
Let $\varepsilon,\delta_{\mathcal{S}}\in\left(0,1\right)$, and $\delta_{h}\in\rb{0,\frac{1}{5}}$, and assume that $\pteacher<\frac{1}{2}$. 
For any $N$ larger than 
\begin{align*}
   \rb{-\log\rb{\pteacher}+\log\left(\frac{1}{\delta_{\mathcal{S}}}\right)+2\log\left(\log\left(\frac{1}{\delta_{h}}\right)\right)}\frac{1}{\varepsilon}\,,
\end{align*}
the sample complexity, we have that
    \begin{align*}    \bbP_{\mathcal{A}_{\mathcal{P}}}\rb{\mathbb{P}_{S\sim\mathcal{D}^{N}}\left(\exError{\mathcal{A}_\calP(S)} < \varepsilon\right) \ge 1 - \delta_{\mathcal{S}}}\ge1-\delta_{h}\,.
    \end{align*}
\end{theorem}
\begin{remark}
    Note that the randomness in $\mathcal{A}_{\mathcal{P}}$ is only from the sampling of the sequence $\rb{h_t}_{t=1}^{\infty}$, and not from the dependence of $\mathcal{A}_{\mathcal{P}}\rb{\mathcal{S}}$ on $\mathcal{S}$.
\end{remark}
\paragraph{Discussion.}
\thmref{thm: gen of gac with teacher refined} guarantees generalization with probability at least $1-\delta_{h}$ over the hypothesis sampling and probability $\delta_{\mathcal{S}}$ over the data sampling. 
This separation between $\delta_{h}$ and $\delta_{\mathcal{S}}$ highlights how both sources of randomness play a role in generalization. 
Interestingly, the sample complexity term $N$ exhibits a logarithmic dependence on $\delta_{\mathcal{S}}$ and only a doubly logarithmic dependence on $\delta_{h}$.  Thus, for any $\delta_{\mathcal{S}}$ and $\varepsilon$, the probability of not sampling a PAC interpolator decays extremely fast (doubly exponential in $N$): 
\begin{align*}
    \delta_{h} = \exp \rb{- \exp \rb{\frac{1}{2}\rb{\varepsilon N + \log \rb{{\pteacher}} + \log \rb{{\delta_{\mathcal{S}}}}}}}\,.
\end{align*}
In other words, the sampled interpolator is \textit{`typically PAC'}, i.e., PAC with overwhelmingly high probability over the sampled interpolator sequence $\rb{h_t}_{t=1}^{\infty}$.

\begin{proof}
Recall that the hypothesis chosen by \GaCtext{}, $h\sim\mathcal{P}_\mathcal{S}$, is interpolating by definition. 
Set $\tau=\left\lceil \frac{\log\left(\delta_h\right)}{\log\left(1-\pteacher\right)}\right\rceil$, then
\begin{align*}
&\mathbb{P}_{\mathcal{A}_{\mathcal{P}}}\left(\mathbb{P}_{S\sim\mathcal{D}^{N}}\left(\exError{\mathcal{A}_\calP(S)}\le\varepsilon\right)\ge1-\delta_{\mathcal{S}}\right)
 \\
  \bb{\text{Total probability}}
  =&\mathbb{P}_{\mathcal{A}_{\mathcal{P}}}\left(\mathbb{P}_{S\sim\mathcal{D}^{N}}\left(\exError{\mathcal{A}_\calP(S)}\le\varepsilon\right)\ge1-\delta_{\mathcal{S}}\middle\vert h^{\star} \in \mathcal{H}_{\tau}\right)\mathbb{P}_{\mathcal{A}_{\mathcal{P}}}\left(h^{\star} \in \mathcal{H}_{\tau}\right)
 \\
  +&\mathbb{P}_{\mathcal{A}_{\mathcal{P}}}\left(\mathbb{P}_{S\sim\mathcal{D}^{N}}\left(\exError{\mathcal{A}_\calP(S)}\le\varepsilon\right)\ge1-\delta_{\mathcal{S}}\middle\vert h^{\star} \notin \mathcal{H}_{\tau}\right)\mathbb{P}_{\mathcal{A}_{\mathcal{P}}}\left(h^{\star} \notin \mathcal{H}_{\tau}\right)
 \\
 \bb{\text{Probability is non-negative}}
  \ge&\mathbb{P}_{\mathcal{A}_{\mathcal{P}}}\left(\mathbb{P}_{S\sim\mathcal{D}^{N}}\left(\exError{\mathcal{A}_\calP(S)}\le\varepsilon\right)\ge1-\delta_{\mathcal{S}}\middle\vert h^{\star} \in \mathcal{H}_{\tau}\right)\mathbb{P}_{\mathcal{A}_{\mathcal{P}}}\left(h^{\star} \in \mathcal{H}_{\tau}\right)
 \\
 \bb{\text{\lemref{lem: eta bound}}}
  \ge&\mathbb{P}_{\mathcal{A}_{\mathcal{P}}}\left(\mathbb{P}_{S\sim\mathcal{D}^{N}}\left(\exError{\mathcal{A}_\calP(S)}\le\varepsilon\right)\ge1-\delta_{\mathcal{S}}\middle\vert h^{\star} \in \mathcal{H}_{\tau}\right)\left(1-\delta_h\right)
 \\
 \bb{\text{\lemref{lem: app use finite hypothesis for gen bound}}}
  =&1-\delta_h\,.
\end{align*}
This holds for any $N\ge\frac{\log\rb{\frac{\tau}{\delta_{\mathcal{S}}}}}{\varepsilon}$. 
Specifically, by \lemref{lem: epsilon bound}, it holds for any $N$ larger than
    \begin{align*}
        \frac{\log\left(\frac{1}{\pteacher}\right)+\log\left(\frac{1}{\delta_{\mathcal{S}}}\right)+\log\left(\log\left(\frac{1}{\delta_h}\right)\right)}{\varepsilon}+\frac{\pteacher}{\varepsilon\log\left(\frac{1}{\delta_h}\right)}\,,
    \end{align*}
    Note that for any $\delta_{h}\in\rb{0,\frac{1}{5}}$ it holds that 
    \begin{align*}
        \frac{\pteacher}{\log\left(\frac{1}{\delta_{h}}\right)}
        \le
        \log\left(\log\left(\frac{1}{\delta_{h}}\right)\right)\,,
    \end{align*}
    bounding the sample complexity by
    \begin{align*}
   \rb{\log\left(\frac{1}{\delta_{\mathcal{S}}}\right)+2\log\left(\log\left(\frac{1}{\delta_{h}}\right)\right)-\log\rb{\pteacher}}\frac{1}{\varepsilon}\,.
\end{align*}
    \end{proof}
\newpage
\subsection{Proofs Using Nonuniform Learnability}
\label{app: nonuniform gen bound}
In the next pages, we will show a result that does not use the teacher assumption for the generalization of randomly sampled networks. 
The result is more general than \lemref{thm: gen of gac with teacher simple} and \ref{thm: gen of gac with teacher refined}, which were both tailored for the teacher assumption. 
However, the price to pay for relaxing this assumption is that the following result is slightly less tight.
Since we do not use the teacher assumption, with some abuse of notation we use $\mathcal{D}$ to denote the joint distribution of feature-label pairs $\rb{\mathbf{x}, y}$.

Instead of the teacher assumption, we rely on the probability of interpolation, defined as follows.
\begin{definition}
    For a training set $\calS$ and a random hypothesis $h$ from prior $\mathcal{P}$, the interpolation probability is defined as
    \begin{align*}
        \hat{p}_\calS \triangleq\bbP_{h\sim\mathcal{P}} \rb{\emError{h}=0}\,.
    \end{align*}
\end{definition}

\begin{theorem} [Generalization of Guess \& Check, restated] \label{thm: gen_bound_app}
    Under the assumption that $\hat{p}_\calS <\frac{1}{2}$ for all $\calS$, for any $\delta,\eta\in\rb{0,1}$, and $N \in \mathbb{N}$, we have with probability at least $1-\eta$ over $h \sim \calP_{\calS}$ 
    that: $$\bbP_{\calS\sim\calD^N, h\sim\calP_{\calS}}\rb{\exError{h} \leq \varepsilon}\ge 1-\delta\,,$$
    where
    \begin{align}\label{Eq: nonuniform sample complexity def app}
    \varepsilon = \varepsilon_{\delta, \eta} (\calS) =
    \frac{\log\left(\frac{1}{\hat{p}_\calS}\right)+\log\left(\frac{4}{\delta}\right)+\log\left(\log\left(\frac{2}{\eta}\right)\right)+2\log\log\left(\frac{\log\left(\frac{2}{\eta}\right)}{\hat{p}_\calS}+1\right)}{N}
    \,.
    \end{align}
\end{theorem}

\textbf{Theorem \ref{thm: gen_bound_app} proof sketch:} We first show that for any sequence of hypotheses $h_t$
\[\bbP_{\calS\sim\calD^N}\left( \emError{h_t}=0 \textrm{ and } \exError{h_t} > \tilde{\varepsilon}_t \right) \leq \delta_t\,,\]
where 
\[
\tilde{\varepsilon}_t \triangleq \frac{ \log{1/\delta_t}}{N}\,.
\]
Set $\delta_t = \frac{\delta}{4 t \log^2 (t+1)}$ to obtain \[\tilde{\varepsilon}_t =  \frac{\log t + 2\log\log\rb{t+1} + \log{4/\delta}}{N}\,,\] and use a union bound, which yields, for any $\delta>0$ and any sequence $(h_t)_{t\in\mathbb{N}}$ of hypotheses, 
\[\bbP_{\calS\sim\calD^N}\left(\exists_t \rb{\emError{h_t}=0 \textrm{ and } \exError{h_t} > \tilde{\varepsilon}_t }\right) \leq \delta\,.\]
Importantly, since this holds for any $h_t$ we can use our $h_t$ sequence from Algorithm \ref{alg:Guess and Check} and for $t=T$ to get
with probability at least $1-\delta$ over $\calS\sim\calD^N$ that 
$$\exError{\mathcal{A}_\calP(\calS)} \leq \tilde{\varepsilon}_T\,.$$

Finally, we use the fact that $T \mid \calS$ is a geometric random variable with success parameter $\hat{p}_\calS <\frac{1}{2}$ to obtain
$$\bbP_{h_t}\left(T > \frac{\log 2/\eta}{\hat{p}_\calS}\right) \leq \eta\,.$$
Taking the complementary of the probability above, combined with the fact that $\tilde{\varepsilon}_T$ is a an increasing function of $T$ concludes the theorem.

For the complete derivation, we proceed with some lemmas before proving \thmref{thm: gen_bound_app}.

\begin{lemma}\label{lem: bad interpolating lemma}
    For any $\delta>0$ and any sequence of hypotheses $(h_t)_{t\in\mathbb{N}}$: $$\bbP_{\calS \sim\calD^N}\left(\exists_t \emError{h_t}=0 \textrm{ and } \exError{h_t} > \tilde{\varepsilon}_t \right) \leq \delta$$
    where
    \begin{align}\label{eq: eps_t}
        \tilde{\varepsilon}_t = \frac{\log t + 2\log\log\rb{t+1} + \log{4/\delta}}{N}
    \end{align}
\end{lemma}
\begin{proof}
    Set $\delta_t = \frac{\delta}{4 t \log^2 (t+1)}$.
    We first show that $\sum_t \delta_t < \delta$. 
    Since $\delta_t$ is monotonically decreasing
     \[
    \sum_{t=1}^{\infty}\frac{\delta}{4 t \log^2 (t+1)}\le \frac{\delta}{4  \log^2 (2)}+\frac{\delta}{8  \log^2 (3)}+\int_3^{\infty} \frac{\delta}{4 t \log^2 (t)}\mathrm{d} t=\frac{\delta}{4}\rb{\frac{1}{\log^2\rb{2}} + \frac{0.5}{\log^2\rb{3}} + \frac{1}{\log\rb{3}}}\le \delta\,,
    \]
    where we used the change of variables $u=\log\rb{t}, \mathrm{d}u = \frac{\mathrm{d}t}{t}$, to solve the integral
    \[
        \int_3^{\infty} \frac{1}{t \log^2 (t)}\mathrm{d} t = \int_{\log\rb{3}}^{\infty} \frac{1}{ u^2}\mathrm{d} u = \frac{1}{\log\rb{3}}\,.
    \]
For each $h_t$ separately, $\exError{h_t} > \frac{ \log{1/\delta_t}}{N}$ is a deterministic event so
\begin{align*}
\bbP_{S\sim\calD^N}\left( \emError{h_t} = 0 \textrm{ and } \exError{h_t} > \frac{ \log{1/\delta_t}}{N} \right) & = \prod_{n = 1}^N \mathbb{P}_{\rb{\mathbf{x}, y} \sim \mathcal{D}} \rb{ h_t \rb{\mathbf{x}} = y} \\
& = \prod_{n = 1}^N \rb{1 - \exError{h_t}} \\
& = \rb{1 - \exError{h_t}}^{N} \\
& \le \rb{1 - \frac{ \log{1 / \delta_t} }{N}}^{N} \\
& \le \exp \rb{ - N \frac{ \log{ 1 / \delta_t} }{N}} \\
& = \delta_t\,.
\end{align*}
Taking a union bound yields the lemma.
\end{proof}

\begin{lemma}\label{lem: bound with T}
    For any $\delta>0$, and any realization of $\rb{h_t}_{t=1}^\infty$, with probability at least $1-\delta$ over $\calS\sim\calD^N$, $$\calL_{\calD} (h_T) \leq \tilde{\varepsilon}_T\,.$$
    where $\tilde{\varepsilon}_t$ is defined in \eqref{eq: eps_t} and T is defined in Algorithm \ref{alg:Guess and Check}.
\end{lemma}
\begin{proof}
    \lemref{lem: bad interpolating lemma} applies to any sequence of hypotheses, and since $\calS$ is independent of the sequence $h_t$, we can also apply it to $h_t$.  
    Since the Lemma applies to all $t$, it also applies to any random $T$, even if it depends on the sample.  
    For the $T$ used by \GaCtext{} from Algorithm \ref{alg:Guess and Check}, we always have $\emError{h_t}=0$, so we get the result from \lemref{lem: bad interpolating lemma}.
\end{proof}

We now wish to explain the dependence on $T$ in \lemref{lem: bound with T}. We will do so by the following Lemma:
\begin{lemma}\label{lem: T bound} For any $\eta$, and any $\calS$, under the assumption that $\hat{p}_\calS <\frac{1}{2}$, we have that
   $\bbP_{(h_t)}\left(T > \frac{\log 2/\eta}{\hat{p}_\calS}\right) \leq \eta$
\end{lemma}

\begin{proof}
Given $S\sim\calD^N$, we observe that $T$ is geometric with parameter $\hat{p}_\calS$. 
Using the cumulative distribution function of geometric random variables, we can bound:
\begin{align*}
    \bbP_{(h_t)}\left(T>\frac{\log\left(2/\eta\right)}{\hat{p}_\calS}\right)=\left(1-\hat{p}_\calS\right)^{\left\lfloor \frac{\log\left(2/\eta\right)}{\hat{p}_\calS}\right\rfloor }\le\left(1-\hat{p}_\calS \right)^{\frac{\log\left(2/\eta\right)}{\hat{p}_\calS}-1}
\end{align*}
And now we use a basic property of exponents $\left(1-x\right)\le e^{-x}$, which can be rewritten as $\frac{-1}{\log\left(1-x\right)}\le\frac{1}{x}$, to further bound
\begin{align*}
    \bbP_{(h_t)}\left(T>\frac{\log\left(2/\eta\right)}{\hat{p}_\calS}\right) & \le\left(1-\hat{p}_\calS\right)^{\frac{\log\left(2/\eta\right)}{\hat{p}_\calS}-1} \\
 & \le\left(1-\hat{p}_\calS \right)^{\frac{\log\left(\frac{\eta}{2}\right)}{\log\left(1-\hat{p}_\calS \right)}-1} \\
 & =\exp\left(\left(\frac{\log\left(\frac{\eta}{2}\right)}{\log\left(1-\hat{p}_\calS \right)}-1\right)\log\left(1-\hat{p}_\calS \right)\right)\\
 & \le\exp\left(\frac{\log\left(\frac{\eta}{2}\right)}{\log\left(1-\hat{p}_\calS \right)}\log\left(1-\hat{p}_\calS \right)-\log\left(1-\hat{p}_\calS \right)\right) \\
 & =\exp\left(\log\left(\frac{\eta}{2}\right)-\log\left(1-\hat{p}_\calS \right)\right)\\
 & =\frac{\frac{\eta}{2}}{1-\hat{p}_\calS} \le \eta
\end{align*}
\end{proof}

Having proved \lemref{lem: T bound} and \lemref{lem: bound with T} we are ready to prove our main result.

\begin{proof}[Proof of \thmref{thm: gen_bound_app}]
    From \lemref{lem: bound with T} we have that
    \begin{align}\label{eq: epsilon_eta_delta prob inequality}
        \bbP_{\calS\sim\calD^N, h\sim \calP_{\calS}}\rb{\calL_\calD (h) \leq \tilde{\varepsilon}_T}\ge 1-\delta\,.
    \end{align}
    We can now use \lemref{lem: T bound} to obtain with probability of at least $1-\eta$ over the sampling of $h_t$ that $T\leq \frac{\log\left(2/\eta\right)}{\hat{p}_\calS}$ whence
    \begin{align*}
    \tilde{\varepsilon}_T
    &
    = \frac{\log T + 2\log\log\rb{T+1} + \log{4/\delta}}{N}
    \\
    &
    \le \frac{\log \rb{\frac{\log\left(2/\eta\right)}{\hat{p}_\calS}} + 2\log\log\rb{\frac{\log\left(2/\eta\right)}{\hat{p}_\calS}+1} + \log{4/\delta}}{N}
    \\
    &
    \le \frac{\log\left(\frac{1}{\hat{p}_\calS}\right)+\log\left(\frac{4}{\delta}\right)+\log\left(\log\left(\frac{2}{\eta}\right)\right)+2\log\log\left(\frac{\log\left(\frac{2}{\eta}\right)}{\hat{p}_\calS}+1\right)}{N}\,,
    \end{align*}
    so with probability of at least $1-\eta$ we have that
    \begin{align*}\label{eq: epsilon_delta_eta bound}
        \bbP_{\calS\sim\calD^N, h\sim\calP_\calS} \rb{\calL_\calD (h) \leq \varepsilon_{\eta, \delta} (\calS)}
        \ge
        \bbP_{\calS\sim\calD^N}\rb{\calL_\calD (h) \leq \tilde{\varepsilon}_T}\,.
    \end{align*}
    Finally, from \eqref{eq: epsilon_eta_delta prob inequality} we get
    \begin{align*}
        \bbP_{\calS \sim\calD^N}\rb{\calL_\calD(h) \leq \varepsilon_{\eta, \delta} (\calS)}
        \ge
        1-\delta\,.
    \end{align*}
\end{proof}

\begin{lemma} \label{lem: log inequalities appendix}
Let $x > 1$ and $y > 0$. Then
\begin{align*}
    \log \rb{x + y} \le \log \rb{x} + \log \rb{1 + y}\,.
\end{align*}
\end{lemma}

\begin{proof}
Let $x, y > 0$.
\begin{align*}
\log\left(1+x+y\right)\le\log\left(1+x+y+xy\right)=\log\left(\left(1+x\right)\left(1+y\right)\right)=\log\left(1+x\right)+\log\left(1+y\right)\,.
\end{align*}
Now, suppose that $x > 1$ and $y > 0$.
\begin{align*}
\log\left(x+y\right) & =\log\left(1+\left(x-1\right)+y\right)\\
 & \le\log\left(1+\left(x-1\right)\right)+\log\left(1+y\right)\\
 & =\log\left(x\right)+\log\left(1+y\right)\,.
\end{align*}
\end{proof}

\begin{proposition}\label{theorem:pscard simplified}
For any $\delta>0$, with probability at least $1-\delta$ over $\calS\sim\calD^N$ and $h \sim \calP_{\calS}$ 
\[\exError{h} \leq \varepsilon\,,\]
where
\begin{align*}
    \varepsilon =
    \frac{\log\left(\frac{1}{\hat{p}_\calS}\right)+4\log\left(\frac{8}{\delta}\right)+2\log\left(\log\left(\frac{1}{\hat{p}_\calS}\right)\right)}{N}\,.
\end{align*}
\end{proposition}

\begin{proof}
    Here we take an alternative approach to the one presented in the proof of \thmref{thm: gen_bound_app}. Denote
    \begin{align*}
        \tilde{\varepsilon}_T =  \frac{\log T + 2\log\log\rb{T+1} + \log{4/\delta}}{N}\,.
    \end{align*}
    Now note that from the proof of \thmref{thm: gen_bound_app},
    \begin{align}\label{eq: T bound under condition}
        T\leq \frac{\log\left(2/\eta\right)}{\hat{p}_\calS} \Rightarrow \tilde{\varepsilon_T} \leq \varepsilon_{\eta,\delta} (\calS) \,.
    \end{align}
    We use the law of total probability to write
    \begin{align*}
    \bbP_{\calS\sim\calD^N}\rb{\calL_\calD (h) \leq \varepsilon_{\eta,\delta} (\calS)}
    &
    =
    \\
    &=\bbP_{\calS\sim\calD^N}\rb{\calL_\calD(h) \leq \varepsilon_{\eta,\delta} (\calS) \, \middle\vert \, T\leq \frac{\log\left(2/\eta\right)}{\hat{p}_\calS}}
    \bbP_{h_t} \rb{T \leq \frac{\log 2/\eta}{\hat{p}_\calS}}
    \\
    &
    +
    \bbP_{\calS\sim\calD^N}\rb{\calL_\calD(h) \leq \varepsilon_{\eta,\delta} (\calS) \, \middle\vert \, T> \frac{\log\left(2/\eta\right)}{\hat{p}_\calS}}
    \bbP_{h_t} \rb{T > \frac{\log 2/\eta}{\hat{p}_\calS}}
    \\
    &
    \geq
    \bbP_{\calS\sim\calD^N}\rb{\calL_\calD(h) \leq \varepsilon_{\eta,\delta} (\calS) \, \middle\vert \, T\leq \frac{\log\left(2/\eta\right)}{\hat{p}_\calS}}
    \bbP_{h_t} \rb{T \leq \frac{\log 2/\eta}{\hat{p}_\calS}}
    \\
    [\text{\eqref{eq: T bound under condition}}]&
    \ge
    \bbP_{\calS\sim\calD^N}\rb{\calL_\calD(h) \leq \tilde{\varepsilon_T}}
    \bbP_{h_t} \rb{T \leq \frac{\log 2/\eta}{\hat{p}_\calS}}
    \\
    [\text{\lemref{lem: bound with T} and \lemref{lem: T bound}}]&{\geq}
    \rb{1-\delta}\rb{1-\eta}
    \end{align*}
    Now we can choose $\delta,\eta\leftarrow\frac{\delta}{2}$, and we obtain
    \begin{align*}
        \bbP_{\calS\sim\calD^N}\rb{\calL_\calD(h) 
        \leq \varepsilon_{\frac{\delta}{2}, \frac{\delta}{2}} (\calS) }
        &
        \geq\rb{1-\frac{\delta}{2}}\rb{1-\frac{\delta}{2}}
        \\
        &
        = 1-\delta+\frac{\delta^2}{4}
        \\
        &
        \ge
        1-\delta\,.
    \end{align*}
Applying \lemref{lem: log inequalities appendix} multiple times we get 
\begin{align*}
\varepsilon_{\frac{\delta}{2}, \frac{\delta}{2}} (\calS) & =\frac{\log\left(\frac{1}{\hat{p}_\calS}\right)+\log\left(\frac{8}{\delta}\right)+\log\left(\log\left(\frac{4}{\delta}\right)\right)+2\log\log\left(\frac{\log\left(\frac{4}{\delta}\right)}{\hat{p}_\calS}+1\right)}{N}\\
 & \le\frac{\log\left(\frac{1}{\hat{p}_\calS}\right)+2\log\left(\frac{8}{\delta}\right)+2\log\left(\log\left(\frac{\log\left(\frac{4}{\delta}\right)+1}{\hat{p}_\calS}\right)\right)}{N}\\
 & =\frac{\log\left(\frac{1}{\hat{p}_\calS}\right)+2\log\left(\frac{8}{\delta}\right)+2\log\left(\log\left(\frac{1}{\hat{p}_\calS}\right)+\log\left(\log\left(\frac{4e}{\delta}\right)\right)\right)}{N}\\
 & \le\frac{\log\left(\frac{1}{\hat{p}_\calS}\right)+2\log\left(\frac{8}{\delta}\right)+2\log\left(\log\left(\frac{1}{\hat{p}_\calS}\right)\right)+\log\left(\log\left(\log\left(\frac{4e}{\delta}\right)\right)+1\right)}{N}\\
 & \le\frac{\log\left(\frac{1}{\hat{p}_\calS}\right)+4\log\left(\frac{8}{\delta}\right)+2\log\left(\log\left(\frac{1}{\hat{p}_\calS}\right)\right)}{N}
\end{align*}
\end{proof}
\newpage

\newpage
\section{Proofs for Generalization of Quantized Neural Networks
(\secref{sec:quant_nets})}
\label{app:quant_proofs}

This section contains the proofs for the results in Section \ref{sec:quant_nets}, focusing on upper bounding the effective sample complexity $C$. Specifically,
\begin{itemize}
    \item In Appendix \ref{subseq:quant fcn app} we derive the upper bound for vanilla fully connected networks as stated in \thmref{lem:int_quantized}.
    \item In Appendix \ref{subseq: NSFC app} we derive the upper bound for Neuron-Scaled fully-connected networks as stated in \thmref{lem:int_quantized}.
    \item In Appendix \ref{subseq: CNN app} we derive the upper bound for convolutional neural networks.
    \item In Appendix \ref{subsec: SCN app} we derive the upper bound for Channel-Scaled convolutional neural networks.
\end{itemize} 

\subsection{Vanilla Fully-Connected Neural Networks}\label{subseq:quant fcn app}

\textbf{Notation}
Following definition \ref{def: vanilla fcn}, for all $l=1,\dots,L$, write the weight matrices and bias
vectors in block form as 
\begin{align} \label{not: fc block notation appendix}
\mathbf{W}^{\left(l\right)}=\left[\begin{array}{cc}
\mathbf{W}_{11}^{\left(l\right)} & \mathbf{W}_{12}^{\left(l\right)}\\
\mathbf{W}_{21}^{\left(l\right)} & \mathbf{W}_{22}^{\left(l\right)}
\end{array}\right]\in\mathbb{R}^{d_{l}\times d_{l-1}},\mathbf{b}^{\left(l\right)}=\left[\begin{array}{c}
\mathbf{b}_{1}^{\left(l\right)}\\
\mathbf{b}_{2}^{\left(l\right)}
\end{array}\right]\in\mathbb{R}^{d_{l}},
\end{align}
such that
\begin{align*}
\mathbf{W}_{11}^{\left(l\right)}&\in\mathbb{R}^{d_{l}^{\star}\times d_{l-1}^{\star}}\,, \\
\mathbf{W}_{12}^{\left(l\right)}&\in\mathbb{R}^{d_{l}^{\star}\times\left(d_{l-1}-d_{l-1}^{\star}\right)}\,, \\
\mathbf{W}_{21}^{\left(l\right)}&\in\mathbb{R}^{\left(d_{l}-d_{l}^{\star}\right)\times d_{l-1}^{\star}}\,, \\
\mathbf{W}_{22}^{\left(l\right)}&\in\mathbb{R}^{\left(d_{l}-d_{l}^{\star}\right)\times\left(d_{l-1}-d_{l-1}^{\star}\right)}\,, \\
\mathbf{b}_{1}^{\left(l\right)}&\in\mathbb{R}^{d_{l}^{\star}}\,, \\
\mathbf{b}_{2}^{\left(l\right)}&\in\mathbb{R}^{d_{l}-d_{l}^{\star}}\,.
\end{align*}

\begin{remark}
    The blocks have a simple interpretation with reference to \figref{Fig: teacher-student illustration}. 
    $\mathbf{W}_{11}^{\left(l\right)}$ represents the blue edges, and $\mathbf{W}_{12}^{\left(l\right)}$ represents the orange edges.
    $\mathbf{W}_{21}^{\left(l\right)},
    \mathbf{W}_{22}^{\left(l\right)},$
    both represent gray edges. 
    As for the biases, $\mathbf{b}_{1}^{\left(l\right)}$ corresponds to the bias terms of gray vertices, and $\mathbf{b}_{2}^{\left(l\right)}$ corresponds to the bias terms of white vertices.
\end{remark}

\begin{definition}
Define the \textit{coordinate-projection} operator $\pi_{l}:\mathbf{\mathbb{R}}^{d_{l}}\rightarrow\mathbb{R}^{d_{l}^{\star}}$ for $d_l \ge d_{l}^\star$ as
\[
\pi_{l}\left(\left(x_{1},\dots,x_{d_{l}}\right)^{\top}\right)=\left(x_{1},\dots,x_{d_{l}^{\star}}\right)^{\top}\,.
\]
\end{definition}
Notice that this projection commutes with the component-wise activation function $\sigma$. 
In order to keep the proofs focused, we restate and prove each result in \thmref{lem:int_quantized} separately. We first restate only the first part of \thmref{lem:int_quantized}.

\begin{theorem}\label{lem:int_quantized fc app}

For any activation function such that $\sigma(0)=0$, depth $L$, $Q$-quantized teacher with widths $D^*$,  student with widths $D>D^*$ and prior $\mathcal{P}$ uniform over $Q$-quantized parameterizations, we have for Vanilla Fully Connected Networks that: 
\begin{align} \label{eq: PC FCN appendix}
    \SC \leq \BSCM{FC} \triangleq \rb{\sum_{l=1}^{L}{\rb{d_{l}^{\star}d_{l-1}+d_{l}^{\star}}}}\log Q\,.
\end{align}
where we defined $d_0^{\star}\triangleq d_0$.
And by \lemref{thm: gen of gac with teacher simple}, $N=(\SC+3\log 2/\delta)/\varepsilon$ samples are enough to ensure that for posterior sampling (i.e.~\GaCtext{}), $\mathcal{L}(\mathcal{A}_\mathcal{P}(\calS))\leq\varepsilon$ with probability $1-\delta$ over $\calS\sim\calD^N$ and the sampling. 
\end{theorem}

\begin{proof}
The outline of our proof is as follows: we first show a sufficient condition on the parameters of the student to ensure it will be TE, and then count the number of parameter configurations for which this condition holds. 
This yields a lower bound on the number of TE models, and since the hypothesis class of quantized networks with fixed widths and depths is finite, we are able to calculate the probability of a sampled model to be TE.

We now begin our proof. 
For all $l=1,\dots,L$, recall that 
\[
\mathbf{W}^{\star\left(l\right)}\in\mathbb{R}^{d_{l}^{\star}\times d_{l-1}^{\star}},\mathbf{b}^{\star\left(l\right)}\in\mathbb{R}^{d_{l}^{\star}}
\]
are the weights and biases of $\teacher$ respectively. 
Define
\[
\mathcal{E}=\left\{ \student \in \calH_D^{\text{FC}} \middle\vert \forall l=1,\dots,L\; \mathbf{W}_{11}^{\left(l\right)} = \mathbf{W}^{\star\left(l\right)},\; \mathbf{W}_{12}^{\left(l\right)} = \mathbf{0}_{d_{l}^{\star} \times \left(d_{l-1}-d_{l-1}^{\star}\right)},\; \mathbf{b}_{1}^{\left(l\right)} = \mathbf{b}^{\star\left(l\right)}\right\} \,.
\]
We claim that any $\student\in\mathcal{E}$ is TE. 
To prove this, we show by induction over the layer $l$, that for all $\mathbf{x}\in\mathbb{R}^{d_{0}}$,

\begin{align*}
\pi_{l}\left(f_{D}^{\left(l\right)}\left(\mathbf{x}\right)\right)=f_{D^{\star}}^{\left(l\right)}\left(\mathbf{x}\right)\,.
\end{align*}

Let $\student\in\mathcal{E}$.
Begin from the base case, $l=1$.
Since $d_{0}=d_{0}^{\star}$,
\[
\mathbf{W}^{\rb{1}} = \left[\begin{array}{c}
\mathbf{W}_{11}^{\left(1\right)}\\
\mathbf{W}_{21}^{\left(1\right)}
\end{array}\right]
\]
so using the notation from \defref{def: vanilla fcn} we find that
\[
\mathbf{W}^{\left(1\right)}f_{D}^{\left(0\right)}\left(\mathbf{x}\right)+\mathbf{b}^{\left(1\right)}=\mathbf{W}^{\left(1\right)}\mathbf{x}+\mathbf{b}^{\left(1\right)}=\left[\begin{array}{c}
\mathbf{W}_{11}^{\left(1\right)}\\
\mathbf{W}_{21}^{\left(1\right)}
\end{array}\right]\mathbf{x}+\left[\begin{array}{c}
\mathbf{b}_{1}^{\left(1\right)}\\
\mathbf{b}_{2}^{\left(1\right)}
\end{array}\right]=\left[\begin{array}{c}
\mathbf{W}_{11}^{\left(1\right)}\mathbf{x}+\mathbf{b}_{1}^{\left(1\right)}\\
\mathbf{W}_{21}^{\left(1\right)}\mathbf{x}+\mathbf{b}_{2}^{\left(1\right)}
\end{array}\right]
\]
and
\[
f_{D}^{\left(1\right)}\left(\mathbf{x}\right)=\sigma\left(\left[\begin{array}{c}
\mathbf{W}_{11}^{\left(1\right)}\mathbf{x}+\mathbf{b}_{1}^{\left(1\right)}\\
\mathbf{W}_{21}^{\left(1\right)}\mathbf{x}+\mathbf{b}_{2}^{\left(1\right)}
\end{array}\right]\right)=\left[\begin{array}{c}
\sigma\left(\mathbf{W}_{11}^{\left(1\right)}\mathbf{x}+\mathbf{b}_{1}^{\left(1\right)}\right)\\
\sigma\left(\mathbf{W}_{21}^{\left(1\right)}\mathbf{x}+\mathbf{b}_{2}^{\left(1\right)}\right)
\end{array}\right]\,.
\]
Therefore, using the definition of $\pi_1$, we get
\begin{align}\label{eq: project first layer int quant app}
\pi_{1}\left(f_{D}^{\left(1\right)}\left(\mathbf{x}\right)\right)=\sigma\left(\mathbf{W}_{11}^{\left(1\right)}\mathbf{x}+\mathbf{b}_{1}^{\left(1\right)}\right)\,.
\end{align}
Since $\student\in\mathcal{E}$, we have that  $\mathbf{W}_{11}^{\rb{1}}=\mathbf{W}^{\star\rb{1}}$ and $\mathbf{b}_1^{\rb{1}}=\mathbf{b}^{\star\rb{1}}$, so from \eqref{eq: project first layer int quant app} the coordinate projection is
\[
\pi_{1}\left(f_{D}^{\left(1\right)}\left(\mathbf{x}\right)\right)=\sigma\left(\mathbf{W}^{\star\left(1\right)}\mathbf{x}+\mathbf{b}^{\star\left(1\right)}\right)=f_{D^{\star}}^{\left(1\right)}\left(\mathbf{x}\right).
\]
Next, assume that $\pi_{l-1}\left(f_{D}^{\left(l-1\right)}\left(\mathbf{x}\right)\right)=f_{D^{\star}}^{\left(l-1\right)}\left(\mathbf{x}\right)$
for some $l\le L-1$. Since $\student\in\mathcal{E}$, for any $l\in\bb{L}$ we have that $\mathbf{W}_{12}^{\left(l\right)}=\mathbf{0}$. Therefore, following a similar argument we get
\begin{align*}
\pi_{l}\left(\mathbf{W}^{\left(l\right)}f_{D}^{\left(l-1\right)}\left(\mathbf{x}\right)+\mathbf{b}^{\left(l\right)}\right) & =\pi_{l}\left(\left[\begin{array}{cc}
\mathbf{W}_{11}^{\left(l\right)} & \mathbf{W}_{12}^{\left(l\right)}\\
\mathbf{W}_{21}^{\left(l\right)} & \mathbf{W}_{22}^{\left(l\right)}
\end{array}\right]f_{D}^{\left(l-1\right)}\left(\mathbf{x}\right)+\left[\begin{array}{c}
\mathbf{b}_{1}^{\left(l\right)}\\
\mathbf{b}_{2}^{\left(l\right)}
\end{array}\right]\right)\\
 & =\left[\begin{array}{cc}
\mathbf{W}_{11}^{\left(l\right)} & \mathbf{0}\end{array}\right]f_{D}^{\left(l-1\right)}\left(\mathbf{x}\right)+\begin{array}{c}
\mathbf{b}_{1}^{\left(l\right)}\end{array}\\
 & =\left[\begin{array}{cc}
\mathbf{W}^{\star\left(l\right)} & \mathbf{0}\end{array}\right]f_{D}^{\left(l-1\right)}\left(\mathbf{x}\right)+\begin{array}{c}
\mathbf{b}^{\star\left(l\right)}\end{array}\\
 & =\mathbf{W}^{\star\left(l\right)}\pi_{l-1}\left(f_{D}^{\left(l-1\right)}\left(\mathbf{x}\right)\right)+\begin{array}{c}
\mathbf{b}^{\star\left(l\right)}\end{array}\\
 & =\mathbf{W}^{\star\left(l\right)}f_{D^{\star}}^{\left(l-1\right)}\left(\mathbf{x}\right)+\begin{array}{c}
\mathbf{b}^{\star\left(l\right)}\end{array}\,,
\end{align*}
that is 
\begin{align}\label{eq: projection hidden layers}
\pi_{l}\left(\mathbf{W}^{\left(l\right)}f_{D}^{\left(l-1\right)}\left(\mathbf{x}\right)+\mathbf{b}^{\left(l\right)}\right) =\mathbf{W}^{\star\left(l\right)}f_{D^{\star}}^{\left(l-1\right)}\left(\mathbf{x}\right)+\begin{array}{c}
\mathbf{b}^{\star\left(l\right)}\end{array}\,.
\end{align}
Using the commutativity between $\sigma$ and $\pi_{l}$, we have
\begin{align*}
\pi_{l}\left(f_{D}^{\left(l\right)}\left(\mathbf{x}\right)\right)
&
=
\pi_{l}\rb{\sigma\rb{\mathbf{W}^{\left(l\right)}f_{D}^{\left(l-1\right)}\left(\mathbf{x}\right)+\mathbf{b}^{\left(l\right)}}}
\\
&
=
\sigma\rb{\pi_{l}\rb{\mathbf{W}^{\left(l\right)}f_{D}^{\left(l-1\right)}\left(\mathbf{x}\right)+\mathbf{b}^{\left(l\right)}}}
\\
&
\overset{\text{\eqref{eq: projection hidden layers}}}{=}
\sigma\rb{\mathbf{W}^{\star\left(l\right)}f_{D^{\star}}^{\left(l-1\right)}\left(\mathbf{x}\right)+
\mathbf{b}^{\star\left(l\right)}}
\\
&
=
f_{D^{\star}}^{\left(l\right)}\left(\mathbf{x}\right)\,.
\end{align*}
For the last layer, $l=L$, the proof is identical, except for the application of the activation function $\sigma$ at the end, so an analogue to \eqref{eq: projection hidden layers} is enough. 
Since we assume that $\student$ and $\teacher$ have the same output dimension $d_L$, this proves that for all $\mathbf{x}\in\mathbb{R}^{d_0}$
\begin{align*}
\student \rb{\mathbf{x}} = \pi_{L} \rb{\student \rb{\mathbf{x}}} = \teacher \rb{\mathbf{x}}\,.
\end{align*}
That is, $\mathcal{E}\subseteq\left\{ \student \middle\vert \student \equiv \teacher \right\} $ and therefore 
\[
\mathbb{P}\left(\mathcal{E}\right)\le\mathbb{P}\left(\student\equiv \teacher\right)\,.
\]
Finally, to calculate the probability to sample $\student$ in $\mathcal{E}$ we count the number of constrained parameters - parameters which are either determined by $\teacher$ or are $0$ in $\mathcal{E}$.
Looking at the dimensions of $\mathbf{W}_{11}^{\rb{l}}$, $\mathbf{W}_{12}^{\rb{l}}$ and $\mathbf{b}_{1}^{\rb{l}}$, we deduce that there are exactly 
\begin{align*}
    \PC=\sum_{l=1}^L d_{l}^\star \cdot d_{l-1} + d_{l}^\star = \sum_{l=1}^L d_{l}^\star \rb{d_{l-1} + 1}
\end{align*}
such constrained parameters, and denote
\begin{align*}
    \BSCM{FC} \triangleq \rb{\sum_{l=1}^{L}{\rb{d_{l}^{\star}d_{l-1}+d_{l}^{\star}}}}\log Q=\PC\log Q\,.
\end{align*}
Under the uniform prior over parameters $\mathcal{P}$,
\begin{align*}
    \pteacher \ge \mathbb{P} \rb{\mathcal{E}} = Q^{-\PC}\,,
\end{align*}
so
\begin{align*}
    \SC=-\log\rb{\pteacher}\le\PC\log{Q}=\BSCM{FC}
\end{align*}
\end{proof}

\newpage
\subsection{Neuron-Scaled Fully-Connected Neural Networks}\label{subseq: NSFC app}

We now wish to improve our result, introducing some assumptions on the architecture. 
We first need to define a new class of architectures.

\begin{definition}[Scaled-neuron FC restated]  \label{def: scaled neuron fcn app}
For a depth $L$, widths $D=\rb{d_1,\dots,d_L}$, and activation function $\sigma:\mathbb{R}\rightarrow\mathbb{R}$, a scaled neuron fully connected neural network is a mapping $\btheta \mapsto h^{FC}_{\btheta}$ from parameters
\begin{align*}
 \btheta = \left( \left\{ \mathbf{W}^{\left(l\right)} \right\}_{l=1}^{L}, \left\{ \mathbf{b}^{\left(l\right)} \right\}_{l=1}^{L}, \left\{ \bgamma^{\left(l\right)} \right\}_{l=1}^{L} \right) \,,
\end{align*}
where
\begin{align*}
    \mathbf{W}^{\left(l\right)} \in \mathbb{R}^{d_{l}\times d_{l-1}}, \mathbf{b}^{\left(l\right)} \in \mathbb{R}^{d_{l}}, \bgamma^{\left(l\right)} \in \mathbb{R}^{d_{l}}\,,
\end{align*}
defined recursively, starting with $f^{(0)} \left(\mathbf{x}\right) = \mathbf{x}\,, $ then 
\begin{align*} 
    \forall l\in\bb{L-1} \; f^{(l)} \left(\mathbf{x}\right) &= \sigma\left(\bgamma^{\rb{l}}\odot\rb{\mathbf{W}^{\left(l\right)} f^{(l-1)} \left(\mathbf{x}\right)} + \mathbf{b}^{\left(l\right)}\right)\\
    h_{\btheta}^{FC} \rb{\mathbf{x}} &= \sign\rb{\mathbf{W}^{\left(L\right)} f^{(L-1)} \left(\mathbf{x}\right) + \mathbf{b}^{\left(L\right)}} .
\end{align*}
The total parameter count is $M(D)=\sum_{l=1}^L d_l(d_{l-1}+2)$. We denote the class of all scaled neuron fully connected neural network as $\Hsfc{D}$.

\end{definition}

\begin{remark}
    We use the notation $f_{D}^{\left(l\right)}\left(\mathbf{x}\right)$ for both $\fcnhid{l}$ and $\sfchid{l}$ in a fully connected network with hidden dimensions $D$, as they can be inferred from context.
\end{remark}

\begin{theorem}\label{lem:int_quantized sfc app}

For any activation function such that $\sigma(0)=0$, depth $L$, $Q$-quantized teacher with widths $D^*$,  student with widths $D>D^*$ and prior $\mathcal{P}$ uniform over $Q$-quantized parameterizations, we have for Scaled Fully Connected Networks that: 
    \begin{align}\label{Eq: PC FCN better appendix} 
        \SC \leq \BSCM{SFC} \triangleq \rb{\sum_{l=1}^{L}{\rb{d_{l}^{\star}d_{l-1}^{\star} + 2d_{l}}}}\log Q\,,
    \end{align}
    where we defined $d_0^{\star}\triangleq d_0$.
And by \lemref{thm: gen of gac with teacher simple}, $N=(\SC+3\log 2/\delta)/\varepsilon$ samples are enough to ensure that for posterior sampling (i.e.~\GaCtext{}), $\mathcal{L}(\mathcal{A}_\mathcal{P}(\calS))\leq\varepsilon$ with probability $1-\delta$ over $\calS\sim\calD^N$ and the sampling. 
\end{theorem}

\begin{proof}
    The idea is similar to the proof of \thmref{lem:int_quantized fc app}, only that now we can define a set $\mathcal{E}$ with even more elements, which will tighten our bound. For any $l\in\bb{L}$, write $\bgamma$ as a blocks vector
    \begin{align*}
        \bgamma=\left[\begin{array}{c}
\bgamma_{1}^{\left(l\right)}\\
\bgamma_{2}^{\left(l\right)}
\end{array}\right]\in\mathbb{R}^{d_{l}}\,,
    \end{align*}
    where
    \begin{align*}
\bgamma_{1}^{\left(l\right)}&\in\mathbb{R}^{d_{l}^{\star}}\,, \\
\bgamma_{2}^{\left(l\right)}&\in\mathbb{R}^{d_{l}-d_{l}^{\star}}\,.
    \end{align*}
This time, we are interested in:
\[
\mathcal{E}=\left\{ \student \in \calH_D^{\text{SFC}} \middle\vert \forall l=1,\dots,L\; \mathbf{W}_{11}^{\left(l\right)} = \mathbf{W}^{\star\left(l\right)},\; \mathbf{b}_{1}^{\left(l\right)} = \mathbf{b}^{\star\left(l\right)},\;
\mathbf{b}_{2}^{\left(l\right)} = \mathbf{0}_{d_l-d_l^{\star}},\;
\bgamma_{1}^{\left(l\right)} = \bgamma^{\star\left(l\right)},
\bgamma_{2}^{\left(l\right)} = \mathbf{0}_{d_l-d_l^{\star}}
\right\}\,.
\]
As in the proof of \thmref{lem:int_quantized fc app}, we claim that any $\student\in\mathcal{E}$ is TE. 
This time, we specifically show by induction over the layer $l$, that for all $\mathbf{x}\in\mathbb{R}^{d_{0}}$,

\begin{align*}\label{eq:induction basis better fcn app}
    f_{D}^{\left(l\right)}\left(\mathbf{x}\right)=\left[\begin{array}{c}
f_{D^{\star}}^{\left(l\right)}\left(\mathbf{x}\right)\\
\mathbf{0}_{d_l-d_l^{\star}}
\end{array}\right]\,.
\end{align*}

Let $\student\in\mathcal{E}$.
Begin from the base case, $l=1$. Since $d_{0}=d_{0}^{\star}$,
\[
\mathbf{W}^{\rb{1}} = \left[\begin{array}{c}
\mathbf{W}_{11}^{\left(l\right)}\\
\mathbf{W}_{21}^{\left(l\right)}
\end{array}\right]
\]
so using the notation from \defref{def: vanilla fcn} we find that
\[
\bgamma^{\rb{1}} \odot\mathbf{W}^{\left(1\right)}f_{D}^{\left(0\right)}\left(\mathbf{x}\right)+\mathbf{b}^{\left(1\right)}=\bgamma^{\rb{1}} \odot\mathbf{W}^{\left(1\right)}\mathbf{x}+\mathbf{b}^{\left(1\right)}=\left[\begin{array}{c}
\bgamma^{\rb{1}}_1 \odot\mathbf{W}_{11}^{\left(1\right)}\mathbf{x}+\mathbf{b}_{1}^{\left(1\right)}\\
\bgamma^{\rb{1}}_2 \odot\mathbf{W}_{21}^{\left(1\right)}\mathbf{x}+\mathbf{b}_{2}^{\left(1\right)}
\end{array}\right]
\]
and
\[
f_{D}^{\left(1\right)}\left(\mathbf{x}\right)=\sigma\left(\left[\begin{array}{c}
\bgamma^{\rb{1}}_1 \odot\mathbf{W}_{11}^{\left(1\right)}\mathbf{x}+\mathbf{b}_{1}^{\left(1\right)}\\
\bgamma^{\rb{1}}_2 \odot\mathbf{W}_{21}^{\left(1\right)}\mathbf{x}+\mathbf{b}_{2}^{\left(1\right)}
\end{array}\right]\right)=\left[\begin{array}{c}
\sigma\rb{\bgamma^{\rb{1}}_1 \odot\mathbf{W}_{11}^{\left(1\right)}\mathbf{x}+\mathbf{b}_{1}^{\left(1\right)}}\\
\sigma\rb{\bgamma^{\rb{1}}_2 \odot\mathbf{W}_{21}^{\left(1\right)}\mathbf{x}+\mathbf{b}_{2}^{\left(1\right)}}
\end{array}\right]\,.
\]
Therefore, using the definition of $\pi_1$, we get
\begin{align}\label{eq: project first layer int quant bettter before app}
\pi_{1}\left(f_{D}^{\left(1\right)}\left(\mathbf{x}\right)\right)=\sigma\rb{\bgamma^{\rb{1}}_1\odot\mathbf{W}_{11}^{\left(1\right)}\mathbf{x}+\mathbf{b}_{1}^{\left(1\right)}}\,.
\end{align}
Since $\student\in\mathcal{E}$, we have that  $\mathbf{W}_{11}^{\rb{1}}=\mathbf{W}^{\star\rb{1}}$, $\mathbf{b}_1^{\rb{1}}=\mathbf{b}^{\star\rb{1}}$ and $\bgamma_1^{\rb{1}}=\bgamma^{\star\rb{1}}$, so from \eqref{eq: project first layer int quant bettter before app} the coordinate projection is
\begin{align}\label{eq: project first layer int quant bettter app}
\pi_{1}\left(f_{D}^{\left(1\right)}\left(\mathbf{x}\right)\right)=\sigma\rb{\bgamma^{\star\rb{1}}\odot\mathbf{W}^{\star\left(1\right)}\mathbf{x}+\mathbf{b}^{\star\left(1\right)}}=f_{D^{\star}}^{\left(1\right)}\left(\mathbf{x}\right)\,.
\end{align}
In addition, since $\bgamma_{2}^{\left(l\right)} = \mathbf{0}_{d_l-d_l^{\star}}$ and also $\mathbf{b}_{2}^{\left(1\right)} = \mathbf{0}_{d_l-d_l^{\star}}$, as well as $\sigma\rb{0}=0$, we have that
\begin{align}\label{eq:lower blocks are 0}
\sigma\rb{\bgamma^{\rb{1}}_2 \odot\mathbf{W}_{21}^{\left(1\right)}\mathbf{x}+\mathbf{b}_{2}^{\left(1\right)}}=\mathbf{0}_{d_l-d_l^{\star}}\,.
\end{align}
Putting \ref{eq:lower blocks are 0} and \ref{eq: project first layer int quant bettter app} together we have
\begin{align*}
    f_{D}^{\left(1\right)}\left(\mathbf{x}\right)=\left[\begin{array}{c}
f_{D^{\star}}^{\left(1\right)}\left(\mathbf{x}\right)\\
\mathbf{0}_{d_1-d_1^{\star}}
\end{array}\right]\,.
\end{align*}
Next, assume that $f_{D}^{\left(l-1\right)}\left(\mathbf{x}\right)=\left[\begin{array}{c}
f_{D^{\star}}^{\left(l-1\right)}\left(\mathbf{x}\right)\\
\mathbf{0}_{d_{l-1}-d^{\star}_{l-1}}
\end{array}\right]$,
for some $l\le L-1$. Following a similar argument we get 
\begin{align*}
f_{D}^{\left(l\right)}\left(\mathbf{x}\right)
&
=
\\
&
=
\sigma\rb{\bgamma^{\rb{l}}\odot\mathbf{W}^{\left(l\right)}f_{D}^{\left(l-1\right)}\left(\mathbf{x}\right)+\mathbf{b}^{\left(l\right)}}
\\
&
=
\sigma\left(\bgamma^{\rb{l}}\odot\left[\begin{array}{cc}
\mathbf{W}_{11}^{\left(l\right)} & \mathbf{W}_{12}^{\left(l\right)}\\
\mathbf{W}_{21}^{\left(l\right)} & \mathbf{W}_{22}^{\left(l\right)}
\end{array}\right]\left[\begin{array}{c}
f_{D^{\star}}^{\left(l-1\right)}\left(\mathbf{x}\right)\\
\mathbf{0}_{d_{l-1}-d^{\star}_{l-1}}
\end{array}\right]+\left[\begin{array}{c}
\mathbf{b}_{1}^{\left(l\right)}\\
\mathbf{b}_{2}^{\left(l\right)}
\end{array}\right]\right)\\
 & =\sigma\rb{\bgamma^{\rb{l}}\odot\left[\begin{array}{c}
\mathbf{W}_{11}^{\left(l\right)}\\
\mathbf{W}_{21}^{\left(l\right)}
\end{array}\right]f_{D^{\star}}^{\left(l-1\right)}\left(\mathbf{x}\right)+\left[\begin{array}{c}
\mathbf{b}_{1}^{\left(l\right)}\\
\mathbf{b}_{2}^{\left(l\right)}
\end{array}\right]}\\
 & =\sigma\rb{\left[\begin{array}{c}
\bgamma_1^{\rb{l}}\odot\mathbf{W}_{11}^{\left(l\right)}f_{D^{\star}}^{\left(l-1\right)}\left(\mathbf{x}\right)+\mathbf{b}_{1}^{\left(l\right)}\\
\bgamma_2^{\rb{l}}\odot\mathbf{W}_{21}^{\left(l\right)}f_{D^{\star}}^{\left(l-1\right)}\left(\mathbf{x}\right)+\mathbf{b}_{2}^{\left(l\right)}
\end{array}\right]}\\
 & =\left[\begin{array}{c}
\sigma\rb{\bgamma_1^{\rb{l}}\odot\mathbf{W}_{11}^{\left(l\right)}f_{D^{\star}}^{\left(l-1\right)}\left(\mathbf{x}\right)+\mathbf{b}_{1}^{\left(l\right)}}\\
\sigma\rb{\bgamma_2^{\rb{l}}\odot\mathbf{W}_{21}^{\left(l\right)}f_{D^{\star}}^{\left(l-1\right)}\left(\mathbf{x}\right)+\mathbf{b}_{2}^{\left(l\right)}}
\end{array}\right]\\
& =\left[\begin{array}{c}
\sigma\rb{\bgamma^{\star\rb{l}}\odot\mathbf{W}^{\star\left(l\right)}f_{D^{\star}}^{\left(l-1\right)}\left(\mathbf{x}\right)+\mathbf{b}^{\star\left(l\right)}}\\
\mathbf{0}_{d_{l-1}-d^{\star}_{l-1}}
\end{array}\right]\\
& =\left[\begin{array}{c}
f_{D^{\star}}^{\left(l-1\right)}\left(\mathbf{x}\right)\\
\mathbf{0}_{d_{l-1}-d^{\star}_{l-1}}
\end{array}\right]\\
&
=f_{D^{\star}}^{\left(l\right)}\left(\mathbf{x}\right)
\,.
\end{align*}
For the last layer, $l=L$, the proof is identical, except for the scalar product and the application of the activation function, which are removed. 
Since we assume that $\student$ and $\teacher$ have the same output dimension $d_L$, this proves that for all $\mathbf{x}\in\mathbb{R}^{d_0}$
\begin{align*}
\student \rb{\mathbf{x}} = \pi_{L} \rb{\student \rb{\mathbf{x}}} = \teacher \rb{\mathbf{x}}\,.
\end{align*}
That is, $\mathcal{E}\subseteq\left\{ \student \middle\vert \student \equiv \teacher \right\} $ and therefore 
\[
\mathbb{P}\left(\mathcal{E}\right)\le\mathbb{P}\left(\student\equiv \teacher\right)\,.
\]
Finally, to calculate the probability to sample $\student$ in $\mathcal{E}$ we count the number of constrained parameters - parameters which are either determined by $\teacher$ or are $0$ in $\mathcal{E}$.
Looking at the dimensions of $\mathbf{W}_{11}^{\rb{l}}$, $\mathbf{b}_{1}^{\rb{l}}$,
$\mathbf{b}_{2}^{\rb{l}}$
and $\bgamma^{\rb{l}}$, we deduce that there are exactly 
\begin{align*}
    \PC = \sum_{l=1}^L \rb{d_{l}^{\star}d_{l-1}^{\star} + 2d_{l}}
\end{align*}
such constrained parameters, and denote
\begin{align*}
    \BSCM{SFC} \triangleq \rb{\sum_{l=1}^{L}{\rb{d_{l}^{\star}d_{l-1}^{\star} + 2d_{l}}}}\log Q=\PC\log Q\,.
\end{align*}
As in the proof of \thmref{lem:int_quantized fc app}, under the uniform prior over parameters $\mathcal{P}$,
\begin{align*}
    \pteacher \ge \mathbb{P} \rb{\mathcal{E}} = Q^{-\PC}\,,
\end{align*}
so
\begin{align*}
    \SC=-\log\rb{\pteacher}\le\PC\log{Q}=\BSCM{SFC}
\end{align*}
\end{proof}
\newpage
\subsection{Convolutional Neural Network}\label{subseq: CNN app}

We first restate the definition of a CNN.

\begin{definition} [CNN restated] \label{def: cnn app} 
For multi-channel inputs
$\mathbf{x}^j ,\, \forall j\in\bb{c_0}$,
depth $L$, activation function $\sigma\rb{\cdot}$, and multi-index channels number
$D=\rb{c_1,\dots,c_L, d_s}$, $\cnn$ is a convolutional network (CNN) defined recursively, starting with $\forall j\in\left[c_0\right]$, $\cnnl{0}{j} \left(\mathbf{x}\right) = \mathbf{x}^j$, and then, for all $l\in\bb{L}$ and $i\in\left[c_l\right]$, as
\begin{align} 
\cnnl{l}{i} \left(\mathbf{x}\right) =& \sigma\left(\sum_{j=1}^{c_{l-1}} \mathbf{K}_{i, j}^{\left(l\right)} * \cnnl{l-1}{j} \left(\mathbf{x}\right) + b_i^{\left(l\right)}\right) \label{eq: cnn def appendix} \\
\cnn \rb{\mathbf{x}} =& \sign \rb{ \mathrm{Vec} \rb{\cnnlc{L-1} \rb{\mathbf{x}}} ^{\top}\mathbf{w}^{\left(L+1\right)} + b^{\left(L+1\right)}} \,, \nonumber
\end{align}
where for ease of notation the addition of the bias term in \eqref{eq: cnn def appendix} is done element-wise.
The network parameters are\begin{align*}
    \vect{\btheta} = \left( \left\{ \mathbf{K}^{\left(l\right)} \right\}_{l=1}^{L}, \left\{ \mathbf{b}^{\left(l\right)} \right\}_{l=1}^{L}, \vect{w}^{\rb{L+1}}, b^{\rb{L+1}}\right)\,,
\end{align*}
where $\mathbf{K}_{i,j}^{\left(l\right)} \in \mathbb{R}^{k_{l}}$ are convolution operators defined by kernels with $k_l$ parameters, and $\mathbf{b}^{\left(l\right)} \in \mathbb{R}^{c_{l}}$ are bias terms.
$\vect{w}^{\rb{L+1}}\in\R^{d_s}$ and $b^{\rb{L+1}}\in\R$ are the weights and bias of the convolutional network's last fully connected layer, where $d_s$ is the dimension of $\mathrm{Vec} \rb{\cnnlc{L} \rb{\mathbf{x}}}$. 
We denote the class of all convolutional networks with multi-index widths $D$ as $\Hcnn{D}$.
\end{definition}

\begin{remark}
    As in the case of FCNs, we use the ambiguous $\vechid{l}$ in place of $\cnnlc{l}$ to denote a general convolutional layer in a convolutional neural network with channel pattern $D$ where the specific type can be inferred from context.
\end{remark}

\textbf{Motivation.} In order to tackle the analysis of a convolutional neural network at ease, we will define some new operations. 
We observe that our previous arguments from \subsecref{subseq:quant fcn app} is agnostic to the convolution itself and relies only on parameter counting. 
In general, CNNs are analogous to FCNs in the sense that with spatial dimension of $1$, a convolutional layer is equivalent to a layer in a FCN, with each channel analogous to a neuron. 
Therefore, we define:

\begin{definition} \label{def: block convolution appendix}
For any $l\in\bb{L}$, define
\begin{align*}
\mathbf{K}^{\left(l\right)}=\begin{bmatrix}\mathbf{K}_{1,1}^{\left(l\right)} & \cdots & \mathbf{K}_{1,c_{l-1}}^{\left(l\right)}\\
\vdots & \ddots & \vdots\\
\mathbf{K}_{c_{l},1}^{\left(l\right)} & \cdots & \mathbf{K}_{c_{l},c_{l-1}}^{\left(l\right)}
\end{bmatrix}
\end{align*}
and 
\begin{align*}
\vechid{l}=
\begin{bmatrix}
\veccnl{l}{1} \\
\vdots\\
\veccnl{l}{c_l}
\end{bmatrix}\,.
\end{align*}

We use this notation to concisely write the \textit{multi-channel convolution} operation 
    \begin{align*}
        \left(\mathbf{K}^{\left(l\right)} * \vechid{l}\right)_{i}=\sum_{j=1}^{c_{l-1}}\mathbf{K}_{ij}^{\left(l\right)}* \veccnl{l-1}{j}\,.
    \end{align*}
\end{definition}

\textbf{Notation}
With these definitions, we define the extension to the block notation from \eqref{not: fc block notation appendix}: 
\begin{align*}
\mathbf{K}^{\left(l\right),\rb{11}}=\begin{bmatrix}\mathbf{K}_{1,1}^{\left(l\right)} & \cdots & \mathbf{K}_{1,c_{l-1}^{\star}}^{\left(l\right)}\\
\vdots & \ddots & \vdots\\
\mathbf{K}_{c_{l}^{\star},1}^{\left(l\right)} & \cdots & \mathbf{K}_{c_{l}^{\star},c_{l-1}^{\star}}^{\left(l\right)}
\end{bmatrix}\;,
\mathbf{K}^{\left(l\right),\rb{12}}=
\begin{bmatrix}\mathbf{K}_{1, c_{l-1}^{\star} + 1}^{\left(l\right)} & \cdots & \mathbf{K}_{1,c_{l-1}}^{\left(l\right)}\\
\vdots & \ddots & \vdots\\
\mathbf{K}_{c_{l}^{\star}, c_{l-1}^{\star} + 1}^{\left(l\right)} & \cdots & \mathbf{K}_{c_{l}^{\star},c_{l-1}}^{\left(l\right)}
\end{bmatrix}\;,
\end{align*}

\begin{align*}
\mathbf{K}^{\left(l\right),\rb{21}}=\begin{bmatrix}\mathbf{K}_{c_{l}^{\star} + 1,1}^{\left(l\right)} & \cdots & \mathbf{K}_{c_{l}^{\star} + 1, c_{l-1}^{\star}}^{\left(l\right)}\\
\vdots & \ddots & \vdots\\
\mathbf{K}_{c_{l},1}^{\left(l\right)} & \cdots & \mathbf{K}_{c_{l},c_{l-1}^{\star}}^{\left(l\right)}
\end{bmatrix}\;,
\mathbf{K}^{\left(l\right),\rb{22}}=
\begin{bmatrix}\mathbf{K}_{c_{l}^{\star} + 1, c_{l-1}^{\star} + 1}^{\left(l\right)} & \cdots & \mathbf{K}_{c_{l}^{\star} + 1, c_{l-1}}^{\left(l\right)}\\
\vdots & \ddots & \vdots\\
\mathbf{K}_{c_{l}, c_{l-1}^{\star} + 1}^{\left(l\right)} & \cdots & \mathbf{K}_{c_{l},c_{l-1}}^{\left(l\right)}
\end{bmatrix}\;,
\end{align*}
so
\begin{align*}
\mathbf{K}^{\rb{l}} = 
\begin{bmatrix}
    \mathbf{K}^{\rb{l}, \rb{11}} & \mathbf{K}^{\rb{l}, \rb{12}} \\
    \mathbf{K}^{\rb{l}, \rb{21}} & \mathbf{K}^{\rb{l}, \rb{22}}
\end{bmatrix}\,.
\end{align*}
Additionally, with
\begin{align*}
\vhblock{l}{1}=
\begin{bmatrix}
\veccnl{l}{1} \\
\vdots\\
\veccnl{l}{c_l^{\star}}
\end{bmatrix},\;
\vhblock{l}{2}=
\begin{bmatrix}
\veccnl{l}{c_l^{\star} + 1} \\
\vdots\\
\veccnl{l}{c_l}
\end{bmatrix},
\end{align*}
and \defref{def: block convolution appendix} we can write the multi-channel convolution in block form as
\begin{align*}
\mathbf{K}^{\rb{l}} \vechid{l} &= 
\begin{bmatrix}
    \mathbf{K}^{\rb{l}, \rb{11}} & \mathbf{K}^{\rb{l}, \rb{12}} \\
    \mathbf{K}^{\rb{l}, \rb{21}} & \mathbf{K}^{\rb{l}, \rb{22}}
\end{bmatrix} * 
\begin{bmatrix}
    \vhblock{l}{1} \\
    \vhblock{l}{2}
\end{bmatrix} \\
&= \begin{bmatrix}
    \mathbf{K}^{\rb{l}, \rb{11}} * \vhblock{l}{1} + \mathbf{K}^{\rb{l}, \rb{12}} * \vhblock{l}{2} \\
    \mathbf{K}^{\rb{l}, \rb{21}} * \vhblock{l}{1} + \mathbf{K}^{\rb{l}, \rb{22}} * \vhblock{l}{2}
\end{bmatrix}\,.
\end{align*}
Specifically, we can represent a convolutional layer, as defined in \defref{def: cnn app} by 
\begin{align} \label{eq: conv layer in block appendix}
\vechid{l} \rb{\mathbf{x}} = \sigma \rb{\mathbf{K}^{\rb{l}} * \vechid{l-1} + \mathbf{b}^{\rb{l}}}\,.
\end{align}

Finally, we denote
\begin{align*}
    \mathbf{w}^{\rb{L+1}} = 
    \begin{bmatrix}
        \mathbf{w}^{\rb{L+1},1} \\
        \mathbf{w}^{\rb{L+1},2}
    \end{bmatrix}
\end{align*}
where $\mathbf{w}^{\rb{L+1},1} \in \mathcal{Q}^{d_s^{\star}}$ and $\mathbf{w}^{\rb{L+1},2} \in \mathcal{Q}^{d_s - d_s^{\star}}$.

\begin{theorem} [Teacher equivalence probability for CNN restated] \label{thm: cnn pt appendix}
For any activation function and any depth $L$, let $D^{\star}=\rb{c_1^{\star},\dots,c_L^{\star}, d_{s}^{\star}},\; D=\rb{c_1,\dots,c_L, d_s}\in\mathbb{N}^{L+1}$ such that $D^{\star} \le D$.
If there exists some teacher $\teacher \in \Hcnn{D^{\star}}$ and the prior is $\mathcal{P}=\mathcal{P} \rb{\Hcnn{D}}$ then

\begin{align*}
    \SC \le \BSCM{CNN} \triangleq \rb{d_s + 1 + \sum_{l=1}^L k_l  c_{l}^{\star}c_{l-1}  + c_{l}^{\star}}\log\rb{Q}
\end{align*}
And by \lemref{thm: gen of gac with teacher simple}, $N=(\SC+3\log 2/\delta)/\varepsilon$ samples are enough to ensure that for posterior sampling (i.e.~\GaCtext{}), $\mathcal{L}(\mathcal{A}_\mathcal{P}(\calS))\leq\varepsilon$ with probability $1-\delta$ over $\calS\sim\calD^N$ and the sampling. 
\end{theorem}

\begin{proof}
We follow the same strategy as in the proof of \thmref{lem:int_quantized fc app}. 
We define a sufficient condition on the parameters of the student network to ensure teacher equivalence, and then find the probability that this condition holds.
Recall that for all $l\in\bb{L}$, $\mathbf{K}^{\star\left(l\right)}$ and $\mathbf{b}^{\star\left(l\right)}$, where $\mathbf{K}_{i,j}^{\star\left(l\right)} \in \mathbb{R}^{k_{l}}$ is a convolution kernel with $k_l$ parameters, and $\mathbf{b}^{\star\left(l\right)} \in \mathbb{R}^{c_{l}}$, are the teacher's $l^{\text{th}}$ layer's convolution kernels and bias terms, respectively.
Using the notation introduced in \defref{def: block convolution appendix}, the teacher's convolution kernels can be arranged as 
\begin{align*}
\mathbf{K}^{\star\left(l\right)}=\begin{bmatrix}\mathbf{K}_{1,1}^{\star\left(l\right)} & \cdots & \mathbf{K}_{1,c_{l-1}^{\star}}^{\star\left(l\right)}\\
\vdots & \ddots & \vdots\\
\mathbf{K}_{c_{l}^{\star},1}^{\star\left(l\right)} & \cdots & \mathbf{K}_{c_{l}^{\star},c_{l-1}^{\star}}^{\star\left(l\right)}
\end{bmatrix}\,.
\end{align*}
Define
\[
\mathcal{E}=\left\{ \student \in \calH_D^{\text{CNN}} \middle\vert 
\begin{array}{l}
\forall l=1,\dots,L\; \mathbf{K}^{\rb{l}, \rb{11}} = \mathbf{K}^{\star\left(l\right)},\; \mathbf{K}^{\rb{l}, \rb{12}} = \mathbf{0},\; \mathbf{b}_{1}^{\left(l\right)} = \mathbf{b}^{\star\left(l\right)}, \\
\mathbf{w}^{\rb{L+1},1} = \mathbf{w}^{\star\rb{L+1}},\; \mathbf{w}^{\rb{L+1}, 2} = \mathbf{0}_{d_s - d_s^{\star}},\; b^{\rb{L+1}} = b^{\star\rb{L+1}}
\end{array}\right\} \,.
\]
We claim that any $\student\in\mathcal{E}$ is TE and can show this by induction over the layer $l$.
With block notation, it is easy to see that this is identical to the proof of \thmref{lem:int_quantized fc app} when substituting $\mathbf{W}^{\rb{l}}$ with $\mathbf{K}^{\rb{l}}$, $\fcnhid{l}$ with $\cnnlc{l}$, and standard matrix multiplication with convolution. 
Finally, we turn to finding $\mathbb{P} \rb{\mathcal{E}}$.
For a CNN $\cnn \in \mathcal{E}$, there are $c_l^{\star} \cdot c_{l-1}^{\star}$ convolution kernels set to equal the teacher's, $c_l^{\star} \cdot \rb{c_{l-1} - c_{l-1}^{\star}}$ kernels set to $0$, and $c_l^{\star}$ bias terms.
Each convolution kernel $\mathbf{K}_{i,j}^{\rb{l}}$ is defined by $k_l$ parameters.
From the linear layer, we need to account for $d_s$ weight parameters and one bias term.
In total,
\begin{align*}
\pteacher \ge \mathbb{P} \rb{\mathcal{E}} = \frac{1}{Q^\PC}
\end{align*}
where
\begin{align*}
    \PC = \sum_{l=1}^L c_{l}^{\star} \cdot c_{l-1} \cdot k_l + c_{l}^{\star} + d_s + 1 = d_s + 1 + \sum_{l=1}^L  c_{l}^{\star} \cdot \rb{c_{l-1}\cdot k_l  + 1}\,.
\end{align*}
so
\begin{align*}
    \SC=-\log\rb{\pteacher}\le\PC\log{Q}=\BSCM{CNN}
\end{align*}
\end{proof}

\newpage
\subsection{Channel-Scaled Convolutional Neural Networks}\label{subsec: SCN app}

Analogous to the FCN case, we present an additional CNN architecture.
For simplicity, we state the definition of the architecture using the previously defined notation from \ref{def: block convolution appendix} and \eqref{eq: conv layer in block appendix}.

\begin{definition} [Channel Scaled CNN] \label{def: cscnn app} 
For multi-channel inputs
$\mathbf{x}^j ,\, \forall j\in\bb{c_0}$,
depth $L$, activation function $\sigma\rb{\cdot}$, and multi-index channels number
$D=\rb{c_1,\dots,c_L, d_s}$, $\scn$ is a channel scaled convolutional neural network, or \emph{scaled convolutional network} (SCN) for short, defined recursively, starting with $\scnlc{0} \left(\mathbf{x}\right) = \mathbf{x}$, and then, for all $l\in\bb{L}$, as
\begin{align} 
\scnlc{l} \left(\mathbf{x}\right) =&  \rb{\bgamma^{\rb{l}} \odot \sigma\mathbf{K}^{\rb{l}} * \scnlc{l-1} + \mathbf{b}^{\rb{l}}} \label{eq: scn def} \\
\scn \rb{\mathbf{x}} =& \sign \rb{ \mathrm{Vec} \rb{\scnlc{L} \rb{\mathbf{x}}} ^{\top}\mathbf{w}^{\left(L+1\right)} + b^{\left(L+1\right)}} \,, \nonumber
\end{align}
where for ease of notation the addition of the bias term in \eqref{eq: scn def} is done element-wise.
The network parameters are\begin{align*}
    \vect{\btheta} = \left( \left\{ \mathbf{K}^{\left(l\right)} \right\}_{l=1}^{L}, \, \left\{ \mathbf{b}^{\left(l\right)} \right\}_{l=1}^{L}, \, \cb{\bgamma^{\rb{l}}}_{l=1}^{L}, \, \vect{w}^{\rb{L+1}}, \, b^{\rb{L+1}}\right)\,,
\end{align*}
where $\mathbf{K}_{i,j}^{\left(l\right)} \in \mathbb{R}^{k_{l}}$ are convolution operators defined by kernels with $k_l$ parameters, $\mathbf{b}^{\left(l\right)} \in \mathbb{R}^{c_{l}}$ are bias terms, and $\bgamma^{\rb{l}}$ are channel scaling parameters.
$\vect{w}^{\rb{L+1}}\in\R^{s}$ and $b^{\rb{L+1}}\in\R$ are the weights and bias of the convolutional network's last fully connected layer, where $d_s$ is the dimension of $\mathrm{Vec} \rb{\scnlc{L} \rb{\mathbf{x}}}$. 
We denote the class of all SCNs with multi-index widths $D$ as $\Hscn{D}$.
\end{definition}

\begin{theorem} [Teacher equivalence probability for SCN restated] \label{thm: scn pt appendix}
For any activation function and any depth $L$, let $D^{\star}=\rb{c_1^{\star},\dots,c_L^{\star}, d_{s}^{\star}},\; D=\rb{c_1,\dots,c_L, d_s}\in\mathbb{N}^{L+1}$ such that $D^{\star} \le D$.
If there exists some teacher $\teacher \in \Hscn{D^{\star}}$ and the prior is $\mathcal{P}=\mathcal{P} \rb{\Hscn{D}}$ then
\begin{align*}
    \SC \le \BSCM{SCN} \triangleq \rb{d_s^{\star} + 1 + \sum_{\lconv=1}^{L}\rb{{c_{\lconv}^{\star}c_{\lconv-1}^{\star}k_{\lconv} + 2c_{l}}}}\log\rb{Q}
\end{align*}
And by \lemref{thm: gen of gac with teacher simple}, $N=(\SC+3\log 2/\delta)/\varepsilon$ samples are enough to ensure that for posterior sampling (i.e.~\GaCtext{}), $\mathcal{L}(\mathcal{A}_\mathcal{P}(\calS))\leq\varepsilon$ with probability $1-\delta$ over $\calS\sim\calD^N$ and the sampling. 
\end{theorem}

\begin{proof}
The proof is completely analogous to the proof of \thmref{lem:int_quantized sfc app}.
Specifically, we claim that
\[
\mathcal{E} = \left\{ \student \in \Hscn{D} \middle\vert 
\begin{array}{l}
\forall l=1,\dots,L\; \mathbf{K}^{\rb{l}, \rb{11}} = \mathbf{K}^{\star\left(l\right)}, \; \bgamma_{1}^{\rb{l}} = \bgamma_{1}^{\star\rb{l}},\; \bgamma_{2}^{\rb{l}} = \mathbf{0}, \\
\mathbf{b}_{1}^{\left(l\right)} = \mathbf{b}^{\star\left(l\right)},\; \mathbf{w}^{\rb{L+1},1} = \mathbf{w}^{\star\rb{L+1}},\;  b^{\rb{L+1}} = b^{\star\rb{L+1}}
\end{array}\right\} \,.
\]
is a subset of the teacher equivalent SCNs, and therefore
\begin{align*}
    \pteacher \ge \mathbb{P} \rb{\mathcal{E}} = \frac{1}{Q^{\PC}}
\end{align*}
where
\begin{align*}
\PC \triangleq d_{s}^{\star} + 1 + \sum_{\lconv=1}^{L}\rb{{c_{\lconv}^{\star}c_{\lconv-1}^{\star}k_{\lconv} + 2c_{l}}}\,.
\end{align*}
so
\begin{align*}
    \SC=-\log\rb{\pteacher}\le\PC\log{Q}=\BSCM{SCN}
\end{align*}
\end{proof}

\newpage
\section{Proof for the Interpolation Probability in the Continuous Case (\secref{sec: continuous})} 
\label{app:cont_proofs}

\subsection{Technical Lemmas}
To prove the results about continuous single hidden layer NN, we rely on
the following basic Lemma. 
\begin{lemma} \label{lem: angles between random Gaussian vectors appendix}
For any vector $\mathbf{y}$ and random vector $\mathbf{x}\sim\left(\mathbf{0},\mathbf{I}_{d}\right)$, $\varepsilon\in\left(0,\frac{\pi}{2}\right)$ and $u\in\left(0,1\right)$
we have 
\begin{equation}
\mathbb{P}\left(\frac{\mathbf{x}^{\top}\mathbf{y}}{\left\Vert \mathbf{x}\right\Vert \left\Vert \mathbf{y}\right\Vert }>\cos\left(\varepsilon\right)\right)=\frac{1}{2}\mathbb{P}\left(\left|\frac{\mathbf{x}^{\top}\mathbf{y}}{\left\Vert \mathbf{x}\right\Vert \left\Vert \mathbf{y}\right\Vert }\right|>\cos\left(\varepsilon\right)\right)\geq\frac{\sin\left(\varepsilon\right)^{d-1}}{\left(d_{0}-1\right)B\left(\frac{1}{2},\frac{d-1}{2}\right)}\label{eq: Beta Bound 1}
\end{equation}
\begin{align}
\mathbb{P}\left(\left|\frac{\mathbf{x}^{\top}\mathbf{y}}{\left\Vert \mathbf{x}\right\Vert \left\Vert \mathbf{y}\right\Vert }\right|<u\right) & \leq\frac{2u}{B\left(\frac{1}{2},\frac{d-1}{2}\right)},\label{eq: Beta Bound 2}
\end{align}
where we use $B\left(x,y\right)$ to denote the beta function.    
\end{lemma}

\begin{proof}
Since $\mathcal{N}\left(0,\mathbf{I}_d\right)$ is spherically
symmetric, we have 
\[
\mathbb{P}\left(\frac{\mathbf{x}^{\top}\mathbf{y}}{\left\Vert \mathbf{x}\right\Vert \left\Vert \mathbf{y}\right\Vert }>\cos\left(\varepsilon\right)\right)=\frac{1}{2}\mathbb{P}\left(\left|\frac{\mathbf{x}^{\top}\mathbf{y}}{\left\Vert \mathbf{x}\right\Vert \left\Vert \mathbf{y}\right\Vert }\right|>\cos\left(\varepsilon\right)\right)\,,
\]

and we can set $\mathbf{y}=\left[1,0\dots,0\right]^{\top}$, without
loss of generality. 
Therefore, as in \cite{2263641} 
\[
\left|\frac{\mathbf{x}^{\top}\mathbf{y}}{\left\Vert \mathbf{x}\right\Vert \left\Vert \mathbf{y}\right\Vert }\right|^{2}=\frac{x_{1}^{2}}{x_{1}^{2}+\sum_{i=2}^{d}x_{i}^{2}}\sim\mathcal{B}\left(\frac{1}{2},\frac{d-1}{2}\right),
\]
where $\mathcal{B}$ denotes the Beta distribution, since $x_{1}^{2}\sim\chi^{2}\left(1\right)$
and $\sum_{i=2}^{d}x_{i}^{2}\sim\chi^{2}\left(d-1\right)$
are independent chi-square random variables.

Suppose $Z\sim\mathcal{B}\left(\alpha,\beta\right)$, $\alpha\in\left(0,1\right)$,
and $\beta>1$ .
\begin{align*}
\mathbb{P}\left(Z>u\right)= & \frac{\int_{u}^{1}x^{\alpha-1}\left(1-x\right)^{\beta-1}dx}{B\left(\alpha,\beta\right)}\geq\frac{\int_{u}^{1}1^{\alpha-1}\left(1-x\right)^{\beta-1}dx}{B\left(\alpha,\beta\right)}=\frac{\int_{0}^{1-u}x^{\beta-1}dx}{B\left(\alpha,\beta\right)}=\frac{\left(1-u\right)^{\beta}}{\beta B\left(\alpha,\beta\right)}\,.
\end{align*}
Therefore, for $\varepsilon>0$,
\begin{align*}
\mathbb{P}\left(\left|\frac{\mathbf{x}^{\top}\mathbf{y}}{\left\Vert \mathbf{x}\right\Vert \left\Vert \mathbf{y}\right\Vert }\right|^{2}>\cos^{2}\left(\varepsilon\right)\right) & \geq\frac{2\left(1-\cos^{2}\left(\varepsilon\right)\right)^{\frac{d-1}{2}}}{\left(d-1\right)B\left(\frac{1}{2},\frac{d-1}{2}\right)}=\frac{2\sin\left(\varepsilon\right)^{d-1}}{\left(d-1\right)B\left(\frac{1}{2},\frac{d-1}{2}\right)}\,,
\end{align*}
which proves \eqref{eq: Beta Bound 1}.

Similarly, for $\alpha\in\left(0,1\right)$ and $\beta>1$ 
\begin{align*}
\mathbb{P}\left(Z<u\right) & =\frac{\int_{0}^{u}x^{\alpha-1}\left(1-x\right)^{\beta-1}dx}{B\left(\alpha,\beta\right)}\leq\frac{\int_{0}^{u}x^{\alpha-1}1^{\beta-1}dx}{B\left(\alpha,\beta\right)}=\frac{u^{\alpha}}{\alpha B\left(\alpha,\beta\right)}\,.
\end{align*}
Therefore,
\[
\mathbb{P}\left(\left|\frac{\mathbf{x}^{\top}\mathbf{y}}{\left\Vert \mathbf{x}\right\Vert \left\Vert \mathbf{y}\right\Vert }\right|^{2}<u^{2}\right)\leq\frac{2u}{B\left(\frac{1}{2},\frac{d_{0}-1}{2}\right)}\,,
\]
which proves \eqref{eq: Beta Bound 2}. 
\end{proof}

\begin{lemma} \label{lem: asymptotic expansion of the beta function appendix} 
For large $x$
\begin{align}
    B\left(\frac{1}{2},x\right)=\sqrt{\pi/x}+O\left(x^{-3/2}\right)\,.
\end{align}
\end{lemma}
\begin{proof}
Using $\Gamma\left(\frac{1}{2}\right)=\sqrt{\pi}$, Stirling's approximation for the Gamma function 
\begin{align*}
\Gamma \left(x\right) = \sqrt{\frac{2\pi}{x}}\left(\frac{x}{e}\right)^{x}\left(1+O\left(x^{-1}\right)\right)\,,
\end{align*}
and the definition of the Beta function,
\begin{align*}
B\left(\frac{1}{2},x\right)&=\frac{\Gamma\left(\frac{1}{2}\right)\Gamma\left(x\right)}{\Gamma\left(\frac{1}{2}+x\right)} \\
&= \frac{\sqrt{\pi}\sqrt{\frac{2\pi}{x}}\left(\frac{x}{e}\right)^{x}\left(1+O\left(x^{-1}\right)\right)}{\sqrt{\frac{2\pi}{x+\frac{1}{2}}}\left(\frac{x+\frac{1}{2}}{e}\right)^{x+\frac{1}{2}}\left(1+O\left(x^{-1}\right)\right)}\\
&= \sqrt{e\cdot\pi}\sqrt{\frac{1}{x}}\left(\frac{x}{x+\frac{1}{2}}\right)^{x}\left(1+O\left(x^{-1}\right)\right) \\
&= \sqrt{e\cdot\pi}\sqrt{\frac{1}{x}}\frac{1}{\left(1+\frac{1}{2x}\right)^x}\left(1+O\left(x^{-1}\right)\right) \\
&= \sqrt{\frac{\pi}{x}}\left(1+O\left(x^{-1}\right)\right)\,.
\end{align*}
\end{proof}

\begin{corollary}\label{cor: angles between high dimensional random Gaussian vectors appendix}
For any vector $\mathbf{y}$ and $\mathbf{x}\sim\left(\mathbf{0},\mathbf{I}_{d}\right)$, $\varepsilon\in\left(0,\frac{\pi}{2}\right)$ and $u\in\left(0,1\right)$
we have 
\begin{align*}
\mathbb{P}\left(\frac{\mathbf{x}^{\top}\mathbf{y}}{\left\Vert \mathbf{x}\right\Vert \left\Vert \mathbf{y}\right\Vert }>\cos\left(\varepsilon\right)\right)&=\frac{1}{2}\mathbb{P}\left(\left|\frac{\mathbf{x}^{\top}\mathbf{y}}{\left\Vert \mathbf{x}\right\Vert \left\Vert \mathbf{y}\right\Vert }\right|>\cos\left(\varepsilon\right)\right)\\
&\geq\exp\left(d\log\left(\sin\left(\varepsilon\right)\right)-\frac{1}{2}\log\left(d\right)-\frac{1}{2}\log\left(2\pi\right)\right)\left(1+O\left(d^{-1}\right)\right)\,.\label{eq: Beta Bound large dim 1 appendix}
\end{align*}
\end{corollary}
\begin{proof}
\begin{align*}
\frac{\sin\left(\varepsilon\right)^{d-1}}{\left(d-1\right)B\left(\frac{1}{2},\frac{d-1}{2}\right)}
 & =\frac{\sin\left(\varepsilon\right)^{d-1}}{\left(d-1\right)\sqrt{\frac{2\pi}{d-1}}\left(1+O\left(\frac{1}{d-1}\right)\right)}\\
 & =\frac{\sin\left(\varepsilon\right)^{d-1}}{\sqrt{2\pi\left(d-1\right)}}\left(1+O\left(d^{-1}\right)\right)\\
 & =\frac{\sin\left(\varepsilon\right)^{d-1}}{\sqrt{2\pi\left(d-1\right)}}\left(1+O\left(d^{-1}\right)\right)\\
 & =\exp\left(\left(d-1\right)\log\left(\sin\left(\varepsilon\right)\right)-\frac{1}{2}\log\left(d-1\right)-\frac{1}{2}d_{1}^{\star}\log\left(2\pi\right)\right)\left(1+O\left(d^{-1}\right)\right)\\
 & =\exp\left(-\left(d-1\right)\left|\log\left(\sin\left(\varepsilon\right)\right)\right|-\frac{1}{2}\log\left(d-1\right)-\frac{1}{2}\log\left(2\pi\right)\right)\left(1+O\left(d^{-1}\right)\right)\\
 & >\exp\left(-d\left|\log\left(\sin\left(\varepsilon\right)\right)\right|-\frac{1}{2}\log\left(d\right)-\frac{1}{2}\log\left(2\pi\right)\right)\left(1+O\left(d^{-1}\right)\right)\,.
\end{align*}
\end{proof}

For completeness, we follow with the setting and notation introduced in \secref{sec: continuous} with slight modification.

\subsection{Setting and Notation}
Let $\student, \teacher$ be fully connected (\defref{def: vanilla fcn}) two layer neural network models with input dimension $d_{0}$, output dimension $d_{2}=1$ and hidden layer dimensions $d_{1}$ and $d_{1}^{\star}$, respectively.
To simplify notation, we omit the $\sign$ activation from the definition of $\student$ and $\teacher$ and denote
\begin{equation*}
\student\left(\mathbf{x}\right) = \mathbf{z}^{\top} \sigma\left(\mathbf{W}\mathbf{x}\right)
\end{equation*}
\begin{equation*}
\teacher\left(\mathbf{x}\right) = \mathbf{z}^{\star \top} \sigma\left(\mathbf{W}^{\star}\mathbf{x}\right)
\end{equation*}
where
\[
\mathbf{W}=\left[\mathbf{w}_{1},\dots,\mathbf{w}_{d_{1}}\right]^{\top}\in\mathbb{R}^{d_{1}\times d_{0}}\,,\,\mathbf{z}\in\mathbb{R}^{d_{1}}\,,
\]
\[
\mathbf{W}^{\star} = \left[ \mathbf{w}_{1}^{\star}, \dots, \mathbf{w}_{d_{1}^{\star}}^{\star}, \mathbf{0}, \dots, \mathbf{0} \right]^{\top} \in \mathbb{R}^{d_{1} \times d_{0}}\,,\, \mathbf{z}^{\star} \in \mathbb{R}^{d_2\times d_1}\,,
\]
and $\sigma\left(\cdot\right)$ is the common leaky rectifier linear unit (LReLU, \citet{Maas2013RectifierNI}) with parameter $\rho\notin \{0,1\}$.
\begin{equation}
\sigma\left(u\right)= ua\left(u\right)\,\,\mathrm{with}\,\,a\left(u\right)=\begin{cases}
1 & ,\,\mathrm{if}\,u>0\\
\rho & ,\,\mathrm{if}\,u<0
\end{cases}\,,\label{eq: LReLU appendix}
\end{equation}
The training set $\mathbf{X}=\left[\mathbf{x}^{\left(1\right)},\dots,\mathbf{x}^{\left(N\right)}\right]\in\mathbb{R}^{d_{0}\times N}$
consists of $N$ datapoints.
Thus, the output of the FCN on the entire dataset can be written
as 
\begin{equation}
\student \left( \mathbf{X} \right) = \sigma \left( \mathbf{W} \mathbf{X} \right)^{\top} \mathbf{z} \in \mathbb{R}^{N}.\label{eq: NN appendix}
\end{equation}
\begin{equation}
\teacher\left(\mathbf{X}\right)=\sigma\left(\mathbf{W}^{\star}\mathbf{X}\right)^{\top}\mathbf{z}^{\star}\in\mathbb{R}^{N}.\label{eq: teacher NN appendix}
\end{equation}
We denote the labels $y^{\left( n \right)} = \sign \left( \teacher \left(\mathbf{x}^{\left(n\right)}\right)\right)$, and 
\[ \mathbf{y}=\left[y^{\left(1\right)},\dots,y^{\left(N\right)}\right]^{\top}\in\left\{ \pm 1 \right\} {}^{N}
\]
We use the notation $\boldsymbol{a}^{\left(n\right)} = a \left( \mathbf{W}\mathbf{x}^{\left( n \right)} \right) \in \left\{\rho, 1 \right\}^{d_1}$ for the activation pattern of the hidden layer of $\student$ on the input $\mathbf{x}^{\left( n \right)}$.
In addition, we define the flattened weights' vectors of $\student$ and $\teacher$ as 
\begin{align*}
\mathbf{w} & =\mathrm{vec}\left(\mathbf{W}{}^{\top}\mathrm{diag}\left(\mathbf{z}\right)\right)\in\mathbb{R}^{d_{0}d_{1}}\\
\mathbf{w}^{\star} & =\mathrm{vec}\left(\mathbf{W^{\star}}^{\top}\mathrm{diag}\left(\mathbf{z}^{\star}\right)\right)\in\mathbb{R}^{d_{0}d_{1}}\,,
\end{align*}
respectively, where $\mathbf{W}^{\star}$ and $\mathbf{z}^{\star}$ are padded with $0$'s to match the dimensions.
Let $\boldsymbol{\phi}^{\left(n\right)} =\left(\boldsymbol{a}^{\left(n\right)}\otimes\mathbf{x}^{\left(n\right)}\right)y^{\left(n\right)}\in\mathbb{R}^{d_{0}d_{1}}$ where $\otimes$ is the Kronecker product. 
With this notation, it can be shown that 
\begin{align*}
y^{\left(n\right)}h\left(\mathbf{x}^{\left(n\right)}\right)&=\mathbf{w}^\top \mathbf{\phi}^{\left(n\right)}\,.
\end{align*}

\begin{assumption}[Prior over hypotheses, continuous setting, restated] \label{asm: continuous prior appendix}
    Suppose that the weights of $\student$ are random such that each row of the first layer, $\mathbf{w}_i$, is independently sampled from a uniform distribution on the unit sphere $\mathbb{S}^{d_0 -1}$, and the second layer $\mathbf{z}$ is sampled uniformly from $\mathbb{S}^{d_1 - 1}$.
    Both $\mathbf{w}_i$ and $\mathbf{z}$ are independent of the teacher and data.
\end{assumption}

\subsection{Proof of \thmref{thm: interpolation probability of continuous networks}}
\begin{definition}[First layer angular margin, restated] \label{def: First layer angular margin appendix}
    For any training set $\calS=\{\vect{x}_n \}_{n=1}^N$, we say that $\calS$ has \emph{first layer angular margin $\alpha$} w.r.t. the teacher if
    \begin{equation}
    \forall i\in[d_1^{\star}]\,,\,n\in[N]:\left|\frac{\mathbf{x}_n^{\top}\mathbf{w}_{i}^{\star}}{\left\Vert \mathbf{x}_n\right\Vert_2 \left\Vert \mathbf{w}_{i}^{\star}\right\Vert_2 }\right|>\sin\alpha \,.
    \end{equation}
\end{definition}
We denote the event that $\teacher$ and $\student$ agree on the activation pattern of the data by 
\begin{equation}
\tilde{\mathcal{G}}\left(\mathbf{X},\mathbf{W}^{\star}\right)\triangleq\left\{ \mathbf{W}\in\mathbb{R}^{d_{1}\times d_{0}}|\forall i\in d_{1}^{\star}:\,\,\mathrm{sign}\left(\mathbf{w}_{i}^{\top}\mathbf{X}\right)=\mathrm{sign}\left(\mathbf{w}_{i}^{\star\top}\mathbf{X}\right)\right\} \,.\label{eq: activation matching event appendix}
\end{equation}
To bound the probability of this event, we use the following Lemma, adapted from \citet{soudry2017exponentially}. For completeness, we write its proof here. 
\begin{lemma} [Activation matching probability, adapted from \citet{soudry2017exponentially}] \label{lem: activation matching probability appendix}
Let $\mathbf{X}$ be a dataset with first layer angular margin $\alpha$ w.r.t $\teacher$.
Then
\begin{align*}
\mathbb{P}_{\mathbf{W}}\left(\tilde{\mathcal{G}}\left(\mathbf{X},\mathbf{W}^{\star}\right)\right)\ge\left[\frac{\sin\left(\alpha\right)^{d_{0}-1}}{\left(d_{0}-1\right)B\left(\frac{1}{2},\frac{d_{0}-1}{2}\right)}\right]^{d_{1}^{\star}}\,,
\end{align*}
and when $d_0 \gg 1$ and $\frac{d_{1}^{\star}}{d_{0}} \ll 1$
\begin{align*}
\mathbb{P}_{\mathbf{W}}\left(\tilde{\mathcal{G}}\left(\mathbf{X},\mathbf{W}^{\star}\right)\right)\ge\exp\left(d_{1}^{\star}d_{0}\log\left(\sin\left(\alpha\right)\right)-\frac{1}{2}d_{1}^{\star}\log\left(d_{0}\right)-\frac{1}{2}d_{1}^{\star}\log\left(2\pi\right)\right)\left(1+O\left(d_{0}^{-1}d_{1}^{\star}\right)\right)\,.
\end{align*}
\end{lemma}
\begin{proof}
To bound $\mathbb{P}_{\mathbf{W}}\left(\tilde{\mathcal{G}}\left(\mathbf{X},\mathbf{W}^{\star}\right)\right)$,
we define the event that all weight hyperplanes with normals $\mathbf{w}_{i}$,
have an angle of at most $\alpha$ from the corresponding target hyperplanes with normals $\mathbf{w}_{i}^{\star}$.
\[
\forall i \in \left[d_1^{\star}\right]\, \tilde{\mathcal{G}}_{i}^{\alpha}\left(\mathbf{W}^{\star}\right)\triangleq\left\{ \mathbf{W}\in\mathbb{R}^{d_{1}\times d_{0}}|\frac{\mathbf{w}_{i}^{\top}\mathbf{w}_{i}^{\star}}{\left\Vert \mathbf{w}_{i}\right\Vert \left\Vert \mathbf{w}_{i}^{\star}\right\Vert }>\cos\left(\alpha\right)\right\} \,.
\]
Since $\mathbf{X}$ has first layer angular margin $\alpha$, in order that $\mathrm{sign}\left(\mathbf{w}_{i}^{\top}\mathbf{x}^{\left(n\right)}\right)\neq\mathrm{sign}\left(\mathbf{w}_{i}^{\star\top}\mathbf{x}^{\left(n\right)}\right)$,
$\mathbf{w}_{i}$ must be rotated in respect to $\mathbf{w}_{i}^{\star}$
by an angle greater then the angular margin $\alpha$. 
Therefore, we have that
\begin{equation}
\bigcap_{i=1}^{d_{1}^{\star}}\tilde{\mathcal{G}}_{i}^{\alpha}\left(\mathbf{W}^{\star}\right)\subset\tilde{\mathcal{G}}\left(\mathbf{X},\mathbf{W}^{\star}\right)\,.\label{eq: G subsets}
\end{equation}
And so,
\begin{align}
\mathbb{P}_{\mathbf{W}}\left(\tilde{\mathcal{G}}\left(\mathbf{X},\mathbf{W}^{\star}\right)\right) & \overset{\left(1\right)}{\geq}\mathbb{P}_{\mathbf{W}}\left(\bigcap_{i=1}^{d_{1}^{\star}}\tilde{\mathcal{G}}_{i}^{\alpha}\left(\mathbf{W}^{\star}\right)\right)\,\label{eq: PG_W}\\
&\overset{\left(2\right)}{=}\prod_{i=1}^{d_{1}^{\star}}\mathbb{P}_{\mathbf{W}}\left(\mathbf{W}\in\tilde{\mathcal{G}}_{i}^{\alpha}\left(\mathbf{W}^{\star}\right)\right) \\
&=\prod_{i=1}^{d_{1}^{\star}}\mathbb{P}_{\mathbf{W}}\left(\frac{\mathbf{w}_{i}^{\top}\mathbf{w}_{i}^{\star}}{\left\Vert \mathbf{w}_{i}\right\Vert \left\Vert \mathbf{w}_{i}^{\star}\right\Vert }>\cos\left(\alpha\right)\right) \\
&\overset{\left(3\right)}{\geq}\left[\frac{\sin\left(\alpha\right)^{d_{0}-1}}{\left(d_{0}-1\right)B\left(\frac{1}{2},\frac{d_{0}-1}{2}\right)}\right]^{d_{1}^{\star}}\,,
\end{align}
where in (1) we used \eqref{eq: G subsets}, in (2) we used the independence of $\left\{ \mathbf{w}_{i}\right\} _{i=1}^{d_{1}^{\star}}$ and in (3) we used \eqref{lem: angles between random Gaussian vectors appendix}.
When $d_0 \gg 1$, we can use \corref{cor: angles between high dimensional random Gaussian vectors appendix} to get
\begin{align*}
\mathbb{P}_{\mathbf{W}}\left(\mathbf{W}\in\tilde{\mathcal{G}}\left(\mathbf{X},\mathbf{W}^{\star}\right)\right) \ge \exp\left(d_{1}^{\star}d_{0}\log\left(\sin\left(\alpha\right)\right)-\frac{1}{2}d_{1}^{\star}\log\left(d_{0}\right)-\frac{1}{2}d_{1}^{\star}\log\left(2\pi\right)\right)\left(1+O\left(d_{0}^{-1}\right)\right)^{d_{1}^{\star}}.
\end{align*}

We can simplify this equation when $d_0 \gg 1$ with the asymptotic expansion of the beta function from \lemref{lem: asymptotic expansion of the beta function appendix}.
If $d_{1}^{\star}d_{0}^{-1}\ll0$ then  the error $\left(1+O\left(d_{0}^{-1}\right)\right)^{d_{1}^{\star}}=1+O\left(d_{0}^{-1}d_{1}^{\star}\right)$.
Overall, this means that 
\begin{align*}
\mathbb{P}_{\mathbf{W}}\left(\mathbf{W}\in\tilde{\mathcal{G}}\left(\mathbf{X},\mathbf{W}^{\star}\right)\right) & \ge\left[\frac{\sin\left(\alpha\right)^{d_{0}-1}}{\sqrt{2\pi\left(d_{0}-1\right)}}\right]^{d_{1}^{\star}}\left(1+O\left(d_{1}^{\star}d_{0}^{-1}\right)\right)\\
 & \ge\exp\left(d_{1}^{\star}d_{0}\log\left(\sin\left(\alpha\right)\right)-\frac{1}{2}d_{1}^{\star}\log\left(d_{0}\right)-\frac{1}{2}d_{1}^{\star}\log\left(2\pi\right)\right)\left(1+O\left(d_{0}^{-1}d_{1}^{\star}\right)\right)
\end{align*}
\end{proof}

\begin{definition}[Second layer angular margin, restated] \label{def: Second layer angular margin appendix}
    For any training set $\calS=\{\rb{\vect{x}^{\rb{n}}} \}_{n=1}^N$, we say that $\calS$ has \emph{second layer angular margin $\beta$} w.r.t. the teacher if
    \begin{equation} 
    \forall n\in\left[N\right] \; \left|\frac{\teacher\rb{\mathbf{x}^{\rb{n}}})}{\left\Vert\mathbf{x}^{\rb{n}}\right\Vert_2 \left\Vert \mathbf{z}^{\star}\right\Vert_2 }\right|>\sqrt{d_1 \left(1+\rho^2\right)}\sin\beta\,.
    \end{equation}
\end{definition}

\begin{assumption}\label{asm: angular margin appendix}
    Let $\alpha<\beta\in\rb{0,\frac{\pi}{2}}$. 
    There exists $\lambda\in\rb{0,1}$ such that, with probability at least $1-\lambda$ over the training set $\calS=\{\rb{\vect{x}_n} \}_{n=1}^N\sim \mathcal{D}^N$, $\calS$ has first layer angular margin $\alpha$ (\defref{def: First layer angular margin appendix}) and second layer angular margin $\beta$ (\defref{def: Second layer angular margin appendix}).
\end{assumption}

\begin{figure}
    \centering
    \includegraphics[width=0.48\textwidth]{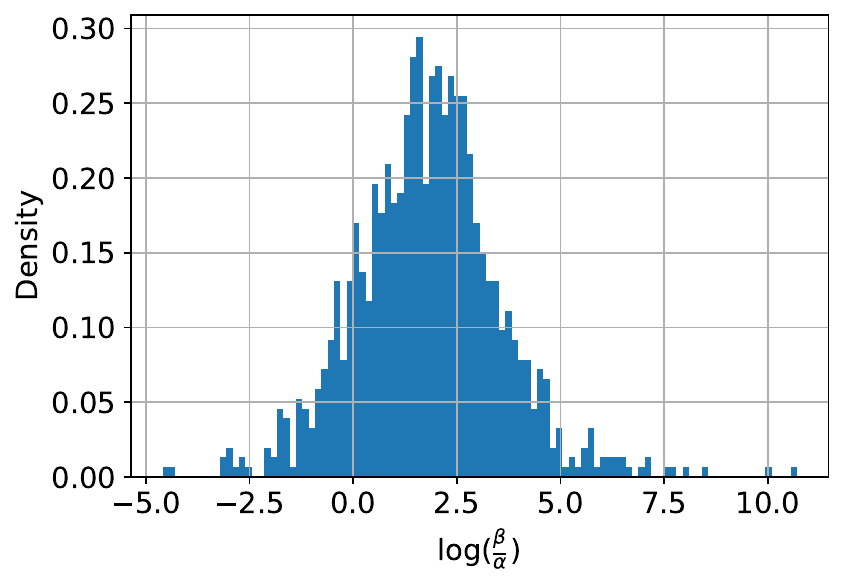}
    \caption{The density of the log of the ratio between $\beta$ and $\alpha$, for standard-Gaussian data, and a two-layer neural network with $\rho=0.01$ and $d_0=500$, $d_1=10,000$, $d_1^{\star}=1,000$. We sampled $50,000$ such datapoints and calculated $\alpha, \beta$ as the minimal angles as in \eqref{eq: angular margin} and \eqref{eq:output angular margin} for a randomly initialized model, for a total of $1,000$ times.}
    \label{fig:alpha-beta-sim}
\end{figure}

\begin{theorem} [Interpolation of Continuous Networks, restated] \label{thm: interpolation probability of continuous networks appendix}
Under \asmref{asm: angular margin appendix}, with probability at least $1 - \lambda$ over the dataset $\mathcal{S}$, the probability of interpolation is lower-bounded by
\begin{align*}
\hat{p}_\mathcal{S} &= \mathbb{P}_{\mathbf{W},\mathbf{z}}\left(\mathcal{L}_S \left(h\right)=0\right) \\
& \ge \exp{\left( d_{1}^{\star} d_{0} \log\left(\sin\left(\alpha\right)\right) + d_{1} \log\left(\sin\left(\gamma\right)\right) - \frac{1}{2} d_1^{\star} \log \rb{d_0} + O\left( d_1^{\star} + \log{\left(d_1\right)}\right)\right)}
\end{align*}
where $\gamma = \arccos{\frac{\cos{\beta}}{\cos{\alpha}}}$.
\end{theorem}

\begin{proof}
By \asmref{asm: angular margin appendix} with probability at least $1 - \lambda$, $\mathbf{X}$ has first and second layers angular margins $\alpha$ and $\beta$, respectively, w.r.t $\teacher$.
We assume that these properties hold for the rest of the proof.
Suppose that we condition on the event $\bigcap_{i=1}^{d_{1}^{\star}}\tilde{\mathcal{G}}_{i}^{\alpha}\left(\mathbf{W}^{\star}\right)$ from \eqref{eq: G subsets}.
Then, as in the proof of \lemref{lem: activation matching probability appendix}, under the margin assumptions, $\student$ and $\teacher$ agree on the activation pattern $\mathbf{a}^{\left(n\right)}$ and therefore 
\begin{align*}
y^{\left(n\right)}h^{\star}\left(\mathbf{x}^{\left(n\right)}\right)&=\mathbf{w}^{\star\top} \mathbf{\phi}^{\left(n\right)}
\end{align*}
in addition to
\begin{align*}
y^{\left(n\right)}h\left(\mathbf{x}^{\left(n\right)}\right)&=\mathbf{w}^\top \mathbf{\phi}^{\left(n\right)}\,.
\end{align*}
Using basic properties of the Kronecker product,
\begin{align*}
    \left\Vert \mathbf{\boldsymbol{\phi}}^{\left(n\right)}\right\Vert ^{2} =\left\Vert \mathbf{a}^{\left(n\right)}\right\Vert^{2}\left\Vert \mathbf{x}^{\left(n\right)}\right\Vert^{2}\,.
\end{align*}
Furthermore, since $\mathbf{w}_i,\mathbf{w}_i^{\star}\in\mathbb{S}^{d_0-1}$,
\begin{align*}
    \left\Vert \mathbf{w}\right\Vert ^{2}=\sum_{i=1}^{d_{1}}z_{i}^{2}\left\Vert \mathbf{w}_{i}\right\Vert ^{2}=\left\Vert \mathbf{z}\right\Vert ^{2}
\end{align*}
and similarly $\left\Vert \mathbf{w}^{\star}\right\Vert ^{2}=\left\Vert \mathbf{z}^{\star}\right\Vert ^{2}$.
$\mathbf{a}^{\left(n\right)}\in\left\{1,\rho\right\}^{d_1}$ so
\begin{align*}
    \left\Vert \mathbf{a}^{\left(n\right)}\right\Vert^{2} \le d_1 \cdot \left(1 + \rho^2\right)\,.
\end{align*}
With these identities we deduce
\begin{align*}
    \left|\frac{\mathbf{w}^{\star\top} \mathbf{\phi}^{\left(n\right)}}{\left\Vert\mathbf{w}^{\star}\right\Vert \left\Vert \mathbf{\phi}^{\left(n\right)}\right\Vert }\right| &= \left|\frac{h^{\star}\left(\mathbf{x}^{\left(n\right)}\right)}{\left\Vert\mathbf{x}^{\left(n\right)}\right\Vert\left\Vert\mathbf{a}^{\left(n\right)}\right\Vert \left\Vert \mathbf{z}^{\star}\right\Vert }\right| \ge \frac{1}{\sqrt{d_1 \left(1+\rho^2\right)}}\left|\frac{h^{\star}\left(\mathbf{x}^{\left(n\right)}\right)}{\left\Vert\mathbf{x}^{\left(n\right)}\right\Vert  \left\Vert \mathbf{z}^{\star}\right\Vert }\right|\,.
\end{align*}
Using the second layer angular margin, 
\begin{align*}
\left|\frac{h^{\star}\left(\mathbf{x}^{\left(n\right)}\right)}{\left\Vert\mathbf{x}^{\left(n\right)}\right\Vert_2 \left\Vert \mathbf{z}^{\star}\right\Vert_2 }\right|>\sqrt{d_1 \left(1+\rho^2\right)}\sin\beta\ge\left\Vert\mathbf{a}^{\left(n\right)}\right\Vert\sin\beta
\end{align*}
so
\begin{align}\label{eq: two layer margin appendix}
\left|\frac{\mathbf{w}^{\star\top} \mathbf{\phi}^{\left(n\right)}}{\left\Vert\mathbf{w}^{\star}\right\Vert \left\Vert \mathbf{\phi}^{\left(n\right)}\right\Vert }\right|>\sin\beta\,.
\end{align}
Following the same logic as in the proof of \lemref{lem: activation matching probability appendix}, in order that $\sign\left(y^{\left(n\right)}\right)\neq\sign\left(h\left(\mathbf{x}^{\left(n\right)}\right)\right)$, $\mathbf{w}$ must be rotated by angle at least $\beta$ compared to $\mathbf{w}^{\star}$.
That is,
\begin{equation}
\frac{\mathbf{w}^{\top}\mathbf{w}^{\star}}{\left\Vert \mathbf{w}\right\Vert \left\Vert \mathbf{w}^{\star}\right\Vert }>\cos\beta\label{eq: output margin condition appendix}
\end{equation}
implies interpolation of the dataset.
From symmetry, the probability of \eqref{eq: output margin condition appendix} is exactly half that of 
\begin{align*}
\frac{\left(\mathbf{w}^{\top}\mathbf{w}^{\star}\right)^{2}}{\left\Vert \mathbf{w}\right\Vert ^{2}\left\Vert \mathbf{w}^{\star}\right\Vert ^{2}}>\cos^{2}\beta\,.
\end{align*}
Next,
\begin{align*}
\left(\mathbf{w}^{\top}\mathbf{w}^{\star}\right)^{2} & =\left(\sum_{i=1}^{d_{1}}z_{i}z_{i}^{\star}\mathbf{w}_{i}\cdot\mathbf{w}_{i}^{\star}\right)^{2}\,,
\end{align*}
so \eqref{eq: output margin condition appendix} is equivalent to
\[
\left(\sum_{i=1}^{d_{1}^{\star}}z_{i}z_{i}^{\star}\mathbf{w}_{i}\cdot\mathbf{w}_{i}^{\star}\right)^{2}>\left\Vert \mathbf{z}\right\Vert ^{2}\left\Vert \mathbf{z}^{\star}\right\Vert ^{2}\cos^{2}\beta\,.
\]
Conditioning on the events $\bigcap_{i=1}^{d_{1}^{\star}}\tilde{\mathcal{G}}_{i}^{\alpha}\left(\mathbf{W}^{\star}\right)$, and \begin{align*}\label{eq: positive zzstar appendix}
\forall i=1,\dots,d_{1}^{\star}\ z_{i}z_{i}^{\star}\ge0\,,
\end{align*}
\eqref{eq: output margin condition appendix} holds if
\begin{equation}
\left(\mathbf{z}\cdot\mathbf{z^{\star}}\right)^{2}\cos^{2}\alpha=\left(\sum_{i=1}^{d_{1}^{\star}}z_{i}z_{i}^{\star}\cos\alpha\right)^{2}>\left(\left\Vert \mathbf{z}\right\Vert ^{2}\left\Vert \mathbf{z}^{\star}\right\Vert ^{2}\right)\cos^{2}\beta\label{eq: output margin event}
\end{equation}
i.e. 
\[
\frac{\left(\mathbf{z}\cdot\mathbf{z^{\star}}\right)^{2}}{\left\Vert \mathbf{z}\right\Vert ^{2}\left\Vert \mathbf{z}^{\star}\right\Vert ^{2}}>\frac{\cos^{2}\beta}{\cos^{2}\alpha}\,.
\]
Denote 
\[
\gamma=\arccos\left(\frac{\cos\beta}{\cos\alpha}\right)
\]
then, putting this all together, 
\begin{align}
 & \mathbb{P}_{\mathbf{w}}\left(\forall n\in\left[N\right]\,y^{\left(n\right)}h\left(\mathbf{x}^{\left(n\right)}\right)>0\right)\nonumber\\
 & \ge\mathbb{P}\left(\forall n\in\left[N\right]\,y^{\left(n\right)}h\left(\mathbf{x}^{\left(n\right)}\right)>0,\mathbf{W}\in\bigcap_{i=1}^{d_{1}^{\star}}\tilde{\mathcal{G}}_{i}^{\alpha}\left(\mathbf{W}^{\star}\right),\forall i=1,\dots,d_{1}^{\star}\ z_{i}z_{i}^{\star}>0\right)\nonumber\\
 & =\mathbb{P}\left(\forall n\in\left[N\right]\,y^{\left(n\right)}h\left(\mathbf{x}^{\left(n\right)}\right)>0\middle\vert\bigcap_{i=1}^{d_{1}^{\star}}\tilde{\mathcal{G}}_{i}^{\alpha}\left(\mathbf{W}^{\star}\right),\forall i=1,\dots,d_{1}^{\star}\ z_{i}z_{i}^{\star}>0\right)\mathbb{P}\left(\bigcap_{i=1}^{d_{1}^{\star}}\tilde{\mathcal{G}}_{i}^{\alpha}\left(\mathbf{W}^{\star}\right),\forall i=1,\dots,d_{1}^{\star}\ z_{i}z_{i}^{\star}>0\right)\nonumber\\
 & \ge\mathbb{P}\left(\frac{\mathbf{w}^{\top}\mathbf{w}^{\star}}{\left\Vert \mathbf{w}\right\Vert \left\Vert \mathbf{w}^{\star}\right\Vert }>\cos\beta\middle\vert\bigcap_{i=1}^{d_{1}^{\star}}\tilde{\mathcal{G}}_{i}^{\alpha}\left(\mathbf{W}^{\star}\right),\forall i=1,\dots,d_{1}^{\star}\ z_{i}z_{i}^{\star}>0\right)\mathbb{P}\left(\bigcap_{i=1}^{d_{1}^{\star}}\tilde{\mathcal{G}}_{i}^{\alpha}\left(\mathbf{W}^{\star}\right),\forall i=1,\dots,d_{1}^{\star}\ z_{i}z_{i}^{\star}>0\right)\nonumber\\
 & \ge\frac{1}{2}\mathbb{P}\left(\frac{\left(\mathbf{z}\cdot\mathbf{z^{\star}}\right)^{2}}{\left\Vert \mathbf{z}\right\Vert ^{2}\left\Vert \mathbf{z}^{\star}\right\Vert ^{2}}>\frac{\cos^{2}\beta}{\cos^{2}\alpha}\middle\vert\bigcap_{i=1}^{d_{1}^{\star}}\tilde{\mathcal{G}}_{i}^{\alpha}\left(\mathbf{W}^{\star}\right),\forall i=1,\dots,d_{1}^{\star}\ z_{i}z_{i}^{\star}>0\right)\mathbb{P}\left(\bigcap_{i=1}^{d_{1}^{\star}}\tilde{\mathcal{G}}_{i}^{\alpha}\left(\mathbf{W}^{\star}\right),\forall i=1,\dots,d_{1}^{\star}\ z_{i}z_{i}^{\star}>0\right)\nonumber\\
 & =\frac{1}{2}\mathbb{P}\left(\frac{\left(\mathbf{z}\cdot\mathbf{z^{\star}}\right)^{2}}{\left\Vert \mathbf{z}\right\Vert ^{2}\left\Vert \mathbf{z}^{\star}\right\Vert ^{2}}>\cos^{2}\gamma\middle\vert\bigcap_{i=1}^{d_{1}^{\star}}\tilde{\mathcal{G}}_{i}^{\alpha}\left(\mathbf{W}^{\star}\right),\forall i=1,\dots,d_{1}^{\star}\ z_{i}z_{i}^{\star}>0\right)\mathbb{P}\left(\bigcap_{i=1}^{d_{1}^{\star}}\tilde{\mathcal{G}}_{i}^{\alpha}\left(\mathbf{W}^{\star}\right)\right)\mathbb{P}\left(\forall i=1,\dots,d_{1}^{\star}\ z_{i}z_{i}^{\star}>0\right)\nonumber\\
 & =\frac{1}{2}\mathbb{P}\left(\frac{\left(\mathbf{z}\cdot\mathbf{z^{\star}}\right)^{2}}{\left\Vert \mathbf{z}\right\Vert ^{2}\left\Vert \mathbf{z}^{\star}\right\Vert ^{2}}>\cos^{2}\gamma\middle\vert\forall i=1,\dots,d_{1}^{\star}\ z_{i}z_{i}^{\star}>0\right)\label{eq: cont interpolation first term appendix}\\ &\quad\cdot\mathbb{P}\left(\bigcap_{i=1}^{d_{1}^{\star}}\tilde{\mathcal{G}}_{i}^{\alpha}\left(\mathbf{W}^{\star}\right)\right)\label{eq: cont interpolation second term appendix}\\
 &\quad\cdot\mathbb{P}\left(\forall i=1,\dots,d_{1}^{\star}\ z_{i}z_{i}^{\star}>0\right)\,.\label{eq: cont interpolation third term appendix}
\end{align}
Starting from \eqref{eq: cont interpolation first term appendix},
\begin{align*}
\mathbb{P}\left(\frac{\left(\mathbf{z}\cdot\mathbf{z^{\star}}\right)^{2}}{\left\Vert \mathbf{z}\right\Vert ^{2}\left\Vert \mathbf{z}^{\star}\right\Vert ^{2}}>\cos^{2}\gamma\middle\vert\forall i=1,\dots,d_{1}^{\star}\ z_{i}z_{i}^{\star}>0\right)=\mathbb{P}\left(\frac{\left(\sum_{i=1}^{d_{1}^{\star}}\left|z_{i}z_{i}^{\star}\right|\right)^{2}}{\left\Vert \mathbf{z}\right\Vert ^{2}\left\Vert \mathbf{z}^{\star}\right\Vert ^{2}}>\cos^{2}\gamma\right)
\end{align*}
and since $\left(\mathbf{z}\cdot\mathbf{z^{\star}}\right)^{2}\le\left(\sum_{i=1}^{d_{1}^{\star}}\left|z_{i}z_{i}^{\star}\right|\right)^{2}$ almost surely,
\begin{align*}
\mathbb{P}\left(\frac{\left(\mathbf{z}\cdot\mathbf{z^{\star}}\right)^{2}}{\left\Vert \mathbf{z}\right\Vert ^{2}\left\Vert \mathbf{z}^{\star}\right\Vert ^{2}}>\cos^{2}\gamma\middle\vert\forall i=1,\dots,d_{1}^{\star}\ z_{i}z_{i}^{\star}>0\right)\ge\mathbb{P}\left(\frac{\left(\mathbf{z}\cdot\mathbf{z^{\star}}\right)^{2}}{\left\Vert \mathbf{z}\right\Vert ^{2}\left\Vert \mathbf{z}^{\star}\right\Vert ^{2}}>\cos^{2}\gamma\right)\,.
\end{align*}
From \lemref{lem: angles between random Gaussian vectors appendix}, we obtain
\[
\mathbb{P}\left(\frac{\left(\mathbf{z}\cdot\mathbf{z^{\star}}\right)^{2}}{\left\Vert \mathbf{z}\right\Vert ^{2}\left\Vert \mathbf{z}^{\star}\right\Vert ^{2}}>\cos^{2}\left(\gamma\right)\right)\ge\frac{2\sin\left(\gamma\right)^{d_{1}^{\star}-1}}{\left(d_{1}^{\star}-1\right)B\left(\frac{1}{2},\frac{d_{1}^{\star}-1}{2}\right)}\,.
\]
As for \eqref{eq: cont interpolation second term appendix}, we know from \lemref{lem: activation matching probability appendix} that 
\begin{align*}
\mathbb{P}\left(\bigcap_{i=1}^{d_{1}^{\star}}\tilde{\mathcal{G}}_{i}^{\alpha}\left(\mathbf{W}^{\star}\right)\right) \ge \left[\frac{\sin\left(\alpha\right)^{d_{0}-1}}{\left(d_{0}-1\right)B\left(\frac{1}{2},\frac{d_{0}-1}{2}\right)}\right]^{d_{1}^{\star}}\,.
\end{align*}
For \eqref{eq: cont interpolation third term appendix}, $\mathbb{P}_{z}\left(z\right)=\mathcal{N}\left(0,1\right)$ so 
\[
\mathbb{P}\left(\forall i=1,\dots,d_{1}^{\star}\ z_{i}z_{i}^{\star}>0\right)=2^{-d_{1}^{\star}}\,.
\]
Overall,
\begin{align*}
\mathbb{P}_{\mathbf{w}}\left(\forall n\in\left[N\right]\,y^{\left(n\right)}h\left(\mathbf{x}^{\left(n\right)}\right)>0\right)\ge2^{-d_{1}^{\star}}\frac{\sin\left(\gamma\right)^{d_{1}^{\star}-1}}{\left(d_{1}^{\star}-1\right)B\left(\frac{1}{2},\frac{d_{1}^{\star}-1}{2}\right)}\left[\frac{\sin\left(\alpha\right)^{d_{0}-1}}{\left(d_{0}-1\right)B\left(\frac{1}{2},\frac{d_{0}-1}{2}\right)}\right]^{d_{1}^{\star}}\,.
\end{align*}
When $d_0 \gg d_1^{\star} \gg 1$ we get from \corref{cor: angles between high dimensional random Gaussian vectors appendix}
\begin{align*}
\mathbb{P}_{\mathbf{w}} \left( \forall n \in \left[N\right] \, y^{\left(n\right)} h\left(\mathbf{x}^{\left(n\right)}\right) > 0 \right) & \ge 2^{-d_{1}^{\star}} \exp\left( d_{1}^{\star} d_{0} \log\left(\sin\left(\alpha\right)\right) - \frac{1}{2} d_{1}^{\star} \log\left(d_{0}\right) - \frac{1}{2} d_{1}^{\star} \log\left(2\pi\right)\right) \left(1+O\left(d_{0}^{-1}d_{1}^{\star}\right)\right) \\
& \quad \cdot \exp\left( d_{1} \log\left( \sin\left(\gamma\right)\right) - \frac{1}{2} \log\left(d_{1}\right) - \frac{1}{2} \log\left(2\pi\right)\right) \left(1 + O \left(d_{1}^{-1}\right)\right)\,.
\end{align*}
That is,
\begin{align*}
\hat{p}_{\mathcal{S}} &= 
\mathbb{P}_{\mathbf{w}} \left(\forall n\in\left[N\right]\,y^{\left(n\right)}h\left(\mathbf{x}^{\left(n\right)}\right)>0\right) \\
& \ge \exp{\left( d_{1}^{\star} d_{0} \log\left(\sin\left(\alpha\right)\right) + d_{1} \log\left(\sin\left(\gamma\right)\right) - \frac{1}{2} d_1^{\star} \log \rb{d_0} + O\left( d_1^{\star} + \log{\left(d_1\right)}\right)\right)}\,.
\end{align*}
\end{proof}

The following generalization bound follows directly from \propref{theorem:pscard simplified} and \thmref{thm: interpolation probability of continuous networks}.
\begin{corollary} [Generalization of continuous two layer networks, restated] \label{thm: cont_gen_bound_app}
Under the assumption that $\hat{p}_{\calS}<\frac{1}{2}$, for any $\varepsilon, \delta \in \rb{0,1}$, 
$$\bbP_{S\sim\calD^N, h \sim \mathcal{P}_{\mathcal{S}}} \rb{\exError{h} \leq \varepsilon} \ge 1-\delta\ - \lambda,,$$
for
\begin{align*}
N \ge \frac{ \BSCM{cont} + 4\log\left(\frac{8}{\delta}\right) +2\log \rb{ \BSCM{cont}}}{\varepsilon}
\end{align*}
with
\begin{align*}
    \BSCM{cont} = - d_{1}^{\star} d_{0} \log\left(\sin\left(\alpha\right)\right) - d_{1} \log\left(\sin\left(\gamma\right)\right) + \frac{1}{2} d_1^{\star} \log \rb{d_0} + O\left( d_1^{\star} + \log{\left(d_1\right)}\right)\,.
\end{align*}
\end{corollary}

\begin{proof}
Under \asmref{asm: angular margin appendix}, \thmref{thm: interpolation probability of continuous networks} implies that w.p. at least $1 - \lambda$ over $\mathcal{S} \sim \mathcal{D}^N$, 
\begin{align*}
\hat{p}_{\mathcal{S}} \ge \exp{\left(d_{1}^{\star}d_{0}\log\left(\sin\left(\alpha\right)\right) + d_{1}\log\left(\sin\left(\gamma\right)\right) + O\left(d_1^{\star}\log{\left(d_0\right)} + \log{\left(d_1\right)}\right)\right)}\,.
\end{align*}
We denote this event by $\mathcal{E}_1$.
Recalling that $\BSCM{cont} \ge -\log\rb{\hat{p}_{\calS}}$ when conditioned on $\mathcal{E}_1$, from \propref{theorem:pscard simplified} we deduce that,
\begin{align*}
\bbP_{\calS \sim\calD^N, h \sim \mathcal{P}_{\mathcal{S}}} \rb{\exError{h} \leq \varepsilon} \ge 1-\delta\,.
\end{align*}
We denote this event by $\mathcal{E}_2$.
Using the inclusion exclusion principle,
\begin{align*}
\bbP_{\calS \sim\calD^N, h \sim \mathcal{P}_{\mathcal{S}}} \rb{\mathcal{E}_1 \cap \mathcal{E}_2} &= \bbP_{\calS \sim\calD^N, h \sim \mathcal{P}_{\mathcal{S}}} \rb{\mathcal{E}_1} + \bbP_{\calS \sim\calD^N, h \sim \mathcal{P}_{\mathcal{S}}} \rb{\mathcal{E}_2} - \bbP_{\calS \sim\calD^N, h \sim \mathcal{P}_{\mathcal{S}}} \rb{\mathcal{E}_1 \cup \mathcal{E}_2} \\
& \ge 1 - \lambda + 1 - \delta - 1 \\
& = 1 - \delta - \lambda \,.
\end{align*}
Therefore,
$$\bbP_{\calS \sim\calD^N, h \sim \mathcal{P}_{\mathcal{S}}} \rb{\exError{h} \leq \varepsilon} \ge 1-\delta\ - \lambda\,.$$
\end{proof}

\newpage
\section{Proofs for Sparsest Quantized Interpolator Learning Rule}
\label{app:sparse}

\subsection{Setting and Notation}
Given a directed graph $G=\rb{V,E}$ and $x \in V$, we use $d^{\mathrm{in}} \rb{x}$ to denote the in-degree of $x$, i.e.
\begin{align*}
    d^{\text{in}} \rb{x} \triangleq \sum_{y \in V} \mathbb{I} \bb{\rb{y, x} \in E}\,.
\end{align*}

Under the same quantization scheme as in \secref{sec:quant_nets}, consider the following learning rule:

\begin{definition}\label{def: sparsest app}
    $\mathcal{A}_{0}\rb{\calS}=h_{\btheta_0}$ returns the sparsest quantized interpolator,
    \begin{align*}
        \btheta_{0} = \argmin_{\btheta \in \mathcal{Q}^{M}} {\norm{\btheta}_{0}}\;\text{s.t}\;\forall n\in\bb{N}\;y^{\rb{n}}h_{\btheta}\rb{\mathbf{x}^{\rb{n}}}>0\,,
    \end{align*}
    where $\norm{\btheta}_0$ is the number of nonzero values in $\btheta$.
    With some abuse of notation, we use $\norm{\mathcal{A}_0 \rb{\mathcal{S}}}_0$ and $\norm{\btheta_0}_0$ interchangeably.
\end{definition}

Recall that we denote the total number of parameters in a teacher network $\teacher$ by 
\begin{align*}
    \ccfct = \sum_{l=1}^{L} \rb{d_l^{\star} d_{l-1}^{\star} + d_{l}^{\star}}\,.
\end{align*}
We additionally denote by
\begin{align*}
    \wcfct &= \sum_{l=1}^{L} d_l^{\star} d_{l-1}^{\star}\,, \\
    \bcfct &= \sum_{l=1}^{L} d_{l}^{\star}\,
\end{align*}
the maximal number of non-zero weights and biases in $\teacher$, respectively.
Denote the class of fully-connected neural networks with at most $\wcfct$ non-zero weights and $\bcfct$ biases as $\mathcal{H}_{\wcfct, \bcfct}$ \footnote{This is different from $\mathcal{H}_D^{FC}$ as no specific depth and hidden layer widths are assumed for $\mathcal{H}_{\wcfct, \bcfct}$}. 
Notice that the number of neural networks in $\mathcal{H}_{\wcfct, \bcfct}$ is bounded by the number of neural networks with $\ccfct$ edges and no bias terms.
We denote the set of such neural networks by $\mathcal{H}_{\ccfct}$, then
\begin{align} \label{eq: wb m ineq sparse appendix}
    \abs{\mathcal{H}_{\wcfct, \bcfct}} \le \abs{\mathcal{H}_{\ccfct}}\,.
\end{align}
We emphasize that $\mathcal{H}_{\wcfct, \bcfct}$ and $\mathcal{H}_{\ccfct}$ do not have fixed depth and hidden layer width, and contains models which do not conform to a specific $D$.

\subsection{Generalization Bound}

\begin{lemma}
The number of FCNs with $\norm{\mathcal{A}_0\rb{\calS}}_0$ non-zero parameters is upper bounded by $\abs{\mathcal{H}_{\ccfct}}$. 
\end{lemma}

\begin{proof}
    Since $\teacher$ is an interpolating solution, we have that $\norm{\mathcal{A}_0\rb{\calS}}_0 \le \ccfct$.
    The bound follows from the case in which all neurons have 0 bias, and the number of non-zero weights is equal to the number of non-zero parameters.
\end{proof}

Next, we note that any fully-connected neural network with no bias terms can be represented as a \textit{weighted directed acyclic graph} (WDAG) $G = \rb{V,E,w}$.
The vertices $V$ represent the neurons, and the network's non-zero weights are represented as weighted edges. 
Notice that the input neurons in an FCN are of 0 in-degree, and that all neurons are reachable from some input neuron.
This motivates us to define the following.

\begin{definition}\label{def: set of graphs}
Let $\tilde{\mathcal{G}}_{\ccfct, d_0}$ be the set of DAGs, $G=\rb{V,E}$, containing a subset $\Sigma\subseteq V$ with $0$ in-degree.
Such that
\begin{align*}
    \tilde{\mathcal{G}}_{\ccfct, d_0} \triangleq \cb{ G = \rb{V, E} \middle\vert \; G \; \text{is a DAG}, \; \abs{E} = \ccfct, \; \exists \Sigma \subseteq V\,:\, \abs{\Sigma}=d_0,\; \forall x \in \Sigma \; d^{\text{in}} \rb{x} = 0}\,.
\end{align*}
We say that a vertex $v \in V$ is reachable from $\Sigma$ if there exists some directed path from a vertex in $\Sigma$ to $v$.
With this notion, we further specify
\begin{align*}
    \mathcal{G}_{\ccfct, d_0} \triangleq \cb{ G = \rb{V, E} \in \tilde{\mathcal{G}}_{\ccfct, d_0} \, \middle\vert \, \forall v\in V \; \text{$v$ is reachable from $\Sigma$}}\,.
\end{align*}
That is, $\mathcal{G}_{\ccfct, d_0}$ is the subset of $\tilde{\mathcal{G}}_{\ccfct, d_0}$ in which any node is reachable from some node in $\Sigma$.
\end{definition}

Clearly, $\abs{\mathcal{G}_{\ccfct, d_0}}$ is an upper bound for $\abs{\mathcal{H}_{\ccfct}}$, so
\begin{align} \label{eq: sparse calH first bound appendix}
    \abs{\mathcal{H}_{\ccfct}} \le Q^{\ccfct} \abs{\mathcal{G}_{\ccfct, d_0}}\,.
\end{align}

\begin{lemma} \label{lem: sparse calG bound appendix}
\[
\\
\abs{\mathcal{G}_{\ccfct, d_0}} \le \frac{\left(\ccfct\left(\ccfct + d_{0}\right)\right)^{\ccfct}}{\ccfct!}\le\left(\ccfct+d_{0}\right)^{2\ccfct}\,.
\]
\end{lemma}

\begin{proof}
We start from the basic property, that any directed graph
$G = \left(V,E\right)$ can be represented as a bipartite undirected graph $\tilde{G} = \left(V^{\prime}\cup V^{\prime\prime},\tilde{E}\right)$
where $V^{\prime},V^{\prime\prime}$ are copies of $V$ and 
\[
\tilde{E}=\left\{ \left\{ v_{1}^{\prime},v_{2}^{\prime\prime}\right\} \in V^{\prime}\times V^{\prime\prime}\middle\vert\left(v_{1},v_{2}\right)\in E\right\} \,.
\]
Let $G = \rb{V, E} \in \mathcal{G}_{\ccfct, d_0}$, $\Sigma \subseteq V$ the appropriate $0$-in-degree subset of nodes in $G$, and $\Tilde{G} = \rb{V^\prime \cup V^{\prime\prime}, \Tilde{E}}$ its corresponding bipartite representation.
By the definition of $\mathcal{G}_{\ccfct, d_0}$, every vertex $v \in V \setminus \Sigma$ is reachable from some $x \in \Sigma$ and therefore for all $v \in V \setminus \Sigma$, $d^{\mathrm{in}} \rb{v} \ge 1$ in G, and $\mathrm{deg} \rb{v^{\prime\prime}} \ge 1$ in $\Tilde{G}$.
Since for all $x \in \Sigma$, $d^{\mathrm{in}} \rb{x} = 0$, and $\abs{E} = \abs{\Tilde{E}} = \ccfct$, we can use the pigeonhole principle to deduce that $V = \Sigma \cup U$ where $U$ is a set of at most $\wcfct$ vertices, $\abs{U} \le \wcfct$.
Hence, any $G \in \mathcal{G}_{\ccfct, d_0}$ can be represented using an undirected bipartite graph $\hat{G} = \rb{\hat{V}^\prime \cup \hat{V}^{\prime\prime}, \hat{E}}$ such that 
\begin{align*}
    \abs{\hat{V}^{\prime}} \le \ccfct + d_0 \,, 
\end{align*}
\begin{align*}
    \abs{\hat{V}^{\prime\prime}} \le \ccfct \,, 
\end{align*}
\begin{align*}
    \abs{\hat{E}} = \ccfct \,,
\end{align*}
where $\hat{V}^{\prime}$ is a copy of $\Sigma \cup U$, and $\hat{V}^{\prime\prime}$ is a copy of $U$.
This means that we can bound $\abs{\mathcal{G}_{\ccfct, d_0}}$ with the number of such graphs.
The number of possible edges in $\hat{G}$ is 
\begin{align*}
    \abs{\hat{V}^{\prime}} \cdot \abs{\hat{V}^{\prime\prime}} = \left(\ccfct + d_{0}\right) \cdot \ccfct\,,
\end{align*}
so, overall, the number of such bipartite representations for graphs in $\mathcal{G}_{\ccfct, d_0}$ is
\[
\binom{\ccfct\left(\ccfct+d_{0}\right)}{\ccfct} \le \frac{\left(\ccfct\left(\ccfct+d_{0}\right)\right)^{\ccfct}}{\ccfct!} \le\frac{\left(\ccfct+d_{0}\right)^{2\ccfct}}{\ccfct!} \le\left(\ccfct + d_{0}\right)^{2\ccfct}
\]
and
\begin{align*}
    \abs{\mathcal{G}} \le \frac{\left(\ccfct\left(\ccfct+d_{0}\right)\right)^{\ccfct}}{\ccfct!}\,.
\end{align*}

\end{proof}

Collecting the bounds from \eqref{eq: wb m ineq sparse appendix}, \eqref{eq: sparse calH first bound appendix} and \lemref{lem: sparse calG bound appendix} we find the following corollaries.

\begin{corollary}
The number of $\ccfct$-sparse $Q$-quantized fully-connected neural networks is bounded by
\[
\left|\mathcal{H}_{\ccfct}\right| \leq \frac{\left(\ccfct\left(\ccfct+d_{0}\right)\right)^{\ccfct}}{\ccfct!} Q^{\ccfct} \le \left(\ccfct+d_{0}\right)^{2\ccfct} Q^{\ccfct}\,.
\]
\end{corollary}

Using \thmref{Theorem: classic finite hypothesis class generalization result} we get,

\begin{corollary}
Let $\varepsilon > 0$ and $\delta \in \rb{0, 1}$.
With probability at least $1-\delta$ over $\mathcal{S} \sim \mathcal{D}^N$, 
\begin{align*}
    \exError{\mathcal{A}_0 \rb{\mathcal{S}}} \le \varepsilon
\end{align*}
when
\begin{align*}
    N \ge \frac{2 \ccfct \log\left( \ccfct +d_{0} \right) + \ccfct \log\left( Q \right) + \log\left(\frac{1}{\delta}\right)}{\varepsilon}\,.
\end{align*}
\end{corollary}


\end{document}